\documentclass[mnsc,nonblindrev]{informs3} 
\usepackage{enumitem}
\usepackage{algorithm}
\usepackage{import}
\usepackage{hyperref}
\usepackage{caption}
\usepackage{subcaption}

\usepackage{multirow}
\usepackage{booktabs}
\usepackage{tabularx}

\usepackage{amssymb, bbold, mathtools}
\newtheorem{proposition}{Proposition}
\newtheorem{theorem}{Theorem}

\newtheorem{lemma}{Lemma}
\newtheorem{definition}{Definition}
\newtheorem{remark}{Remark}
\newtheorem{running}{Example}

\usepackage{color}
\usepackage{capt-of}
\newtheorem{assumption}{Assumption}

\usepackage{MnSymbol}%
\usepackage{enumitem}
\setlist[enumerate]{wide=0pt, leftmargin=\parindent}

\usepackage{multirow}
\usepackage{graphicx}
\usepackage[affil-it]{authblk}

\definecolor{red}{RGB}{163, 31, 52}
\definecolor{gray}{RGB}{194, 192, 191}
\definecolor{blue}{RGB}{59, 89, 152}
\definecolor{green}{RGB}{34, 139, 34}

\usepackage{acronym}
\acrodef{ucb}[UCB]{upper confidence bound}

\renewcommand{\tfrac}[2]{#1/#2}
\newcommand{\norm}[1]{\|#1\|}

\newcommand{\ceil}[1]{\lceil#1\rceil}
\newcommand{\iprod}[2]{\left<#1, #2\right>}
\newcommand{\abs}[1]{\left|#1\right|}
\newcommand{\E}[2]{\mathbb{E}_{#1} \left[ #2 \right]}

\newcommand{\set}[2]{\left\{ #1\ : \ #2 \right\}}
\newcommand{\tset}[2]{\{ #1\ : \ #2 \}}
\newcommand{\tpose}{^\top}
\newcommand{\defn}[0]{:=}
\newcommand{\e}[1]{\ensuremath{\times 10^{#1}}}

\newcommand{\mb}{\mathbb}

\newcommand{\mr}{\mathrm}
\renewcommand{\mc}{\mathcal}

\renewcommand{\d}{\mr{d}}

\newcommand{\st}{\mr{s.t.}}

\newcommand{\Prob}[1]{\mr{Prob}\left[#1\right]}

\DeclareMathOperator{\supp}{supp}

\DeclareMathOperator{\interior}{int}
\DeclareMathOperator{\relint}{relint}

\DeclareMathOperator{\dom}{dom}

\DeclareMathOperator{\cone}{cone}

\def\setr {\mathfrak R}
\def\real{\mathbb R}

\def\prob{\textsc{Deceit}}
\def\dis{\textsc{Dist}}
\def\xd{\tilde{X}_d}
\def\xn{\tilde{X}_n}
\def\dual{\textsc{Dual}}

\def \rew{{\textsc{Rew}}}
\def \reg{{\textsc{Reg}}_{\pi}}

\usepackage[mathscr]{euscript}
 \let\mathscr\relax
\usepackage[scr]{rsfso}

\def \optp {P}

\def \stop {{\textsc t}}
\def \stop {{t}}
\def \badexploit {\mathcal  B^1}
\def \badexplore {\mathcal  B^2}
\def \tauexplore {\tau}

\graphicspath{{figures/}}
\OneAndAHalfSpacedXII 

\usepackage{natbib}
\bibpunct[, ]{(}{)}{,}{a}{}{,}%
\def\bibfont{\small}%
\ECRepeatTheorems

\EquationsNumberedThrough    

\MANUSCRIPTNO{}

\begin{document}


\RUNAUTHOR{Van Parys and Golrezaei}

\RUNTITLE{Optimal Learning for Structured Bandits}

\TITLE{Optimal Learning for Structured Bandits}

 \ARTICLEAUTHORS{%
   \AUTHOR{Bart Van Parys  ~~and~~ Negin Golrezaei
   }
\AFF{Massachusetts Institute of Technology (MIT) - Sloan School of Management, \EMAIL{vanparys,golrezae@mit.edu }}
 } 

\ABSTRACT{
  We study structured multi-armed bandits, which is the problem of online decision-making under uncertainty in the presence of structural information.  In this problem, the decision-maker needs to discover the best course of action despite observing only uncertain rewards over time.  {The decision-maker is aware of certain convex structural information regarding the reward distributions; that is, the decision-maker knows 
  the reward distributions of the arms belong to a  convex compact set.}
  In the presence such structural information, they then would like to minimize their regret by exploiting this information,  where the regret is its performance difference against a benchmark policy that knows the best action ahead of time.  In the absence of structural information, the classical upper confidence bound (UCB) and Thomson sampling algorithms are well known to suffer minimal regret.  As recently pointed out by \citet{russo2018learning} and \cite{pmlr-v54-lattimore17a}, neither algorithms are, however, capable of exploiting structural information that is commonly available in practice.  We propose a novel learning algorithm that we call ``DUSA'' whose regret matches  the information-theoretic regret lower bound up to a  constant factor and can handle a wide range of structural information.  Our algorithm DUSA solves a dual counterpart of the regret lower bound at the empirical reward distribution and follows its suggested play.  
We show that this idea leads to the first computationally viable learning policy with asymptotic minimal regret for various structural information, including well-known structured bandits such as linear, Lipschitz, and convex bandits, and novel structured bandits that have not been studied in the literature due to the lack of a unified and flexible framework.
}


\maketitle

\section{Introduction}

In the multi-armed bandit framework, a decision-maker is repeatedly offered a set of arms (options) with unknown rewards to choose from.
In order to perform well in this framework, one needs to strike a balance between exploration and exploitation. 
Many classical bandit algorithms, including the \ac{ucb} algorithm \citep{garivier2011kl} and Thompson sampling \citep{thompson1933likelihood, kaufmann2012thompson}, have been designed specifically to optimally trade-off exploration and exploitation in the absence of structural information.
While these algorithms  perform well when the rewards of all arms are arbitrary and unstructured, they perform poorly when there is a correlation structure between the rewards of the arms \citep{russo2018learning, pmlr-v54-lattimore17a}.

Structural information allows a decision-maker to use rewards observed concerning one arm to indirectly deduce knowledge concerning other arms as well.
As an example, in revenue management problems, where we use posted price mechanisms to sell items, every arm corresponds to a price.
The probability of receiving a non-zero reward ought to decrease with the price as suggested by standard economic theory and represents structural information.
In healthcare,  one expects to see a similar performance for drugs (arms) which are composed of similar active ingredients.
Structural information reduces the need for exploration and hence directly reduces the suffered regret.
However, structured bandit problems are far more complicated than their classical counterparts in which rewards observed in the context of one arm carry no information about any other arm.  
We present here a unified yet flexible framework to study structured multi-armed bandit problems.
\emph{Our key research question is how can one exploit structural information to design efficient optimal learning algorithms that help the decision-maker to obtain higher rewards?} 

{\color{black} \textbf{A unified framework to model structural information.}  To answer this question, as one of our main contributions, we present a simple model that allows us to capture a wide range of structural information. In our model, we assume that the reward distribution of arms belongs to a   convex compact set $\mathcal P$ which is known to the decision-maker. The set $\mathcal P$ captures the structural information concerning the unknown reward distribution and allows us to have a unified representation of the structural information. For instance, via this set $\mathcal P$, we can represent well-known structural information, including linear, Lipschitz, and convex bandits. 
  More importantly, the set $\mathcal P$  enables us to model and study structural information that has not been studied in the literature due to the lack of a unified framework; see Section \ref{sec:struct-band-ex} for further details.}

\textbf{Mimicking the regret lower bound.} We take a principled approach to design an effective learning algorithm for structured bandits. Our approach is based on the information-theoretic regret lower bound for structured bandits characterized first by \citet{graves1997asymptotically}.
This regret lower bound quantifies the exact extent to which structural information may reduce the suffered regret of any uniformly good bandit policy, i.e., any policy with a sublinear worst-case regret bound.
The regret lower bound is characterized as a semi-infinite optimization problem where the decision variable can be interpreted as the rate with which suboptimal arms need to be explored.
The constraints ensure that uniformly good bandit policies can distinguish the true reward distribution from certain \emph{deceitful}  distributions;  see Equation \eqref{eq:problematic-distributions} for their formal definition.

Having this interpretation of the regret lower bound in mind, we design an algorithm that aims at targeting the exploration rates suggested by the regret lower bound.
However, targeting the suggested exploration rates exactly is not possible as  obtaining these rates  requires the actual reward distributions of the arms to be known.
One way to overcome this challenge is to follow the exploration rates suggested by the regret lower bound, not for the actual unknown reward distribution, but for its empirical counterpart based on the past rewards instead.
Unfortunately, this idea put forward by \citet{combes2017minimal} does not lead to a practical algorithm as obtaining the exploration rates suggested by the regret lower bound in every round involves solving a semi-infinite optimization problem which, except for a few exceptional cases, is computationally prohibitive.\footnote{We highlight that the approach in \citet{combes2017minimal} only works for  specific structured bandit problems such as Lipschitz bandits with Bernoulli rewards. For these problems, unlike general structured bandit problems such as dispersion bandits that we introduce in Section \ref{sec:struct-band-ex}, the  semi-infinite optimization problem associated with the regret lower bound admits a (semi) closed-form solution and can be solved easily.
}
 Our approach advances two key innovations which alleviate several shortcomings of the OSSB policy of \citet{combes2017minimal}:

 \begin{itemize}[leftmargin=*]
\item [i)] \textbf{Dual perspective.} First, we design a dual-based algorithm that instead of solving the regret lower bound problem directly solves its dual counterpart.
The dual counterpart, which is a finite convex optimization problem, can be solved effectively and as a result, our  dual-based algorithm  is
computationally efficient for a large class of structured bandits.
By taking advantage of the dual counterpart, we offer a unified approach to incorporate convex structural information.
To the best of our knowledge, this is the first use of convex duality to directly derive asymptotically optimal and efficiently computable online algorithms for a general class of structured bandit problems.

\item [ii)] \textbf{Sufficient information test.} Second, via constructing an easy-to-solve \emph{sufficient information} test, our policy, which is called {DU}al {S}tructure-based {A}lgorithm (DUSA),  solves the dual counterpart of the regret lower bound problem in merely a logarithmic number of rounds, i.e., $\mc O(\log(T))$, instead of solving it in every round, where $T$ is the number of rounds.
Our sufficient information test, which is a simple univariate optimization problem,  verifies in any particular round whether sufficient exploration has been done in the past and if so, the dual problem does not need to be resolved.
If the sufficient information test passes, DUSA plays the empirically optimal arm, otherwise, DUSA solves the dual regret lower bound and enters an exploration phase.
\end{itemize}

\textbf{Minimal regret.} We show that DUSA can optimally exploit any structural information which can be represented by an arbitrary convex constraint on the reward distribution. This novel design leads to an optimal learning algorithm whose regret matches the theoretical regret lower bound.
Put differently, DUSA optimally utilizes the underlying structure by solving a sequence of dual counterparts of regret lower bound problem and this allows DUSA to offer a novel duality principle for structured bandit problems.

{\color{black}\textbf{Numerical Studies.} We also conduct numerical studies in which we evaluate DUSA on various structured bandits. We first focus on linear and Lipschitz bandits, which are well-studied structured bandits. Even though DUSA is a universal structured bandit algorithm and is not tailored to specific structural information, we observe that the average performance of DUSA is comparable to that of algorithms that are designed specifically for the aforementioned structured bandits. Furthermore, DUSA's worst-case performance tends to be more concentrated around its average performance, suggesting that DUSA is a more reliable algorithm for conservative decision-makers.

For Bernoulli Lipschitz bandits, we further compare the computation/running  time of DUSA with that of the OSSB algorithm of \cite{combes2017minimal} which solves the regret lower bound (not its dual counterpart) in every single round. For Bernoulli Lipschitz bandits,  the regret lower bound has a simple (semi) closed-form solution and hence is computationally efficient to solve. This allows OSSB to achieve a slightly better computation time for finite $T$, compared with DUSA.\footnote{We  highlight that unlike DUSA,  the OSSB algorithm is only designed for Bernoulli rewards so that the regret lower bound enjoys a (semi) closed-form solution. When moving away from Bernoulli rewards, it is not clear how to extend OSSB.
} 
We further evaluate DUSA for dispersion bandits, which are new structured bandits which we introduce. For such structured bandits, we observe that  DUSA outperforms the vanilla UCB algorithm that ignores the structural information. } 

\section{Related Work}
\label{sec:related}
Our work contributes to the literature on the stochastic multi-armed bandit problem.
As stated earlier, the multi-armed bandit framework has been widely applied in many different domains; see, for example, \cite{besbes2009dynamic, keskin2014dynamic}, and \cite{golrezaei2019dynamic}.
Most of the papers in the literature assume that the rewards of arms are statistically independent of each other; see \cite{bubeck2012regret} for a survey on multi-armed bandits.
While this independence assumption simplifies the problem of designing a learning algorithm, when arms are correlated, it can lead to a suboptimal regret.
Considering this, many papers aim at designing learning algorithms that exploit structural information.
However, most of these papers focus on very restricted types of structural information.

Several papers study the design of learning algorithms in the presence of a linear reward structure \citep{dani2008stochastic, rusmevichientong2010linearly, mersereau2009structured, pmlr-v54-lattimore17a}. Some other papers focus on a Lipschitz reward structure that roughly speaking, enforces the reward of similar arms to be close to each other \citep{pmlr-v35-magureanu14, mao2018contextual}. Other structures that have been studied include imposing upper bounds on the average rewards of arms \citep{gupta2019multi} and imposing lower and upper bounds on the realized rewards of the arms \citep{bubeck2019multi}. There are also papers that study structural information in (i) contextual multi-armed bandit problems (e.g., \cite{slivkins2011contextual} and  \cite{balseiro2019contextual}) and (ii) combinatorial decision-making (e.g., \cite{streeter2009online, zhang2019one}, and \cite{niazadeh2020online}).   
 In this work, we present a general and unified framework to exploit a wide range of  structures including some of the {\color{black}structures} we discussed above. The novel class of asymptotically optimal online algorithms we propose here is inspired by those discussed in \citet{combes2017minimal}. They similarly consider online algorithms that imitate the information-theoretic lower bound by solving semi-infinite regret lower bound optimization problems over time.
At a high level, there are three main differences between our algorithm and theirs.
First, our algorithm does not require to solve a semi-infinite regret lower bound optimization problem.
It instead solves its dual convex counterpart, which is computationally more tractable than the primal characterization of the regret lower bound problem.
Second, while the online algorithm in \citet{combes2017minimal} needs to solve the regret lower bound optimization problem in every round, our algorithm only solves its dual counterpart in a logarithmic number of rounds.
As stated earlier, by taking advantage of  convex duality, we design a simple sufficient information test that allows our algorithm not to solve the dual problem in every round. Hence, by relying on the duality perspective of the lower regret bound,  we not only obtain a tractable reformulation, we also significantly decrease the number of times that the dual counterpart of the regret lower problem needs to be solved.
{\color{black}Finally, \citet{combes2017minimal} assume that the reward distribution of each arm is uniquely determined by its mean and impose structures only on the mean reward of arms.
We make no such limiting assumption and allow explicitly for arms to have almost any arbitrary reward distribution. In addition, via our convex set $\mathcal P$, we can go beyond only imposing structures on the mean reward of arms; see, for example, our dispersion and divergence bandits in Section \ref{sec:struct-band-ex}.} 

By mimicking the regret lower bound that embeds the structural information, our algorithm may play suboptimal arms in order to obtain information about other arms.\footnote{\color{black} The idea of mimicking the regret lower bound has been also used in a follow-up paper by \cite{jun2020crush}. Similar to \cite{pmlr-v54-lattimore17a}, they also consider a structural information that impose constraints/structures  on the average rewards of arms. In our work, unlike the two aforementioned work, via convex set $\mathcal P$, we have the flexibility of imposing constraints/structures on the reward distributions of arms.} This is in contrast with UCB and Thompson Sampling. These algorithms stop pulling suboptimal arms once they verified their suboptimality. This prevents these algorithms to exploit structural information to the fullest extent as observed by \cite{pmlr-v54-lattimore17a} and \cite{russo2018learning}.  Pulling suboptimal arms in the presence of structural information has been shown to be an effective technique to reduce regret \citep{russo2018learning}. In \cite{russo2018learning},  the authors design a novel algorithm, which they name Information Directed Sampling (IDS), that aims to capture the structural information by balancing the regret of pulling an arm with its information gain. Our work is different from that of  \cite{russo2018learning} in three aspects. First, while we take a frequentist approach,  IDS is a Bayesian bandit algorithm. Second, we present a systematic way to incorporate structural information which allows us to design an algorithm whose regret matches the regret lower bound whereas IDS does not necessarily obtain minimal regret. Third, while IDS may have to update its decision in every round, our algorithm updates its strategy only in a logarithm number of rounds.
 
\section{Structured Multi-armed Bandit Problems}
\label{sec:mabp}

\textbf{\color{black}Setup.} In stochastic multi-armed  bandit problems, a decision-maker needs to select an arm from a finite set $X$ per round over the course of $T$ rounds. When arm $x$ is pulled in round $t$, the learner earns a nonnegative random reward $R_t(x)=r \in \setr$ with probability $\optp(r, x)$. \footnote{To ease the exposition, we assume that the set of all potential rewards $\setr$ is common to all arms and is discrete. Extending the results in this paper to a setting in which the  set of rewards $\setr(x)$ may depend on the arm $x$ is trivial but requires more burdensome notation.
} The rewards $R_t(x)$ are assumed to be independent across all arms $x\in X$ and rounds $t \in [T]$.
The learner's goal is to maximize the total expected reward over the course of $T$ rounds by pulling the arms. The reward distributions are unknown to the decision-maker beyond the fact that $\optp=(\optp(r,x))_{r\in\setr,\,x\in X}\in\mathcal P$, where $\mathcal P$ is a closed convex subset with non-empty interior.
The compact set $\mathcal P$ captures the structural  information concerning the unknown reward distribution and allows us to have a unified representation of the structural information discussed previously.  
For instance, in the context of revenue management problem where each arm $x$ corresponds to a price, the learner knows ahead of time that $P(r= 0, x) \ge P(r= 0, x')$ for any $x\geq x'$; that is,
$\mathcal P=\tset{Q\in \mathcal P_{\Omega}}{ Q(r= 0, x) \ge Q(r= 0, x')~\text{for any $x\geq x'$} }$, where $\mathcal P_\Omega$ is the set of all possible reward distributions:
\begin{align}\label{eq:p_omega}
  \mathcal P_\Omega\defn\set{Q \in \real_+^{\abs{\setr}\times \abs{X}}}{\textstyle\sum_{r\in \setr} Q(r, x)=1\quad \forall x\in X}\,.
\end{align}
As another example, if the learner is aware of a lower bound $\ell_x$ and an upper bound $u_x$ on the expected reward of arm $x\in X$, then the compact set $\mc P$ can be written as $\tset{Q\in P_{\Omega}}{\sum_{r\in \setr}r Q(r, x)\in [\ell_x, u_x]\quad \forall x\in X}$.
We present several additional examples in Section \ref{sec:struct-band-ex}.

{\textbf{\color{black} Regret.}} Consider a policy $\pi$ used by the learner.
Let $x_t^\pi$  be the arm selected by the policy $\pi$ in round $t$. This selection is based on previously selected arms and observations; that is, the random variable $x_t^\pi$ is $\mathcal F^\pi_t$-measurable, where $\mathcal F^\pi_t$ is the $\sigma$-algebra generated by $(x_1^\pi, R_1(x_1^\pi), \dots, x_{t-1}^\pi, R_{t-1}(x_{t-1}^\pi))$.
Let $\Pi$ be the set of all policies whose arm selection rules in any round $t\in [T]$ is $\mathcal F^\pi_t$-measurable. 
The performance of the policy $\pi\in \Pi$ over  $T$ rounds is quantified as its  \emph{regret}, defined below, which is the gap between the expected reward of policy $\pi$ and the maximum expected reward of an omniscient learner:
\begin{equation}
  \label{eq:regret}
  \reg(T, \optp) \defn T \max_{x\in X} \,\textstyle\sum_{r\in \setr} r \optp(r, x) - \sum_{t=1}^T \E{}{\sum_{r\in \setr} r \,  \optp(r, x_t^\pi)}\,. 
\end{equation}
Here, the expectation is with respect to any potential randomness in the bandit policy $\pi$.
For any reward distribution $\optp \in \mathcal P$, let $x^\star(\optp ) \in \arg\max_{x\in X}\,\textstyle \sum_{r\in \setr} r \optp(r, x)$ denote the set of optimal arms.
The per-round expected reward of an optimal arm is 
\(
  \rew^\star(\optp) := \textstyle\max_{x\in X} \sum_{r\in \setr} r \optp(r, x).
\)
Then, the first term in the regret of policy $\pi$ can be written as $T\cdot \rew^\star(\optp)$.
Note that pulling any suboptimal arm in $\tilde X(\optp)\defn X\backslash x^\star(\optp)$ offers less expected reward than pulling arms in $x^\star(\optp)$. We will denote the suboptimality gap of arm $x$ with
\(
\textstyle\Delta(x, \optp) := \rew^\star(\optp)-\sum_{r\in \setr} r \optp(r, x).
\)
{(\color{black} Recall that the reward distribution $P$ and hence suboptimality gap of any arm $x$ is unknown to the decision-maker.)}
If a policy keeps pulling suboptimal arms in every round, it  will suffer a large regret which is linear in $T$. 
A sensible bandit policy $\pi$ must thus manage to keep a small regret uniformly over all $\mathcal P$.  
Definition \ref{def:good} formalizes this notion, which was first introduced by \citet{lai1985asymptotically}.

\begin{definition}[Uniformly Good Policies] \label{def:good}
  A policy $\pi\in \Pi$ is  uniformly good if for all $\alpha>0$ and for any reward distribution $\optp\in \mathcal P$, we have
  \(
    \limsup_{T\to\infty} ~\tfrac{\reg(T, \optp)}{T^\alpha} = 0
  \)
  when the number of rounds $T$ tends to infinity. 
\end{definition}

It is generally impossible for uniformly  good policy $\pi\in \Pi$ to obtain a zero regret. Indeed, in most cases, an online policy must balance exploiting the arms that are optimal given the current information and exploring seemingly suboptimal arms.
Fortunately,  the total regret caused by exploring suboptimal arms can be kept relatively small.
As we discuss in  Section \ref{sec:regret-lower-bound}, the regret lower bounds for  bandit problems indicate that the regret is expected to grow logarithmically in $T$.

\subsection{\color{black}Examples of Structured Bandits}
\label{sec:struct-band-ex}

{\color{black}In this section, we present two sets of examples. In the first set of examples,  we show how to present well-known structural information using our framework. In the second example, we present two novel structural bandit classes, which can only be presented in our framework; that is, they 
cannot be represented in settings studied in prior work (e.g., \cite{filippi2010parametric, lattimore2014bounded, combes2017minimal}). }

\begin{running}[{\color{black}Representing well-known structural information in our framework}]
  \label{ex:structured-problems-1}
\phantom{}
\begin{enumerate}
  \setlength\itemsep{0em}
\item \textit{Generic bandits.}  Assume that the reward of any arm $x$ is drawn from an  arbitrary distribution with the discrete  support of $\setr$. In this case, the set $\mathcal P$ coincides with $\mathcal P_\Omega$, defined in Equation \eqref{eq:p_omega}.

\item \textit{Separable bandits.}  Assume that the reward of arm $x$ can be any distribution in $\mathcal P_{x}$ where $\mathcal P_{x}$, $x\in X$, is a closed  feasible set for the reward distribution of arm $x$. Then, the set $\mathcal P$ coincides with
  \(
    \textstyle \prod_{x\in X}\mathcal P_x\,.
  \)
 For example, if $\mathcal P_x$ is $\tset{Q \in \real_+^{\abs{\setr}}}{\sum_{r\in \setr} Q(r)=1}$ for all arms $x$, we have $\mathcal P=\mathcal P_\Omega$.
\item \textit{Bernoulli Lipschitz bandits.} Assume that the reward distribution for all the arms is Bernoulli, i.e., $\setr=\{0, 1\}$,  and the expected reward for each arm $x\in X\subseteq \Re$ is a Lipschitz function. Then, the set $\mathcal P= \mathcal P_{\mathrm{Lips}}$ is given as
  \(
  \tset{Q\in \mathcal P_\Omega}{%
    Q(1, x) - Q(1, x') \leq L \cdot d(x, x') ~~ \forall x\neq x' \in X      
    }\,,
  \)
  where $d(x, x')$  is a symmetric distance function between  arms $x$ and $x'$ and $L\geq 0$ is the Lipschitz constant.
\item {\color{black}\textit{Parametric bandits.} Consider a bandit problem in which arm $x$ promises expected reward $\mu_x$ where the vector $(\mu_x)_{x\in X}$ is only known to belong to a convex set $C$. More precisely, we assume that  the reward distributions belong to the following set $  \mathcal P_{\mathrm{param}}(\delta)$:
  \[
  \mathcal P_{\mathrm{param}}(\delta) \defn \set{Q\in \mathcal P_\Omega}{\mu \in C, ~ \textstyle \left(\sum_{x\in X}(\sum_{r\in\setr} r Q(r, x)) -\mu_x\right)^2 \leq \delta }.
\]
Observe that by setting $\delta$ to zero  and $C=\prod_x C_x$ in  definition of $\mathcal P_{\mathrm{param}}(\delta)$, we recover the parametric structured bandit model in \cite{lattimore2014bounded}, where it is assumed that $\mu_x\in C_x \defn \set{\mu_x(\theta)}{\theta \in \Theta}$.\footnote{{\color{black} For our theoretical result to hold, we need to have $\delta> 0$. This ensures that the set $\mathcal P_{\mathrm{param}}(\delta)$ has an interior which is a required assumption for our theoretical result to hold. However, as we show in Section~\ref{sec:numer-exper}, we can still use our designed algorithm for the case of $\delta =0$, i.e.,  $\mathcal P_{\mathrm{param}}(0)$, obtaining sublinear regret. } 
} 
Note that when $\mu_x(\theta)$ can be written as $c_x^T \theta$ for some vectors $c_x$, we recover linear bandits as in \cite{filippi2010parametric}. We may also recover convex bandits simply by setting $\delta=0$ and by considering
\[
  C = \set{\mu}{\forall x_1, x_2\in X, \exists g(x_1)\in \real^n~~ \mu_{x_2} \geq \mu_{x_1} + g(x_1)\tpose(x_2-x_1)}.
\]
The previous set captures a convex structure between the arms by forcing a subgradient inequality to hold between any two arms $x_1$ and $x_2$ with $g(x_1)\in\real^n$ the subgradient of $\mu_{x_1}$ at arm $x_1$; see also \cite[Section 6.5.5]{boyd2004convex}.
}
\end{enumerate}
\end{running}

{\color{black}
  \begin{running}[Representing novel structural information in our framework]  \label{ex:structured-problems-2}
    \phantom{}
  \begin{enumerate}
      \item \textit{Dispersion bandits.} In many applications, while we do not have accurate information about the expected reward for different arms, we may have some side information about the level of dispersion of the reward distribution of the arms. Consider the following structural information:
    \begin{align}\textstyle\mathcal P_{\text{dis}}=\Big\{Q\in \mathcal P_\Omega:\tfrac{\sum_{r\in\setr } r^2Q(r, x)}{\sum_{r\in\setr }r Q(r, x) }\le \gamma(x) \text{~~ for any $x\in X$}\Big\}\,.
    \label{eq:disp_bandits}\end{align}
    Here, $\gamma(x) > 0$ bounds the second-moment to mean ratio of the reward distribution of arm $x$. Note that existing techniques in the literature cannot handle this  structural information as this class of distributions is not mean parameterized, as it is the case for parametric bandits.
    Remark that we can decompose the second-moment to mean ratio as
    \[
    \textstyle\tfrac{\sum_{r\in\setr } r^2Q(r, x)}{\sum_{r\in\setr }r Q(r, x) } = \sum_{r\in\setr }r Q(r, x)+ \tfrac{\sum_{r\in\setr } (r-\sum_{r\in\setr }r Q(r, x))^2 Q(r, x)}{\sum_{r\in\setr }r Q(r, x) }\,,
    \]
    where the first term (i.e., $\sum_{r\in\setr }r Q(r, x)$) is the mean reward of arm $x$  and the second term (i.e., $\tfrac{\sum_{r\in\setr } (r-\sum_{r\in\setr }r Q(r, x))^2 Q(r, x)}{(\sum_{r\in\setr }r Q(r, x)) }$) is the variance-to-mean ratio for the reward of arm $x$. 
    Roughly speaking, under the aforementioned structural information, arms with large (respectively small) average reward are expected to have small (respectively large) variance-over-mean. As we show in our numerical studies in Section \ref{sec:numer-exper}, such intuitive structural information allows the decision-maker to identify optimal arms quickly as those arms enjoying little regret dispersion also promise large expected rewards.
    \item \textit{Divergence bandits.} It is often the case that arms which are ``close'' to each other, as measured by a given metric $d$ on $X$, can be expected to have ``similar'' reward distributions. Consider hence the following structural information
     \[
  \mathcal P_{\text{diverg}}=\tset{Q\in \mathcal P_\Omega}{%
    D(Q(x) , Q(x')) \leq  d(x, x') ~~ \forall x\neq x' \in X      
    }\,,
  \]
  where $D(M , M')$ is a convex divergence metric between reward distributions while $d(x, x')$ a distance function between arms $x$ and $x'$. Examples of such convex divergence metrics include $f$-divergences, total variation and Wasserstein distances, as well as many statistical distances such as the Kolmogorov-Smirnov distance.
  Note that existing techniques in the literature cannot handle this structural information as again this class of distributions is not mean parameterized. We remark that divergence bandits are distinct from Lipschitz bandits which only impose conditions on the mean reward between arms whereas divergence bandits impose conditions on the arms' reward distribution.
\end{enumerate}    
  \end{running}
}

\section{Regret Lower Bound}
\label{sec:regret-lower-bound}
The regret for bandit problems on which some structural information is imposed should get smaller as more information is provided on the unknown reward distribution $P$.
Nevertheless, as we will show, {\color{black} in most cases}, the regret is still expected to scale logarithmically with the number of rounds $T$.\footnote{\color{black} In some rare instances, the regret lower bound does not scale with $T$. In these instances, roughly speaking,  the convex set $\mathcal P$ eliminates  the optimality of sub-optimal arms. More precisely, in these instances, the set of deceitful reward distributions, defined in Equation \eqref{eq:problematic-distributions}, is empty, and hence sub-optimal arms can be simply ruled out. We show for these rare instances, the regret of our algorithm DUSA does not scale with $T$, and is finite. See Remark \ref{rem:proof:nondeceitful-bandits}.}
The precise logarithmic rate with which the regret needs to grow can be quantified  via an optimization problem.
To do so, we focus on the set of \emph{deceitful} reward distributions which makes learning the optimal arm challenging. 
Consider a multi-armed bandit problem with a given reward distribution $\optp$. For any arm $x'$, we define 
\begin{equation}
  \label{eq:problematic-distributions}
  \prob(x', \optp) \defn \set{Q\in \mc P}{   
       Q(x^\star(\optp))=\optp(x^\star(\optp)),~\textstyle\sum_{r \in \setr} r Q(r, x')  > \sum_{r \in \setr} r Q(r, x^\star(\optp))
  }  
\end{equation}
as sets of deceitful reward distributions relative to the actual but unknown reward distribution $\optp$. Deceitful reward distributions are precisely those reward distributions in  $\mathcal P$ that behave identical to the true distribution $P$ when  the unknown optimal arm $x^\star(\optp)$  is played, i.e., see the first constraint in \eqref{eq:problematic-distributions}, but deceivingly have a better arm  $x'$ to play as per the last constraint in \eqref{eq:problematic-distributions}.  Put differently, the reward distribution for the  actual optimal arm $x^\star(\optp)$ is the same in $\optp$ as for any deceitful reward distribution $Q\in \prob(x', \optp)$. For the deceitful  reward distribution $Q$, however, the expected reward of arm $x'$ is greater than that of arm $x^\star(\optp)$.

To obtain a low regret, a policy should be able to distinguish the actual reward distribution $\optp$ from its deceitful reward distributions based on the past observed rewards.
Later, we discuss  the extent to which this  can be quantified in terms of the information distance between reward distribution $\optp$ and its deceitful distributions.
Specifically, the distribution $\optp$ can be distinguished from its deceitful reward distributions  if 
the information distance between the distribution $\optp$ and its deceitful reward distributions is not too ``small.''
We will make this statement formal shortly hereafter.
Before that, we recall that the information distance between any two positive measures $M$ and $M'$ on a set $\setr$ is characterized as \begin{align}I(M', M)  \defn
    \textstyle\sum_{r\in \setr} M'(r) \log\left(\tfrac{M'(r)}{M(r)}\right) - M'(r) + M(r)]\label{eq:I_def}\end{align}
    if $M'\ll M$, and $+\infty$ otherwise. Here, $M'\ll M$ denotes the implication $M(r)=0 \! \implies \!M'(r)=0$ for all $r\in \setr$. Finally, we use the shorthand $M>0$ for the implication $M(r)>0$ for all $r\in \setr$.
For those unfamiliar, all properties relevant to this paper concerning the information distance are conveniently presented in Appendix
\ref{sec:information-topology}.

We are now ready to present the well-known result stating a lower regret bound as an optimization problem in terms of information distances between $\optp$ and its deceitful  distributions. 

\begin{proposition}[Regret Lower Bound]
  \label{lemma:lower_bound}
  Let $\pi$ be any uniformly good policy.
  For any reward distribution $0<\optp\in \mathcal P$, we have that
    $\liminf_{T\to\infty} ~\tfrac{\reg(T, \optp)}{\log(T)}\geq C(\optp)$, 
  where the regret lower bound function is characterized as
  \begin{align}
    \label{eq:silo2}
    \begin{split}
      C(\optp):=  \inf_{\eta\geq 0} & ~\textstyle\sum_{x\in \tilde X(P)} \eta(x) \Delta(x, \optp)\\
      \st & ~1 \leq  \dis(\eta, x, \optp)~~ \forall x\in \tilde X(\optp)\,.
    \end{split}\\
    \intertext{Here, if the set $\prob(x', \optp)$  is non-empty, the distance function $\dis$   is defined as follows}                         
    \label{eq:inner-optimization}
    \begin{split}
      \dis(\eta, x', \optp):=\inf & ~\textstyle \sum_{x\in \tilde X(\optp)}  \eta(x) I(\optp(x), Q(x)) \\
      \st & ~Q \in \prob(x', \optp)\,.
    \end{split}
  \end{align}
  Otherwise, we take $\dis(\eta, x', P)=\infty$.
\end{proposition}

Proposition \ref{lemma:lower_bound} is a straightforward corollary \citep[Theorem 1]{combes2017minimal} of a much more general result proven by \citet[Theorem 1]{graves1997asymptotically} in a Markovian setting. Thus, we do not provide its proof here. A  special case of Proposition \ref{lemma:lower_bound} for discrete bandit model sets $\mathcal P$ was proven much earlier by \citet{rajeev1989asymptotically}. As indicated in Proposition \ref{lemma:lower_bound} by the infima, we make no assumptions whether or not either minimization problems \eqref{eq:silo2} or \eqref{eq:inner-optimization} are attained. This will be a source of technical difficulty which we address. 

 The decision variable $\eta$ in the lower bound \eqref{eq:silo2} can be interpreted as a proxy for the logarithmic rate $N_{T+1}(x) \approx \eta(x) \cdot \log(T)$ with which any uniformly good policy must pull each suboptimal arm $x\in \tilde X(\optp)$ during the first  $T$ rounds in order to have a shot at suffering minimal regret. Here,  $N_{T+1}(x)$ denotes the number of rounds in the first $T$ rounds that arm $x$ is pulled. 
The logarithmic rate with which the regret is  accumulated is given  as the objective function of the lower bound \eqref{eq:silo2}. Recall that $\Delta(x, \optp)$ is  the suboptimality gap of arm $x$ under reward distribution $\optp$.

 The constraint in the lower bound \eqref{eq:silo2} can be regarded as a condition on the amount of information collected on the suboptimal arms for any policy to be uniformly good.
In fact, the constraint $\dis(\eta, x', \optp)\geq 1$ for $x'\in \tilde X(\optp)$ admits an interpretation in terms of optimal hypothesis testing. 
In order to have a uniformly small regret, any policy must indeed identify the optimal arm with high probability. 
To do so, \citet[Theorem 1]{graves1997asymptotically} argue that any uniformly good policy must be able to distinguish the actual distribution $\optp$ from its deceitful reward distributions $\prob( x, \optp)$ for all $x\in \tilde X(\optp)$, using some optimal hypothesis test. {\color{black} Remark that the regret lower bound \eqref{eq:silo2} adapts to the considered structured bandit problem, as characterized by its associated set $\mc P$, only through these information (optimal hypothesis testing) constraints.} The lower bound stated in Proposition \ref{lemma:lower_bound} and its properties will be of central importance in this paper.    
Thus, in the next section,  we will  have a closer look at the distance function on which the lower bound $C(\optp)$ is critically based.
We finish this section by presenting our main result informally. 

\begin{theorem}[Main Result (Informal)] For any $0<\epsilon< \tfrac{1}{|X|}$, there exists a computationally efficient bandit policy $\pi$ whose expected regret bound is essentially optimal; that is, 
  \(
    \textstyle\limsup_{T\to\infty} ~\tfrac{\mathbb{E}[\reg(T, \optp)]}{\log(T)}\leq (1+\epsilon)C(\optp)+O(\epsilon). 
  \)
\end{theorem}

\subsection{Partitioning of Suboptimal Arms}
 \label{sec:part}

The constraints in the regret lower bound problem, presented in Proposition \ref{lemma:lower_bound},  require  that for any suboptimal arm $x\in \tilde X(\optp)$, the information distance between $\optp$ and any of its deceitful distributions $\prob(x, \optp)$ to be  at least one.
However, for some suboptimal arms, the set of deceitful distributions may  be empty.
Let 
\begin{equation}
  \label{eq:maximum-reward}
  \begin{array}{rl}
    \rew_{\max}(x, \optp) \defn \max & \sum_{r \in \setr} r Q(r, x) \\[0.5em]
    \st & Q\in\mathcal P, ~ Q(x^\star(\optp))=\optp(x^\star(\optp))
                                    
  \end{array}
\end{equation} 
be the maximum reward an arm $x$ can yield given that the first deceitful condition $Q(x^\star(\optp))=\optp(x^\star(\optp))$ hold.
Obviously, when one has $\rew_{\max}(x, \optp)\le  \rew^\star(\optp)$, the  second requirement, i.e.,  $ \textstyle\sum_{r \in \setr} r Q(r, x)  >  \sum_{r \in \setr} r Q(r, x^\star(\optp)) = \rew^\star(\optp)$,  cannot be satisfied for any deceitful reward distribution $Q$.
Based on this observation, we partition  the suboptimal arms into the following two groups.

\begin{itemize}[leftmargin=*]
\item \textit{Non-deceitful arms  $\xn(\optp)$.} A suboptimal arm $x$ is non-deceitful, i.e., $x\in \xn(\optp)$,  if $\rew_{\max}(x, \optp)\le  \rew^\star(\optp)$.
  For any  non-deceitful arm $x\in \xn(\optp)$, the set of associated deceitful distributions $\prob(x, \optp)$ is empty.
  Consequently, non-deceitful arms do not impose any conditions on the rates with which suboptimal arms need to be played in the regret lower bound  \eqref{eq:silo2}; we have indeed $\dis(\eta, x, \optp)=+\infty$ for any logarithmic exploration rate $\eta$.
  That is not to say that these arms are not to be explored.
  It may be beneficial to play non-deceitful arms to gain sufficient information on the set of arms we describe next.

\item \textit{Deceitful arms $\xd(\optp)$.} Suboptimal arm $x\in \xd(\optp)$ if $\rew_{\max}(x, \optp)>  \rew^\star(\optp)$. For these arms, the set of  associated deceitful distributions $\prob(x, \optp)$ is non-empty and the distance $\dis(\eta, x, \optp)$ between reward distribution $P$ and these deceitful distributions is finite. As a result, 
the learner has to gather enough information by exploring suboptimal arms  to effectively reject their associated deceitful distributions.
\end{itemize}

 In the following, we present an example to further clarify our notion of deceitful and non-deceitful arms.
  Consider a separable Bernoulli bandit problem with potential rewards $\setr=\{0, 1\}$ and only two arms $X=\{a, b\}$.
  Assume that the reward distribution of arm $a$ is such that $\optp(0, a)\geq \frac{2}{5}$ and the reward distribution 
  of  arm $b$ can be arbitrary.
  That is, we have a separable bandit problem with reward distribution in the set
  \begin{align}\label{eq:p_in_example}
    \mathcal P = \mathcal P_1\times\mathcal P_2 \defn \tset{Q \in \real_+^{{2}}}{\textstyle\sum_{r\in \{0,1\}} Q(r)=1,\, Q(0)\ge \frac{2}{5}} \times \tset{Q \in \real_+^{{2}}}{\textstyle\sum_{r\in \{0,1\}} Q(r)=1}.
  \end{align}
  Clearly, the maximum rewards for each arm are given as $\rew_{\max}(a, \optp)=1-\frac{2}{5} = \frac{3}{5}$ and $\rew_{\max}(b,\optp)=1$ and are independent of $\optp$. 
  Now, let the actual reward distributions be $\optp(a) = [1/2, 1/2]$ and $\optp(b)=[1/5, 4/5]$. That is, the probability of receiving reward zero under arm $a$ (respectively arm $b$) is $1/2$ (respectively  $1/5$). Then, it is easy to see that the optimal arm is  $b$, i.e., $x^{\star}(\optp) =b$, and $\rew^{\star}(\optp) = 4/5$.  Observe that $\rew_{\max}(a, \optp)< \rew^{\star}(\optp)$ and as a result, arm $a$ is not deceitful.
  To make the  concept of non-deceitful arms more tangible, assume that the learner keeps playing the optimal arm $b$.
  Eventually, by just playing the optimal arm, he will learn that the non-deceitful arm $a$ is not optimal.
  This is so because his empirical average reward of the optimal arm converges to $4/5$, which is greater than the maximum reward that one expects to obtain from suboptimal arm $a$. (Recall that $\rew_{\max}(a, \optp) = \frac{3}{5}$ for any $\optp \in \mathcal P$.)
  This implies that for this example, the lower bound, $C(\optp)$, presented in Proposition \ref{lemma:lower_bound}, is zero.

The following proposition characterizes the distance function $\dis$ for deceitful arms.  

\begin{proposition}[Distance Function for Deceitful  Arms]
  \label{prop:continuity-projection}
  Consider any reward distribution $\optp\in \mathcal P$, deceitful arm $x'\in \xd(\optp)$, and positive exploration rate $\eta\geq 0$. Then,
\begin{equation}
  \label{eq:dist-function-continuous}
  \begin{array}{r@{~~}l@{~}r@{~~}l}
    \dis(\eta, x', P)= \min_{Q\in \mathcal P} & \sum_{x\in \tilde X(P)} \eta(x) I(\optp(x), Q(x)) \\[0.5em]
    \st & Q(x^\star(\optp))=\optp(x^\star(\optp)),~ \sum_{r \in \setr} r Q(r, x')  \geq \rew^\star(\optp).
  \end{array}
\end{equation}
\end{proposition}

The proof of this result is presented  in Appendix \ref{sec:Continuity}. {\color{black} Proposition \ref{prop:continuity-projection} presents a full characterization of the  distance function $\dis$ for deceitful arms.  Recall that from Equation \eqref{eq:inner-optimization}, the distance function $\dis$ for any arm $x'$ is given by 
\(
      \dis(\eta, x', \optp)=\inf \set{\sum_{x\in \tilde X(\optp)}  \eta(x) I(\optp(x), Q(x))}{
      Q \in \prob(x', \optp)}.
\)
Proposition \ref{prop:continuity-projection} shows that when arm $x'$ is deceitful,  (i) the infimum  in the previous optimization problem is  achieved and is finite, and (ii) its constraint (i.e., $Q \in \prob(x', \optp)$) can be replaced by the following two constraints $ Q(x^\star(\optp))=\optp(x^\star(\optp))$ and  $\sum_{r \in \setr} r Q(r, x')  \geq \rew^\star(\optp)$.}
Hence, the partitioning of arms into two types alleviates the technical difficulty caused by potential non-attainment of the infimum \eqref{eq:inner-optimization}.
We finish this section by revisiting some of our  examples in Section \ref{sec:struct-band-ex}. 

\begin{running}[Regret Lower Bound]
Using Proposition \ref{prop:continuity-projection}, we characterize  the lower bound and distance functions ($C$ and $\dis$) for  generic, separable, and Lipschitz Bernoulli bandits.
  \begin{enumerate}
    \setlength\itemsep{0em}
  \item[1--2.] \textit{Generic and separable bandits.} For generic, and more generally separable bandits, the lower bound presented in Proposition \ref{lemma:lower_bound} coincides precisely with the results first proven in the seminal work of \citet{lai1985asymptotically}. Recall that for separable bandits, the set $\mc P$ decomposes as the Cartesian product $\prod_{x\in X} \mc P_{x}$.
    This implies that for any reward distribution $\optp\in \mc P$ and suboptimal arm $x'\in \tilde X(\optp)$, we get 
  \(\rew_{\max}(x', \optp) = \max \set{\textstyle\sum_{r\in \setr} r Q(r)}{Q\in \mc P_{x'}}\). For any deceitful arm $x'$, we  define a distribution $Q^\star_{x'}(x) = \optp(x)$ for all $x\in X\backslash \{x'\}$ and let
    \begin{equation*}
      \begin{array}{rl}
        Q^\star_{x'}(x')\in \arg \min_{Q} \set{ I(\optp(x'), Q)}{Q \in \mc P_{x'}, ~ \sum_{r \in \setr} r Q(r) \geq \rew^{\star}(\optp)}.
      \end{array}
    \end{equation*}
    The constructed reward distribution $Q^\star_{x'}$ can be interpreted as the worst deceitful reward distribution for  deceitful arm $x'$. That is, it can be verified that $Q^\star_{x'}$ is the optimal solution to problem \eqref{eq:dist-function-continuous}. 
Also note that  the constraint $\dis(\eta, x', \optp)\geq 1$ in the regret lower bound problem 
     demands that $\eta(x') \geq \tfrac{1}{I(\optp(x'), Q^\star_{x'}(x'))}$ for each deceitful arm. This leads to  $\eta^\star(x', \optp) = \tfrac{1}{I(P(x'), Q^\star_{x'}(x'))}$ for $x'\in \xd(P)$ and $\eta^\star(x', \optp) = 0$ for $x'\in \xn(P)$, which is the minimizer in   problem \eqref{eq:silo2} defining the lower regret bound for all suboptimal arms $x' \in \tilde X(\optp)$.

  \item[3.] Bernoulli Lipschitz  bandits. For Bernoulli Lipschitz  bandits, the lower bound presented in Proposition \ref{lemma:lower_bound} coincides precisely with the results first proven in \citet{pmlr-v35-magureanu14}.
    It can be readily verified that here $\rew_{\max}(x', \optp) = \min(\max(\setr), \rew^\star(\optp)+L\cdot d(x', x^\star(\optp))$.
    Hence, all suboptimal arms are deceitful when $\rew^\star(\optp)<\max(\setr)$;
    otherwise if $\rew^\star(\optp)=\max(\setr)$,  then trivially all suboptimal arms are non-deceitful and consequently $C(\optp)=0$. 
To see why note that arm $x'$ is deceitful if $\rew_{\max}(x', \optp)> \rew^\star(\optp)$. Then,  when $\rew^\star(\optp)<\max(\setr)$, we certainly have 
\(\rew_{\max}(x', \optp)=\min(\max(\setr), \rew^\star(\optp)+L\cdot d(x', x^\star(\optp)) > \rew^\star(\optp)\,, \)
which shows that any suboptimal arm $x'$ is deceitful. On the other hand, when $\rew^\star(\optp)= \max(\setr)$, 
\(\rew_{\max}(x', \optp)=\min(\max(\setr), \rew^\star(\optp)+L\cdot d(x', x^\star(\optp))= \max(\setr) \le \rew^\star(\optp)\,, \)
which shows that any suboptimal arm $x'$ is not deceitful.
    Let for any suboptimal arm $x'$  
\[
      Q^\star_{x'}(1, x)  = \max(\optp(1, x), \rew^\star(\optp)-L\cdot d(x, x')),\quad Q^\star_{x'}(0, x)  = 1- Q^\star_{x'}(1, x),
  \]
    for all $x\in X$. Remark that $Q^\star_{x'}(1, x)$ matches the true distribution $P(1,x)$ everywhere expect when $x$ is close to $x'$. In other words, $Q^\star_{x'}(1, x)$ is only distorted in a neighborhood around $x'$; see also Figure \ref{fig:lipschitz}.
    The  reward distribution $Q^\star_{x'}$ can  be interpreted as the worst deceitful reward distribution for arm $x'$ as we  have that $\dis(\eta, x', \optp) = \sum_{x\in \tilde X(P)}\eta(x) I(\optp(x), Q_{x'}^\star(x))$.
    By Proposition \ref{lemma:lower_bound}, therefore, the regret lower bound enjoys a (semi) closed form solution in the form of the linear  program
    \begin{equation}
      \label{eq:lipschitz:parametric-linear}
      \begin{array}{rl}
        \textstyle C(\optp) = \min_{\eta\geq 0} \set{\sum_{x\in \tilde X(P)}  \eta(x) \Delta(x, \optp)}{ 1\leq \textstyle\sum_{x\in \tilde X(\optp)}\eta(x) I(\optp(x), Q_{x'}^\star(x)) ~~ \forall x'\in  \tilde X(\optp)}.
      \end{array}
    \end{equation}
  \end{enumerate}
\end{running}
   
\begin{figure}[ht]
  \centering
  \includegraphics[height=4.5cm]{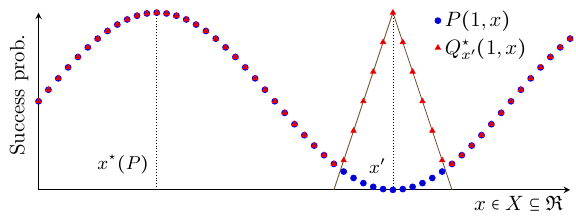}
  \caption{Worst-case deceitful reward distribution $Q^\star_{x'}$ for a deceitful arm $x'$ in a Bernoulli Lipschitz  bandit with reward distribution $P$ and optimal arm $x^\star(P)$. Here we let $d(x, x')=\abs{x-x'}$.} 
  \label{fig:lipschitz}
\end{figure}

\section{High-Level Ideas of Our Policy}\label{sec:high-level}

Our policy, which is called {DU}al {S}tructure-based {A}lgorithm (DUSA), is inspired by the lower bound formulation discussed in the previous section.
The lower bound problem \eqref{eq:silo2} suggests that to obtain minimal regret, each suboptimal arm $x$ in $\tilde X(\optp)$ should be explored at least $N_{t+1}(x)=\eta_{\epsilon}(x, \optp)\cdot \log(t)$ times. The logarithmic exploration rate $\eta_\epsilon(P)$ is an arbitrary $0<\epsilon$-suboptimal solution, i.e., $\eta_\epsilon(P)$ is a feasible solution to problem \eqref{eq:silo2} and achieves an objective value at most $\epsilon$ worse than $C(\optp)$, instead of an actual minimizer.  
We consider near optimal solutions to circumvent the technical difficulty that the infimum in \eqref{eq:silo2} may not be attained. We  discuss later that considering such near optimal exploration rates $\eta_\epsilon(P)$ is desirable even if actual minimizers do exist.

Following the previously discussed strategy is evidently not possible without knowing the true reward distribution $\optp$. Our idea is to keep track of the reward distribution of each arm  with the help of the empirical counterpart $P_t$ constructed based on the historical rewards observed before round $t$.
The policy then   ensures that the number of times each arm $x$ is pulled during the first $t$ rounds, which is denoted by $N_{t+1}(x)$,  is close to $\eta_{\epsilon}(x, P_t)\cdot \log(t)$ where the logarithmic rate $\eta_{\epsilon}(P_t)$ is an $0<\epsilon$-suboptimal solution of the following optimization problem
\begin{align}
\label{eq:lb2}
  \begin{split}
    C(P_t)=\inf_{\eta} & ~\textstyle\sum_{x\in \tilde X(P_t)} \eta(x) \Delta(x, P_t)\\
    \st & ~\eta\geq 0, ~1 \leq  \dis(\eta, x, P_t)~~ \forall x\in \xd(P_t).
  \end{split}
\end{align}

In the following, we first discuss why and when mimicking the regret lower bound leads to a good learning policy. This discussion leads to a mild assumption on the true reward distribution $P$. Afterward, we argue that in general  mimicking the regret lower bound, i.e., solving the semi-infinite problem \eqref{eq:lb2}, is  computationally expensive. Motivated by this, we present a computationally-efficient convex dual-based approach instead.

\subsection{When/Why Mimicking the Lower Bound Works} \label{sec:conti}
The proposed strategy  aims to play suboptimal arms at the logarithmic target rate $\eta_{\epsilon}(P_t)$ in any round $t$. Such a strategy  can  attain asymptotic optimality when  by converging the empirical distribution $P_t$ to the true reward distribution $P$, 
 (i) the empirical counterpart of the regret lower bound, i.e.,  $C(P_t)$, approaches the actual regret lower bound, i.e., $C(P)$, and (ii) the empirical (suboptimal) logarithmic target rate $\eta_{\epsilon}(P_t)$ converges to $\eta_{\epsilon}(P)$. To meet these requirements, at the reward distribution $P$,  the regret lower bound function $C(Q)$ should be continuous and the $\epsilon$-suboptimal exploration rates $\eta_\epsilon(Q)$ should admit a continuous selection at $P$. We note that $\eta_\epsilon(Q)$ may be a set-valued mapping\footnote{We indeed expect multiple $\epsilon$-suboptimal solutions to exist in problem (\ref{eq:lb2}).} and as a result, it is necessary to design a  selection rule that chooses one of the suboptimal solutions while ensuring that the designed rule is continuous in the reward distribution.

The non-uniqueness of $\epsilon$-suboptimal solutions, although at first glance perhaps perceived as an undesirable complication, is in fact a blessing in disguise. It indeed enables the existence of a continuous selection and echos our previous point that considering $\epsilon$-suboptimal solutions is beneficial as the exact minimizers $\eta^\star(Q)$ (should they exist) may fail to admit such continuous selection.

\begin{running}[Continued]
  \phantom{}
  \begin{enumerate}    
  \item[3.] \textit{Lipschitz Bernoulli bandits.}
    We have previously determined $\eta^\star(Q)$ in the context of Bernoulli Lipschitz Bandits as the minimizers in the parametric linear optimization problems stated in Equation \eqref{eq:lipschitz:parametric-linear}.
    However, it is well known \citep[Section 4.3]{bank1982parametric} that the minimizers in such parametric linear optimization problems are  upper semi-continuous but not necessarily lower semi-continuous mappings.
    In Figure \ref{fig:selection}, we attempt to illustrate the lack of continuous exploration rate selections from such upper semi-continuous mappings $\eta^\star(Q)$. In this figure, at the reward distribution $P$, the optimal exploration rate $\eta$ is not unique and can be any point in a vertical line.\footnote{Having multiple optimal solutions to the regret lower problem does not necessarily imply that we have multiple  optimal arms. To see why, consider an example with one optimal arm and two suboptimal arms whose expected rewards are identical. Suppose that the learner is aware of the fact that the expected reward of two arms is the same. In this case, the optimal solution to the regret lower problem may not be unique. However, the optimal arm is unique.} The illustrated optimal mapping $ \eta^\star(Q)$ is upper semi-continuous, but not lower semi-continuous. For this mapping, one cannot choose one of the optimal solutions at $P$ that leads to a continuous selection. However, this figure shows that this problem is alleviated by considering an $\epsilon$-suboptimal mapping instead. Here, the green area illustrates all the $\epsilon$-suboptimal solutions ($\eta_{\epsilon}(Q)$) and the red curve ($\eta'_{\epsilon}(Q)$) presents a particular continuous selection for the $\epsilon$-suboptimal mapping. 
  \end{enumerate}
\end{running}

The following mild assumption on the reward distribution $P$ is necessary  to guarantee the necessary  continuity properties.

\begin{assumption}\label{assumption}
 The reward distribution $P$ \begin{enumerate}
    {\color{black} \item 
 has a unique optimal arm, 
 \item satisfies 
 $
    \rew_{\max}(x, P)\neq \rew^{\star}(P)
 $ for any $x\in X$, and 
 \item is in the interior of $\mathcal P$.} 
  \end{enumerate}
  Here, $\rew_{\max}(x, P)$ is defined in Equation \eqref{eq:maximum-reward} and $\rew^{\star}(P) = \sum_{r \in \setr} r P(r, x^\star(\optp))$. Let $\mathcal P'\subseteq \mathcal P$ denote the set of all such distributions.
\end{assumption}

\begin{figure}[ht]
  \begin{subfigure}[t]{0.45\textwidth}
    \centering
    \includegraphics[height=4.5cm]{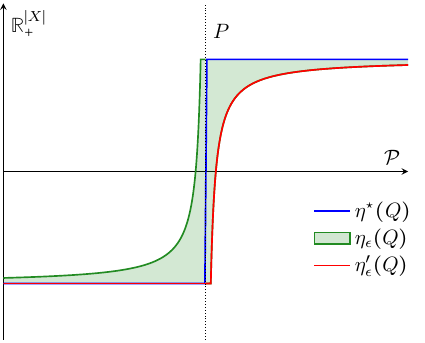}
    \caption{}
    \label{fig:selection}
  \end{subfigure}%
  \begin{subfigure}[t]{0.5\textwidth}  
    \centering
    \includegraphics[height=4.5cm]{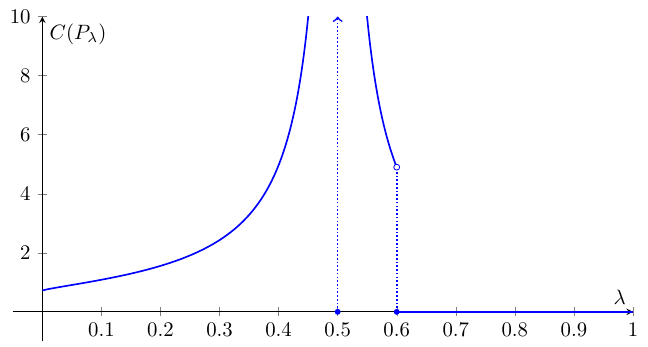}
    \caption{}
    \label{fig:discontinuity}
  \end{subfigure}  
  \caption{(a) The mapping $\eta^\star(Q)$ between a reward distribution $Q$ and the set of all optimal exploration rates is typically merely upper semi-continuous but not necessarily lower semi-continuous.
    Even though the optimal exploration rates may exist, there exists no continuous selection at the reward distribution $P$ where the optimal exploration rates are not uniquely defined. Considering $\epsilon$-suboptimal solutions for $\epsilon>0$ alleviates this problem. 
    (b) The  regret lower bound $C(P_\lambda)$ for the separable bandit problem discussed in Section \ref{sec:conti}. 
    }
\end{figure}

The first condition in Assumption \ref{assumption} requires the optimal arm under the true reward distribution to be unique. The second condition, i.e., $\rew_{\max}(x, P)\neq \rew^{\star}(P)$ for any $x\in X$, implies that there exists a neighborhood around the true reward distribution $P$ such that  for any distribution $Q$ within this neighborhood, the sets of deceitful arms $\xd(Q)$, non-deceitful arms $\xn(Q)$, and optimal arm $x^{\star}(Q)$ are the same. Recall that the suboptimal arm $x\in \xd(P)$ if $\rew_{\max}(x, P)> \rew^{\star}(P)$ and the suboptimal arm $x\in \xd(P)$ if $\rew_{\max}(x, P)\le  \rew^{\star}(P)$. As we show 
at the end of this section via counter example, both of these conditions are essential for the regret lower bound $C(P)$ as well its (suboptimal) solutions $\eta_{\epsilon}(P)$  to be continuous.
 The last  assumption that $P\in \interior (\mathcal P)$ ensures that our estimator $P_t$ eventually realizes in the set of all potential reward distributions $\mathcal P$.
 That is, we have asymptotically $\lim_{t\to\infty} \Prob{P_t\in \mathcal P} = 1$. We believe that this last assumption is not critical in that is tied to the naive empirical estimator $P_t$ that does not exploit the structural information. As a more sophisticated estimator would make the regret analysis of our policy far more complicated, we do not explore this approach in this work and leave it as a future direction.
 
 \begin{proposition}[Continuity of Regret Lower Bound]
  \label{prop:continuity-regret-lower-bound}
  Consider the regret lower bound  problem  \eqref{eq:silo2} characterizing the regret lower bound function. Recall that  $\mathcal P'\subseteq \mathcal P$ is the set of all  distributions that satisfy Assumption \ref{assumption}. Then,{\color{black}
  \begin{enumerate}
      \item the regret lower bound function $C(Q)$ is continuous on $\mathcal P'$, and 
      \item there exists an $0<\epsilon$-suboptimal rate selection $\eta^c_{\epsilon}(Q)$ which is continuous on $\mathcal P'$.
  \end{enumerate}}
  
\end{proposition}

The proof of this result, along with other continuity results, is presented  in Appendix \ref{sec:Continuity}. We now argue that the first two conditions in Assumption \ref{assumption} are necessary by considering a simple example. Consider  a separable Bernoulli bandit problem with only two arms $X=\{a, b\}$. The reward distribution of the first arm $a$ is such that $P(0, a)\geq \frac{2}{5}$ and the reward distribution 
of the second arm $b$ can be arbitrary. That is, $\mathcal P$ is given in Equation \eqref{eq:p_in_example}. 
For any $\lambda \in [0, 1]$, we define parameterized  reward distributions $P_\lambda(a) = [\frac{1}{2}, \frac{1}{2}]$ and $P_\lambda(b)=[1-\lambda, \lambda]$. That is, the probability of receiving reward zero under arm $a$ (respectively arm $b$) is $\frac{1}{2}$ (respectively  $1-\lambda$). Then, it is easy to see that the optimal arms and reward are given by
\[
  x^\star(P_\lambda)=
  \begin{cases}
    \{b\} & {\mathrm{if}~} \lambda \in (\frac{1}{2}, 1], \\
    \{a, b\} & {\mathrm{if}~} \lambda = \frac{1}{2}, \\
    \{a\} & {\mathrm{if}~} \lambda \in [0, \frac{1}{2})
  \end{cases}
  \hspace{2em}
  {\mathrm{and}}
  \hspace{2em}
  \rew^{\star}(P_\lambda)=
  \begin{cases}
    \lambda & {\mathrm{if}~} \lambda \in (\frac{1}{2}, 1], \\
    \frac{1}{2} & {\mathrm{if}~} \lambda \in [0, \frac{1}{2}].
  \end{cases}
\]
Observe that at $\lambda =\frac{1}{2}$, the optimal arm is not unique. That is, the first condition of Assumption \ref{assumption} is violated. 
Furthermore, at $\lambda=\frac{3}{5}$, we have $\rew_{\max}(a, P)=\rew^{\star}(P)$. That is, the second condition of Assumption \ref{assumption} is violated. Next, we show that these violations result in discontinuity of the regret lower bound. 
  
 For the parameter values $\lambda\in [\frac{3}{5}, 1]$, we have that $\rew^{\star}(P_\lambda)\geq \rew_{\max}(a, P_\lambda) =\frac{3}{5}$ and hence $C(P_\lambda)=0$.
Likewise, when $\lambda=\frac{1}{2}$,  both arms are optimal, and once again we have that $C(P_\lambda)=0$. 
It can be verified that
\[
  C(P_\lambda) = 
  \begin{cases}
    0 & {\mathrm{if}}~\lambda\in [\frac{3}{5}, 1] \bigcup \{\frac{1}{2}\},\\
    \frac{\lambda-\frac{1}{2}}{\frac{1}{2} \log(\frac{1/2}{\lambda})+\frac{1}{2}\log(\frac{1/2}{1-\lambda})} & {\mathrm{if}}~\lambda\in (1/2, \frac{3}{5}),\\
    \frac{1/2-\lambda}{\lambda \log(\frac{\lambda}{1/2})+(1-\lambda)\log(\frac{1-\lambda}{1/2})} & {\mathrm{if}~} \lambda\in [0, 1/2)\,,
  \end{cases}
\]
which we have visualized in Figure \ref{fig:discontinuity}. The function $ C(P_\lambda)$ is evidently discontinuous at  parameter value $\lambda$ of $\frac{1}{2}$ and $\frac{3}{5}$, as claimed.

\subsection{Dual-based  Representation of Regret Lower Bound}\label{sec:dual} To follow the proposed strategy, we may potentially need to solve the regret lower bound problem, i.e.,  problem \eqref{eq:lb2}, in each round.
Observe that this optimization problem is a semi-infinite minimization problem, which cannot in general be solved directly.
Even  verifying the feasibility of a fixed exploration rate $\eta$ requires  the solution to the auxiliary information minimization problem \eqref{eq:inner-optimization} defining our information distance function $\dis(\eta, x, P_t) $.
We will overcome this challenge by using an equivalent dual representation of this distance function instead.
This novel equivalent dual  representation  enables us to reformulate the regret lower bound problem \eqref{eq:silo2} and its empirical counterpart problem \eqref{eq:lb2} as standard finite convex minimization problems, which are computationally much more attractive compared to solving their semi-infinite representations directly.

We start with stating an equivalent formulation of the distance function with the help of conic hull of the feasible set $\mathcal P$. Let us denote  the conic hull of $\mathcal P$ as the set $\mathcal K=\cone(\mathcal P)\defn \tset{\theta \cdot Q\!}{\!\theta\geq 0, ~Q\in \mathcal P}\,.
$
Observe that this convex cone represents the same information as the set of reward distributions $\mathcal P$ as  we have the following equivalence:
$\mathcal P = \tset{Q\!\!}{\!\!\textstyle\sum_{x\in X, r\in \setr} Q(r, x)=\abs{X}}\bigcap \mathcal K.$
Working with the cone $\mathcal K$ instead of the set $\mathcal P$   helps us provide a unified dual formulation later on. 
This conic representation, along with Proposition \ref{prop:continuity-projection}, allows us to rewrite the information distance function for any deceitful arm $x'\in \xd(P)$ as the following minimization problem
\begin{equation}
  \label{eq:feasibility-cone}
  \begin{array}{rl@{\hspace{4em}}l}
    \dis( \eta, x', P) = \min_{Q} & \multicolumn{2}{l}{\sum_{x\in \tilde X(P)}\eta(x) I(P(x), Q(x))+\chi_\infty(P(x^\star(P))=Q(x^\star(P)))} \\[0.3em]
    \st  & Q \in \mathcal K,   & [{\mathrm{dual~variable:~}}\lambda\in \mathcal K^\star]\\[0.3em]
   & \sum_{x\in X,\, r\in \setr} Q(r, x) = \abs{X},  & [{\mathrm{dual~variable:~}} \beta\in\real]\\[0.3em]
                                         & \sum_{r \in \setr} r Q(r, x')  \geq \rew^\star(P), & [{\mathrm{dual~variable:~}}\alpha\in\real_+]
  \end{array}
\end{equation}
where $\chi_\infty( A)=0$ if event  $A$ occurs; $+\infty$ otherwise.
Note that in problem \eqref{eq:feasibility-cone},  all the nonlinear constraints in the minimization problem \eqref{eq:inner-optimization} are captured by the conic constraint  $Q\in \mathcal K$. 
We denote with $\lambda $ the dual variable of the cone constraint, $\beta\in \real$. Further,  $\alpha\in \real_+$ are the dual variables of the second and third constraints, respectively. Let the dual cone $\mathcal K^\star$ be defined as the unique cone that satisfies
\(
Q \in \mathcal K
\)
if and only if $\iprod{\lambda}{Q} \geq 0, ~ \forall \lambda \in \mathcal K^\star$, 
where we take as inner product $\iprod{\lambda}{Q} = \sum_{x\in X,\,r\in \setr} \lambda(r, x) Q(r, x)$.
Just as the cone $\mathcal K$ captures the same information as the set of reward distributions $\mathcal P$, so does its dual cone $\mathcal K^\star$. This is so because for any  closed convex cones, we have $\mathcal K^{\star\star}=\mathcal K$; see \citet[Section 2.6.1]{boyd2004convex}.

For any reward distribution $\optp\in {\mathcal P}$ and suboptimal arm $x'\in \tilde X(\optp)$, 
we now define our concave dual function:
\begin{align}
  \label{eq:dual}
  \begin{split}
    \dual&( \eta,  x', \optp~;~ \mu)\defn \hspace{-1.2em}\sum_{x\in \tilde X(P), r\in \setr}\hspace{-1.2em} {\textstyle \eta(x) \log\left(\frac{ \eta(x)-\lambda(r, x)-\beta-\alpha r \mb 1(x=x')}{\eta (x)}\right) \optp(r,x)} - \hspace{-1.21em} \sum_{x\in x^\star(\optp), r\in \setr} \hspace{-1.2em} \lambda(r, x)\optp(r, x)  \\
    &   + \alpha \rew^{\star}(\optp) + \beta|\tilde X(\optp)| + \chi_{-\infty}( \eta(x)-\lambda(r, x)-\beta-\alpha r \mb 1(x=x')\geq 0)\,.
  \end{split}      
\end{align}
For notational convenience, we collect all our dual variables in the dual vector $\mu\defn (\alpha, \beta, \lambda)\in \real_+\times \real\times \mathcal K^\star$.
Here, the concave characteristic function $\chi_{-\infty}(A)$ takes on the value zero when event $A$ happens; $-\infty$ otherwise. The following lemma characterizes the dual of the distance function $\dis$ using the $\dual$ function defined above.

\begin{lemma}[Dual Formulation of the Distance Function]
  \label{thm:dual-feasibility} For any reward distribution $P \in \mathcal P_\Omega$ and  arm $x'\in X$,  the following weak duality inequality
  \(
     \dual( \eta,  x', P~; ~\mu)\leq \dis( \eta,  x', P)
  \)
  holds,  where $\mu \in \real_+\times\real\times \mathcal K^\star$ is any feasible dual vector to problem \eqref{eq:feasibility-cone},  
the dual function $\dual$ and distance function $\dis$ are defined in Equations \eqref{eq:dual} and \eqref{eq:inner-optimization}, respectively.   Furthermore, for any reward distribution $\optp\in \mathcal P$ and  suboptimal deceitful arm  $x'\in  \xd(\optp)$, the following  strong duality equality holds
  \begin{equation}
    \label{eq:feasibility-dual}
      \textstyle\dis( \eta,  x', \optp) = \sup_{\mu} \set{\dual( \eta,  x', \optp~; ~\mu)}{\mu \in \real_+\times\real\times \mathcal K^\star}.
  \end{equation}

\end{lemma}

In the following theorem, we characterize the regret lower bound function using the dual formulation in Lemma \ref{thm:dual-feasibility}. The proof of this lemma is presented in Appendix \ref{sec:duality-proof}. Note that all the duality results are presented in Appendix \ref{sec:proof_of_all_duality_results}.
 
\begin{theorem}[Regret Lower Bound -- Dual Formulation] \label{thm:lowerbound_dual}
Consider any reward distribution  $\optp\in\mathcal P$.
  The regret lower bound presented  in Proposition \ref{lemma:lower_bound} is characterized equivalently as 
  \begin{equation}
    \label{eq:silo-dual}
    \begin{array}{rl@{\hspace{2em}}l}
      C(\optp) = \inf_{\eta\geq 0, \, \mu} & \sum_{x\in \tilde X(\optp)} \eta(x) \Delta(x, \optp) &  \\[0.5em]
      \st       & \mu(x)\in \real_+\times\real\times \mathcal K^\star & x\in X, \\[0.5em]
                                           & 1\leq  \dual ( \eta, x, \optp~;~ \mu(x)) &  x\in \xd(P).
    \end{array}
  \end{equation}
\end{theorem}

Observe that the dual characterization of the regret lower bound is convex and finite dimensional.
We further highlight that the dual problem \eqref{eq:silo-dual}, unlike its primal counterpart \eqref{eq:silo2}, is no longer a nested optimization.
The complexity and computational issues of solving the dual lower bound representation are  discussed in Appendix \ref{sec:complexity}. 

\section{Dual Structure Algorithm (DUSA)}
\label{sec:dusa}

In this section, we present our policy,  called DUSA; see Algorithm \ref{eq:dusa}. 
The DUSA algorithm is parameterized by a positive number $\epsilon>0$ which we shall refer to as the accuracy parameter.
Our optimal algorithm is best understood by considering its asymptotic limit $\epsilon\to 0$.
Hence, most of our commentary regarding the intuition behind the algorithm will focus on this regime. 

\begin{algorithm}
  \textbf{Input.} Accuracy parameter $0<\epsilon$.
  \begin{itemize}[leftmargin=*]
  \item \textbf{Initialization.}
    \begin{itemize}[leftmargin=1em]
    \item During  the first $|X|$ rounds, pull each arm $x \in X$ once and  set 
      $  P_{\abs{X}+1}(R(x), x)=1$, where  $R(x)$ is the observed reward of arm $x$ in its corresponding round. 
    \item  Let $s_{\abs{X}+1}=0$ and  $N_{\abs{X}+1}(x) = 1$  for all $x\in X$. For all arms $x \in X$, initialize the dual variables $\lambda_{|X|+1}(x)\in \mathcal K^\star$, $\alpha_{|X|+1}(x)\in \real_+$, and  $\beta_{|X|+1}(x)\in \real$ and reference logarithmic rate $\eta_{|X|+1}'(x)>0$. 
    \end{itemize}
    
  \item  For $t\in \abs{X}+1, \dots$
    \begin{itemize}[leftmargin=1em]
    \item \textbf{Sufficient information condition.} For every deceitful arm $x\in \xd(P_t)$, solve the following convex univariate minimization problem
      \(
        \overline{\dual}_t(x) \defn \overline{\dual}\left(\tfrac{N_{t}}{\log(t)},  x, P_t~; ~\mu_t(x)\right)\defn
        {\displaystyle \max_{\rho\ge 0}} ~\dual\left( \tfrac{N_{t}}{\log(t)},  x,  P_{t}~;~ \rho \mu_{t}(x) \right).
      \)
    \item \textbf{Exploitation.}
      If $\overline{\dual}_t(x)\ge 1+\epsilon$ for all deceitful arms $x \in \xd(  P_{t})$, then $s_{t+1}=s_{t}$, $\mu_{t+1} = \mu_{t}$, $\eta'_{t+1}=\eta'_t$ and
       pull a least played empirically optimal arm
        \(
        x_t \in \textstyle\arg\min \set{N_t(x')}{x'\in\arg\max_{x\in X} \sum_{r\in \setr} r P_{t}(r, x)}.
      \)
    \item \textbf{Exploration.} Else,  set $s_{t+1}=s_t+1$. Furthermore, 
      \begin{itemize}[leftmargin=1em]
      \item If $\min_{x\in X} N_{t}(x) \leq {\color{black} \frac{\epsilon s_{t}}{1+\log(1+s_{t})} }$, pull the least pulled arm \(x_t = \underline x_t \in \arg\min_{x\in X} N_{t}(x)\,.\)
      \item Update exploration target rates $\eta_{t}=\text{\sf SU}(  P_{t}, \eta'_t~;~\mu_{t}, \epsilon)$ using shallow update Algorithm \ref{alg:shallow} and let
        \(
          \bar x_t \in \textstyle\arg\min_{x\in X} \tfrac{N_{t}(x)}{ \eta_{t}(x)}.
        \)
      \item
        {\color{black}Let $x_t^\star \in \textstyle\arg\min \set{N_t(x')}{x'\in X^\star_t}$ be  the  empirically optimal arm. If $N_{t}(x^\star_t) \leq N_t(\bar x_t)$, pull \(x_t = x_t^\star.\) Else, 
        pull 
        \(x_t = \bar x_t.\) }
      \item Update  dual variables $(\eta'_{t+1}, \mu_{t+1})=\text{\sf DU}(  P_{t}~;~ \epsilon)$ using deep update Algorithm \ref{alg:deep}.    
      \end{itemize} 

    \item  \textbf{Updating variables.}
      \begin{itemize}[leftmargin=1em]
      \item Observe reward of the pulled arm $R_t(x_t)$. For any $x\in X$, $x\neq x_t$, and $r\in \setr$, set $  P_{t+1}(r, x) =   P_{t}(r, x)$. Further, for any $r\ne R_t(x_t)$, $  P_{t+1}(r, x_t) = N_{t}(x_t)  P_t(r, x_t)/(N_{t}(x_t)+1)$, and $  P_{t+1}(R_t(x_t), x_t) = \big(1+N_{t}(x_t) P_t(R_t(x_t), x_t)\big)/(N_{t}(x_t)+1)$. Set $N_{t+1}(x_t) = N_{t}(x_t)+1$ and for any $x\ne x_t$,  $N_{t+1}(x)=N_{t}(x)$.
      \end{itemize}
    \end{itemize}
  \end{itemize}
  \caption{DUal Structure Algorithm (DUSA)}
  \label{eq:dusa}
\end{algorithm}

\textbf{Initialization phase.} DUSA starts with an initialization phase of $|X|$ rounds. During these first $|X|$ rounds, it pulls each arm $x\in X$ once.
Then, based on the observed rewards in these rounds, it initializes its estimate for the reward distribution. As before, the estimate $P_t$ is the empirical distribution of the observed rewards before round $t$. 
At the end of the initialization phase, i.e., round $t= |X|$,  DUSA initializes the dual variables for each arm $x$, i.e., $\alpha_{t+1}(x)$, $\beta_{t+1}(x)$, and $\lambda_{t+1}(x)$,  and  the \emph{reference logarithmic rate}, denoted by $\eta'_{t+1}(x)$.

\textbf{Sufficient information and resolving test.} After the initialization phase, DUSA aims to strike a balance between exploration and exploitation. To do so, it checks if for any empirically deceitful arm  $x\in \xd(  P_{t})$, the following  sufficient information condition holds:
\begin{align}
  \label{eq:test}
 1+\epsilon\leq\! \overline{\dual}_t(x)\defn \!\textstyle \overline{\dual}\left(\frac{N_{t}}{\log(t)},  x, P_t~; ~\mu_t(x)\right):= {\displaystyle\max_{\rho\ge 0}}~ \textstyle {\dual}\left( \frac{N_{t}}{\log(t)},  x,  P_{t}~;~ \rho\cdot\mu_{t}(x) \right).
\end{align}
Here, $\mu_{t}(x)=(\alpha_t(x), \beta_t(x), \lambda_t(x))$ are the dual variables for arm $x$ employed in round $t$. Remark that verifying the sufficient information condition merely requires the resolution of a univariate convex optimization problem. {(\color{black} See Section \ref{sec:numer-exper} for computation time of this simple univariate convex optimization problem in various structured bandits.)} We refer to previously defined function $\overline{\dual}$ as the \emph{dual-test} function.  Roughly speaking, our condition verifies whether sufficient information has been collected to distinguish the empirical  $P_t$ from its deceitful distributions; that is, whether or not $\dis(\tfrac{N_{t}}{\log(t)}, x', P_t)\ge  \overline{\dual}\left(\tfrac{N_{t}}{\log(t)}, x', P_t~;~\mu_t\right)\ge 1+\epsilon$ for any $x'\in \tilde X(P_t)$.
In Appendix \ref{sec:info}, we provide a further clarifying interpretation of this sufficient information test as the information distance between $P_t$ and a half-space containing all deceitful distributions $\prob(x', P)$. 
As we will show later, our sufficient information condition  helps DUSA  solve the lower bound problem  \eqref{eq:silo-dual} only in  $\mc O(\log(T))$ rounds. 

\textbf{Exploitation.} If the sufficient information condition holds for all deceitful arms  $x \in \xd(  P_{t})$, DUSA enters an exploitation phase and it selects  a least played empirically optimal arm, i.e.,
      \(
        x_t \in \textstyle\arg\min \set{N_t(x')}{x'\in X_t^*},
      \) where $X_t^{\star}=\arg\max_{x\in X} \sum_{r\in \setr} r P_{t}(r, x)$.

      \textbf{Exploration.} If the sufficient information test fails, DUSA enters an exploration phase as insufficient information has been collected on the empirically suboptimal arms $\tilde X(P_t)$ in the previous rounds. Let $s_t$ be the number of exploration rounds in the first $t-1$ rounds. Then, if $\min_{x\in X} N_{t}(x) \leq {\color{black}\epsilon \tfrac{s_{t}}{(1+\log(1+s_t))}}$, DUSA pulls the least pulled arm $\underline x_t$.
Otherwise, it considers a new $2 \epsilon$-suboptimal logarithmic target rate $\eta_t$ in
\begin{align}
  \label{eq:shallow-lower-bound}
  \begin{array}{rl}
    \overline{C}(P_t ~;~\mu_t):=\inf_{\eta\geq 0} & \sum_{x\in \tilde X(P_t)} \eta(x) \Delta(x, P_t)\\[0.5em] 
    \st & 1\le \overline{\dual}(\eta, x, {P_t};~  \mu_t(x))\quad \forall x \in \xd({P_t}).
  \end{array}
\end{align}
  This infimum can be computed efficiently because it is characterized as the minimum of a linear objective over an ordinary convex optimization constraint set. However, as $2\epsilon$-suboptimal solutions are not unique, a particular selection must be made.
  To do so, the shallow update (SU) Algorithm \ref{alg:shallow} selects the exploration rate closest to some reference logarithmic rate $\eta'_t$. As the optimization characterization \eqref{eq:minimal-primal-projection-1} is strictly convex, its minimizer when feasible must be unique. Our selection procedure thus circumvents the potential pitfall that the infimum \eqref{eq:shallow-lower-bound} may not in fact admit a minimizer and as outlined in Proposition \ref{prop:continuous-shallow-selection} enjoys much better continuity properties even if this infimum is in fact attained. 
  
  Let $\eta_t=\text{\sf SU}(  P_{t}, \eta'_t~;~\mu_{t}, \epsilon)$ be the output of the  shallow update Algorithm and define $\bar x_t \in \textstyle\arg\min_{x\in X} \tfrac{N_{t}(x)}{ \eta_{t}(x)}$.
  {\color{black} When $N_{t}(x^\star_t) \leq N_t(\bar x_t)$, the algorithm plays the empirically optimal arm $x^\star_t \in \textstyle\arg\min \set{N_t(x')}{x'\in X^\star_t}$. Otherwise, DUSA pulls a suboptimal arm
$
  \bar x_t \in \textstyle \arg\min_{x\in \tilde X(  P_{t})} \tfrac{N_{t}(x)}{ \eta_{t}(x)}
$ 
whose empirical logarithmic exploration rate  deviates from its optimal target rate to the largest relative extent.  
  The condition $N_{t}(x^\star_t) \leq N_t(\bar x_t)$ signals that  empirically optimal arm $x_t^\star$ has not been sufficiently explored to justify the forced exploration of $\bar x_t$. Hence, when $N_{t}(x^\star_t) \leq N_t(\bar x_t)$,   $x^\star_t$ is pulled instead of $\bar x_t$ to ensure $x^\star_t$ is explored enough.}

\begin{algorithm}
\textbf{Input.} Reward distribution $Q$, reference exploration target rate $\eta'$, dual variable $\mu =(\alpha, \beta, \lambda)$, and accuracy parameter $\epsilon>0$. Compute 
\\
   \begin{equation}
    \label{eq:minimal-primal-projection-1}
    \begin{array}{rl}
     \eta_\epsilon(Q) \defn \arg\min_{\eta\ge 0} & \sum_{x\in X} (\eta(x)-\eta'(x))^2 \\[0.5em]
      \st  
                                                   & 1\leq \overline{\dual}(\eta, x, Q~;~ \mu (x))  \quad \forall x\in \xd(Q), \\[0.5em]
      &  \sum_{x\in \tilde{X}(Q)}\eta(x) \Delta(x, Q) \leq \overline{C}(Q ~; ~ \mu)+2\epsilon.             
    \end{array}
  \end{equation}
  If no feasible solution exists, we simply set $\eta_\epsilon(Q) \defn +\infty$.\\
  \textbf{Return.} $\eta_\epsilon(Q)$.
\caption{Shallow Update $\text{\sf SU}(Q, \eta' ~;~\mu, \epsilon)$}
\label{alg:shallow}
\end{algorithm}

Finally, DUSA updates the reference  rate $\eta'_t$ and dual variables  $(\eta'_t, \mu_t)=\mathrm{DU}(P_t, \epsilon)$ using the deep update (DU) Algorithm \ref{alg:deep}. The deep update algorithm returns a $\epsilon$-suboptimal solution of the dual counterpart of the lower bound problem \eqref{eq:silo-dual} computed at the empirical reward distribution $P_t$. As the dual counterpart of the lower bound problem \eqref{eq:silo-dual} is an convex optimization problem,  such $\epsilon$-suboptimal solutions can be readily determined. However, for reasons discussed previously, these solutions are not unique and a particular selection must be made. Our deep update algorithm selects a certain minimal norm solution over a slightly restricted feasible set. As the optimization formulation \eqref{eq:minimal-primal-dual-projection-1} in the deep update algorithm is strictly convex, its minimizer is always unique. Note that the first and second  sets of constraints in problem \eqref{eq:minimal-primal-dual-projection-1} are the same as those in problem \eqref{eq:silo-dual}. This ensures that any feasible solution to problem \eqref{eq:minimal-primal-dual-projection-1} is a feasible solution to problem \eqref{eq:silo-dual}. The third set of constraints of problem \eqref{eq:minimal-primal-dual-projection-1} guarantees  any feasible solution to this problem to be an $\epsilon$-suboptimal solution to problem \eqref{eq:silo-dual}. Note that the optimal value $C(Q)$, which appears in the third constraint,  can be obtained simply by solving an ordinary convex optimization problem.
  The fourth set of constraints of problem \eqref{eq:minimal-primal-dual-projection-1} 
ensures  $\eta$ to be bounded away from zero and helps us establish continuity of the optimal solution to problem \eqref{eq:minimal-primal-dual-projection-1}. The last set of constraints of this problem guarantees the dual function $\dual$ to be  finite.\footnote{The dual function, which is defined in Equation  
\eqref{eq:dual}, is written as a function of the logarithm  of variable  $\omega = \tfrac{ \eta(x)-\lambda(r, x)-\beta(x)-\alpha(x) r \mb 1(x=x')}{\eta (x)}$. The last set of constraints of problem \eqref{eq:minimal-primal-dual-projection-1} insures that the $\omega$'s are bounded away from zero and consequently the dual functions are finite.}  
We would  like to stress that, thanks to our sufficient information test, we do not carry out either the deep or shallow update in every round. 

\begin{algorithm}[ht]
\textbf{Input.} Reward distribution $Q$ and accuracy parameter $\epsilon>0$. Let $\eta_\epsilon'(Q), \mu_\epsilon(Q))$ be
\begin{equation}
  \label{eq:minimal-primal-dual-projection-1}
  \begin{array}{r@{~~}l}
     \arg{\displaystyle\min_{\eta,\mu}} & \sum_{x\in X}\eta(x)^2+\sum_{x\in X}\norm{\mu(x)}_2^2 \\ [0.5em]
    \st  &\eta(x)\in \real_+, ~\mu(x)=(\alpha(x), \beta(x), \lambda(x))\in \mathcal K^\star\times\!\real_+\!\times\!\real~~\forall x\in X, \\[0.5em]
                                                                                  & 1\leq \dual(\eta, x, Q~;~ \mu(x))  \quad\forall x\in \xd(Q),\\[0.5em]
                                                                                  & \sum_{x\in \tilde{X}(Q)}\eta(x) \Delta(x, Q) \leq C(Q)+\epsilon, \\[0.5em]
                                                                                  & \eta(x) \geq \epsilon/(2\sum_{x\in \tilde{X}(Q)} \Delta(x, Q)) \quad\forall x\in \tilde X(Q),\\[0.5em]

                                                                                  & \eta (x)\geq \lambda(r, x)(x')\!+\!\beta(x')\!+\!\alpha(x') r \mb 1(x=x')\!+\!\frac{\epsilon}{(2\sum_{x\in \tilde{X}(Q)} \Delta(x, Q))} ~ \forall x'\in \xd(Q),\,x\in \tilde X(Q).
  \end{array}
\end{equation}
\textbf{Return.} $(\eta_\epsilon'(Q), \mu_\epsilon(Q)) $.
\caption{Deep Update $\textsc{DU}(Q~;~ \epsilon)$}
\label{alg:deep}
\end{algorithm}

As discussed earlier, it is critical to have a continuous selection rule as the $\epsilon$-suboptimal solutions are not unique. The subsequent proposition points out that our deep update selection rule is indeed well-defined and continuous on $\mc P'$ as required; see  the proof of this  result  in Appendix \ref{sec:continuous-selections}.

\begin{proposition}[Continuous Selection Rule for the Deep Update Algorithm]
  \label{prop:continuous-deep-selection}
  Our deep update selection rule $(\eta_\epsilon'(Q), \mu_\epsilon(Q))$, presented in Equation \eqref{eq:minimal-primal-dual-projection-1},  exists and is continuous on $\mathcal P'$, where $\mathcal P'$ consists of all the distributions in $\mathcal P$ that satisfy Assumption \ref{assumption}. 
\end{proposition}

 To establish our main result stated in Theorem \ref{thm:main}, we also require the shallow update selection rule to be stable. The shallow update is stable in the sense that $\text{\sf SU}(P_t, \eta_\epsilon'(P_{t'})~;~\mu_\epsilon(P_{t'}), \epsilon)$ converges to the desired reference rate $\eta'_{\epsilon}(P)$ when both empirical distributions $P_t$ and $P_{t'}$ approach the true reward distribution $P$.  This notion of stability for the shallow update is formalized in the following proposition and will be crucial to establish the optimality of our DUSA policy.  

\begin{proposition}[Stability of the  Shallow Update Selection Rule]
  \label{prop:continuous-shallow-selection}
  For any $P\in \mathcal P'$, we have 
    \[
    \begin{array}{rl@{\quad}l}
      \lim_{\kappa\to 0} ~\max_{Q_1, Q_2} & \norm{ \text{\sf SU}(Q_1, \eta_\epsilon'(Q_2) ~;~\mu_\epsilon(Q_2), \epsilon) - \eta'_\epsilon(P)}_2& =\quad 0. \\[0.5em]
      \st & \norm{Q_1-P}_\infty \leq \kappa, ~ \norm{Q_2-P}_\infty \leq \kappa &
    \end{array}
  \]
\end{proposition}

The proof of Proposition \ref{prop:continuous-shallow-selection} is  in Appendix \ref{sec:proof:shallow}. We now present our main result.

\begin{theorem}[Regret Bound of DUSA]\label{thm:regret}
  \label{thm:main}
  Let the true reward distribution $P\in \mathcal P'$ and consider any accuracy parameter $0<\epsilon<\tfrac{1}{\abs{X}}$. Then, DUSA suffers the following asymptotic regret:
  \begin{align}
    \label{eq:regret-upper-bound}
    \limsup_{T\to\infty}\textstyle ~\frac{\reg(T, \optp)}{\log(T)}&\leq \textstyle(1+\epsilon)(C(\optp)+ \epsilon\left(1+\sum_{x\in X} \Delta(x, \optp)\right))\,,\\
    \intertext{where the regret is computed against the benchmark that knows the best arm in advance.  Furthermore, the expected number of rounds in which DUSA enters the exploration phase in bounded by} 
    \label{eq:exploration-rounds-upper-bound}
    \limsup_{T\to\infty}\textstyle ~\frac{\E{}{s_T}}{\log(T)}& \leq (1+\epsilon)\abs{X}^2(\norm{\eta'_\epsilon(\optp)}_\infty+\epsilon)\,,
  \end{align}
  where the expectation is taken with respect to the randomness in the observed rewards and choices made by DUSA. 
\end{theorem}

Comparing the regret lower bound in Proposition \ref{lemma:lower_bound} with the regret upper bound found in Equation \eqref{eq:regret-upper-bound}, we can hence conclude that 
DUSA enjoys  the minimal asymptotic logarithmic regret  as the accuracy  parameter $\epsilon$ goes to zero.
The maximum target exploration rate $\norm{\eta'_\epsilon(P)}_\infty=\max_{x\in \tilde X(P)} \eta'(x,P)$ appearing in our bound on number of exploration rounds in Equation \eqref{eq:exploration-rounds-upper-bound} remains bounded uniformly for all accuracy parameter values $\epsilon>0$.  This is so because (i) we have from $\epsilon$-suboptimality that $\textstyle\sum_{x\in \tilde X(P)} \eta'_\epsilon( x, P) \Delta(x, P)\leq C(P)+\epsilon<\infty$, and (ii) $\Delta(x, P)>0$ for all suboptimal arms.
Hence, the number of times we have to update the target exploration rates and dual variables using either the shallow or deep update Algorithms \ref{alg:shallow} and \ref{alg:deep} grows merely logarithmically in the number of rounds $T$. Note  that the number of updates is equal to  the number of the exploration rounds. Theorem \ref{thm:main} does not bound the regret of DUSA for any finite number of rounds but instead is completely asymptotic in nature. {\color{black}Nevertheless, in the proof of Theorem \ref{thm:main}, which we will present in Appendix \ref{sec:proof}, non-asymptotic regret bounds are also derived. The non-asymptotic regret of DUSA can be upper bounded by  $C_{\rm{initialize}} + C_{\text{exploit}}+C_{\text{explore}}$,
where $C_{\rm{initialize}} =|X|$ and $C_{\text{exploit}}$ defined in Equation \eqref{eq:c_exploit} are, respectively,  upper bounds on the regret during the  initialization and exploitation rounds.
 Interestingly, thanks to our sufficient information test, $C_{\text{exploit}}$ does not contribute to the asymptotic regret. The term
$C_{\text{explore}}$, which is the dominating factor in  the asymptotic regret, is an upper bound on the regret during the exploration rounds.  In the proof of   Theorem \ref{thm:main}, the term $C_{\text{explore}}=C_2+C_3+C_4$ is itself decomposed into three components.
The terms $C_2$  (respectively $C_3+C_4$) is an  upper bound on the regret during exploration rounds when the empirical  reward distribution is ``far from'' (respectively ``close to'')   the true reward distribution.
 }

{\color{black}
  \begin{remark}[Nondeceitful Bandits]
    \label{remark:nondeceitful-bandits}
    When the bandit $P\in \mc P'$ is nondeceitful, i.e., $\tilde X_d(P)=\emptyset$, the regret lower bound function $C(P)=0$ vanishes which opens up the possibility of achieving a sublogarithmic regret. \citet{jun2020crush} propose a bandit policy which in fact achieves a bounded regret on such nondeceitful bandits. See also \citet{lattimore2014bounded,gupta2020unified} for similar results. We show in Remark \ref{rem:proof:nondeceitful-bandits} in the proof of our main Theorem \ref{thm:main} that the regret of DUSA also remains bounded for all $T$ on such nondeceitful bandits.
  \end{remark}
}

\section{Numerical Experiments}
\label{sec:numer-exper}

{\color{black}In this section, we evaluate the performance of DUSA in the context of several structural bandit problems discussed in Section \ref{sec:struct-band-ex}. We first focus on  linear and Lipschitz bandits, which are well-known structured bandits. Then, to illustrate the flexibility of DUSA, we further evaluate  DUSA in the context of dispersion bandits. We begin by explaining how DUSA is implemented.} 

\subsection{DUSA Implementation}

{\color{black}
We have implemented our DUSA bandit policy for the discussed linear, Lipschitz and dispersion bandits discussed in Section \ref{sec:struct-band-ex} in \texttt{Julia}. All code is available as a \texttt{Julia} package at \url{https://tinyurl.com/473y3wr5}. The code is written in a modular fashion which allows to quickly extend our DUSA package to other structures beyond those we have implemented. Notice  that every structured bandit problem is uniquely characterized by a convex set $\mc P$. Extending DUSA to be able to deal with any convex bandit problems is as simple as exposing its associated dual cone $\mc K^\star=\cone(\mc P)^\star$ to the package; see also Section \ref{sec:dual}. In every exploration round, DUSA may need to compute both the shallow and deep updates stated in Section \ref{sec:dusa}. These updates  demand the solution of a convex optimization problem. The resulting exponential cone optimization problems are subsequently solved using the commercial exponential cone interior point solver \texttt{Mosek}.

In our practical implementation, we make two minor changes in DUSA. 
First, instead of carrying out the shallow and deep updates exactly as stated in Algorithms \ref{alg:shallow} and \ref{alg:deep}, we simply return an arbitrary $\epsilon$-suboptimal solution to the optimization problems \eqref{eq:shallow-lower-bound} and \eqref{eq:lb2}, respectively. In all numerical results we set our accuracy parameter as $\epsilon =1\e{-3}$. Second, in our implementation, DUSA policy   enters the exploitation phase when the following sufficient test passes for all empirically deceitful arms $x\in \xd(P_t)$:
\(
  \overline{\dual}_t(x)\ge (1-\exp(-t/T_0)) \cdot (1+\epsilon)\,,
\) 
where we set $T_0=2000$. This information test is slightly different from the one described  in Algorithm \ref{eq:dusa}.  Our DUSA policy as described in Algorithm \ref{eq:dusa} only enters the exploitation phase when the sufficient information test $\overline{\dual}_t(x)\ge 1+\epsilon$ passes for all empirically deceitful arms $x\in \xd(P_t)$.  We found that replacing this test with this slightly modified version (i.e., $ \overline{\dual}_t(x)\ge (1-\exp(-t/T_0)) \cdot (1+\epsilon)$)  avoids over-exploration during the early rounds of our policy. Finally, we note that computing $\overline{\dual}_t(x)$ for any such arm $x$ merely demands the solution of the univariate convex optimization problem in Equation \eqref{eq:test} which we determine with a simple bisection search.
}

\subsection{Evaluating DUSA Under Well-known Structured Bandits}

\subsubsection{Linear Bandits}
\label{sssec:linear-bandits-experiments}

{\color{black}
\textbf{Setup.}
We consider linear bandits with ten arms; that is $|X| =10$.  The  reward of arm $x\in X$ is drawn from a Bernoulli distribution with the mean  $c_x^T\theta$. Here, $c_x$ and $\theta =[\hat \theta^T, 1]^T$ are 5-dimensional vectors, where $\hat \theta$ is a 4-dimensional vector and  each element  of  $\hat \theta$ is drawn independently from a standard normal distribution. 
After generating $\theta$, we 
generate $c_x$.  To do so, we first generate a  4-dimensional vector, denoted by $\hat c_x$, where each element of this vector is drawn from a standard normal distribution. We then set $c_x = [a\cdot \hat c_x^T, b]^T$, where $a,b\in \mathbb{R}$ are chosen such that $\max_{x\in X} c_x^T\theta =0.9$ and $\min_{x\in X} c_x^T\theta =0.1$. We consider $120$ problem instances, where each problem instance corresponds to a particular coefficient vector $\theta$ and $(c_x)_{x\in X}$. 
For each problem instance, we run our DUSA algorithm and the GLM-UCB algorithm of \cite{filippi2010parametric} $20$ times over $T= 5,000$ rounds.  Linear bandit problems are a special case of the parametric bandits discussed in Section \ref{sec:struct-band-ex} and are explicitly characterized by
\(
  \mc P_{\rm{lin}}=\set{Q\in \mc P_\Omega}{\forall x\in X, ~\exists \theta~\st~ \textstyle\sum_{r\in\setr} r Q(r, x) = c_x^T \theta}.
\)
Our DUSA algorithm can specialized to this class of convex bandit problems simply by considering the associated dual cone $\mc K_{\rm lin}^\star=\cone(\mc P_{\rm lin})^\star$  explicitly characterized in Appendix \ref{sec:dual:running-examples}.

The GLM-UCB algorithm of \cite{filippi2010parametric}, which uses an upper confidence bound principle,  needs to solve two optimization  problems in every round; see Equations (6) and (7) in \cite{filippi2010parametric}. The first optimization problem, which is a maximum likelihood estimation step for $\theta$,  is convex problem and is easy to solve.  The second problem---which involves a projection step for the estimated value of $\theta$ obtained from the maximum  likelihood estimation approach---is quite complex to solve. To implement  GLM-UCB, we only solve the first optimization problem. However, to ensure that we do not return solutions which are too far from the true value $\theta$, we provide upper and lower bounds on the values that $\theta$ may take in the first optimization problem.\footnote{Let $\underline \theta =\min_{i\in [4]} \theta_{i}$ and $\bar \theta =\max_{i\in [4]} \theta_{i}$. Then, in the first optimization problem, we enforce $\theta_i \le 2 \cdot \bar \theta$ and $\theta_i\ge \min(2\cdot \underline \theta, 0 )$.} Despite  this change, the GLM-UCB policy remains very slow. To make it faster, we only solve the first optimization problem every $200$ rounds. Finally, we choose the \emph{confidence} parameter of GLM-UCB  as $\rho(t) =\sqrt{\tfrac{0.3}{\log(t)}}$ where the constant  $0.3$ is chosen using cross validation.

{
  Figure \ref{fig:Linear_bandit}  depicts the normalized cumulative regret of DUSA and GLM-UCB as a function of the number of rounds.
  Here, the normalized cumulative regret is the ratio of the cumulative regret over $T$ rounds to the asymptotic regret lower bound $C(P)\log(T)$ as  in Proposition \ref{lemma:lower_bound}. Each of these box plots depict $2,400$ data points as we generate $120$ problem instances and we run each problem instance $20$ times. We observe that the median of the normalized cumulative regret of DUSA and GLM-UCB are comparable. However, the  normalized cumulative regret of DUSA is more concentrated around its median than that of GLM-UCB. For instance, for GLM-UCB, we observe an outlier whose cumulative regret is around $20$ times the regret lower bound. On the other hand, the worst cumulative regret experienced by DUSA is less than $7$ times the regret lower bound, which  suggests that DUSA is more reliable than GLM-UCB when it comes to worst-case performance. 
 See Appendix \ref{sec:add-num-experiments} for similar experiments with      $\abs{X}=5$ and $\abs{X}=15$. 
}

We now comment on computation time of DUSA.  The computational time of DUSA is primarily determined by how fast the deep update in Algorithm \ref{alg:deep} can be carried out. Crucially, the deep update is only carried out in some of  exploration  rounds.  The average computation time to carry out exploration  rounds is here $0.3$ seconds while the average computation time for exploitation rounds is considerably less at $0.06$ seconds. Finally, the overall average computation time per round measured $0.09$ seconds indicating that, as expected, the number of exploitation rounds makes up a large fraction of all $T$ rounds.

\begin{figure}[ht]
  \centering
  \begin{subfigure}[b]{0.5\textwidth}
    \centering
    \includegraphics[width=\textwidth]{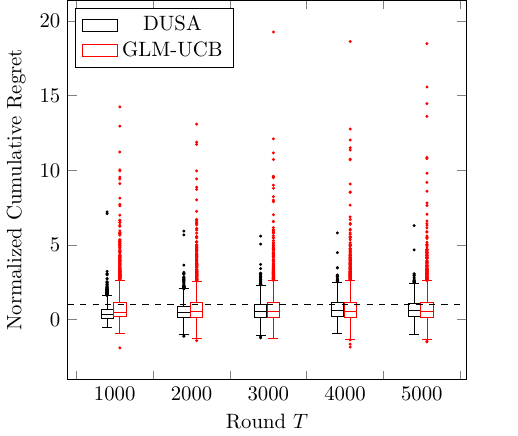}
    \caption{Linear Bandits}
    \label{fig:Linear_bandit}
  \end{subfigure}%
  \hfill
  \begin{subfigure}[b]{0.5\textwidth}
    \centering
    \includegraphics[width=\textwidth]{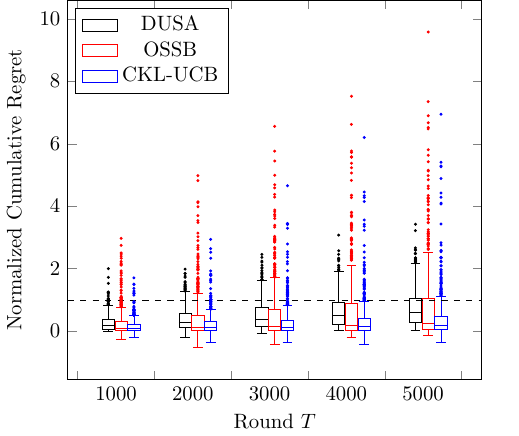}
    \caption{Lipschitz Bandits}
    \label{fig:lips_bandit}
  \end{subfigure}
  \caption{The normalized cumulative regret  of DUSA on linear and Lipschitz bandit instances as a function of the number of rounds $T$.
  }
\end{figure}
}

\subsubsection{Lipschitz Bandits}

{\color{black}
\textbf{Setup.} We consider Lipschitz bandits with $10$ arms and Lipschitz constant $L=0.5$. The reward of each arm $x$ is drawn from a Bernoulli distribution with mean $\theta(x)$. Here, we use the same setup as in \citet{pmlr-v35-magureanu14}  and set $\theta(x) = 0.8 -0.5\cdot|0.5-x|$, where $x$ is drawn from the uniform distribution on the interval $[0,1]$. We consider $150$ problem instances where we run each problem instance $5$ times over the course of $T = 5,000$ rounds. Recall that Lipschitz bandit problems as discussed in Section \ref{sec:struct-band-ex} are explicitly characterized by
\(
  \mc P_{\rm{Lips}}\defn \set{Q\in \mathcal P_\Omega}{%
    Q(1, x) - Q(1, x') \leq L \cdot d(x, x') ~~ \forall x\neq x' \in X    
  }
\)
where we consider the distance function $d(x, x')=\abs{x-x'}$.
Our DUSA policy can also be specialized to this class of convex bandit problems by considering its associated dual cone $\mc K_{\rm Lips}^\star=\cone(\mc P_{\rm Lips})^\star$ explicitly characterized in Appendix \ref{sec:dual:running-examples}. 

We compare our algorithm with two algorithms proposed by \cite{pmlr-v35-magureanu14} and  \cite{combes2017minimal}, named CKL-UCB and OSSB, which are tailored specifically for Bernoulli Lipschitz bandits.\footnote{Since the OSSB algorithm of \cite{combes2017minimal}, when tailored to Bernoulli Lipschitz problem, is virtually identical to the OSLB algorithm of \cite{pmlr-v35-magureanu14}, here  we do not consider OSLB  separately. Both OSLB and OSSB solve in every round the specialized regret lower bound problem stated in Equation \eqref{eq:lipschitz:parametric-linear} to decide what arm to pull.} Note that OSSB solves in every round the specialized regret lower bound problem stated in Equation \eqref{eq:lipschitz:parametric-linear} to decide what arm to pull. 
We solve this resulting linear optimization problem using \texttt{Mosek}. The former CKL-UCB policy is based on an upper confidence principle.

\textbf{Regret.} Figure \ref{fig:lips_bandit} shows a box plot of the normalized cumulative regret of DUSA, OSSB, and CKL-UCB over time.
Each of these box plots depicts $750$ data points as we generate $150$ problem instances and we run each problem instance $5$ times. We again observe that the median normalized cumulative regret across all policies behave quite similar. The normalized cumulative regret of the DUSA is yet again more concentrated around its median as compared to the normalized cumulative regret of its two competitors.

\begin{table}[ht]
  \centering
\footnotesize{
  \begin{tabular}{l@{\hspace{2em}}r@{\hspace{2em}}c@{\hspace{1em}}rrr}
    \cline{1-6}
     & \bfseries OSSB &&  \multicolumn{3}{c}{\bfseries DUSA}\\
    \cmidrule(r{5pt}){2-2}\cmidrule(l{5pt}r{0pt}){4-6} 
    \# Arms & Time & & Time & Time (exploitation) & Time (exploration) \\
    \cline{1-6}
    $\abs{X} = 5$ & 0.042 & & \bf{0.031} & 0.028 & 0.11 \\
    $\abs{X} = 10$ & \bf{0.16} & & 0.27 & 0.082 & 0.50  \\
    $\abs{X} = 15$ & \bf{0.42} & & 1.27 & 0.21 & 1.56  \\
    $\abs{X} = 20$ & \bf{0.77} & & 2.92 & 0.37 & 4.1 \\
    \cline{1-6}
  \end{tabular}}
  \caption{Average computation time in seconds to perform one round of the OSSB and DUSA policy.}
  \label{tab:ossb_vs_dusa_timing}
\end{table}

\textbf{Computation time of DUSA versus OSSB.} 
In Table~\ref{tab:ossb_vs_dusa_timing}, we also compare DUSA with OSSB in terms of their computation time over $T=100,000$ rounds averaged over three Lipschitz bandits instances generated as detailed earlier but now with an increasing number of arms $\abs{X} \in \{5, 10, 15, 20\}$.\footnote{As the number of arms $|X|$ increases,  the computational time of DUSA  is impacted in two distinct ways. First, the computation complexity of the dual optimization problems characterizing our shallow and deep updates which needs to be solved in an exploration round increases as $|X|$ grows. More specifically, for these updates, DUSA requires the solution to convex optimization problems whose size scales linearly with the number of arms.  Second, as the number of arms increases, the number of exploration rounds---in which a convex optimization problem needs to get solved---also grows.  
According to Theorem \ref{thm:main}, the expected number of exploration rounds scales as $\mathcal O(|X|^2 \log(T))$. 
}
The computational effort of OSSB consists predominantly of solving in every round the (semi) closed-form solution stated in Equation \eqref{eq:lipschitz:parametric-linear} to the regret lower bound problem \eqref{eq:silo2} in the context of Bernoulli Lipschitz bandits.
As discussed before, the computation time of DUSA differs significantly on whether or not an exploitation round is considered. Recall that during an exploitation round, DUSA only needs to conduct a simple sufficient information test while during an exploration round, DUSA may need to conduct shallow and deep updates in addition to the sufficient information test. 
Hence, we report the average computation time of DUSA for exploitation and exploration  rounds separately.

The computation time of DUSA averaged over both exploitation and exploration rounds is comparable to the overall computation time of OSSB. However, OSSB achieves a slightly better computation time (for finite $T$), compared with DUSA.
{On one hand, DUSA has an edge over OSSB as it only needs to conduct its shallow and deep updates in $\O(\log(T))$ rounds. On the other hand, 
as we observe in Table 1, due to the sufficient information test in DUSA, the running time of DUSA during its exploitation rounds is roughly half of  the per-round running time of OSSB. Putting these together, we see that OSSB has a slightly better running time than DUSA when $|X|> 5$. }

We further note that the fact that the computation time of OSSB for finite $T$ is slightly better than that of DUSA can be attributed to the fact that for Bernoulli Lipschitz bandits the regret lower bound (\ref{eq:silo2}) admits a simple (semi) closed-form solution and hence is computationally efficient to solve (see Equation \eqref{eq:lipschitz:parametric-linear}). \citet{combes2017minimal} indeed only define a computational procedure for OSSB for a select few convex bandit problems in which the lower regret bound problem \eqref{eq:silo2} admits such a (semi) closed-form solution. The fact that the computation times of DUSA and OSSB are comparable even for those particular problems where OSSB is applicable is quite remarkable. 
}

\subsection{Dispersion Bandits}

{\color{black}
\textbf{Setup.} We consider here dispersion bandit problems with ten arms where the rewards of each arm are supported on $\setr=\{0, 1/10, 2/10,  \dots, 1\}$. Recall that dispersion bandit problems as discussed in Section \ref{sec:struct-band-ex} are explicitly characterized by
\(
  \mathcal P_{\text{dis}}=\set{Q\in \mathcal P_\Omega}{\tfrac{\sum_{r\in\setr } r^2Q(r, x)}{\sum_{r\in\setr }r Q(r, x) }\le \gamma(x)\quad \forall x\in X}.
\)
Our DUSA algorithm can also be specialized to this class of convex bandit problems by considering its associated dual cone $\mc K_{\rm dis}^\star=\cone(\mc P_{\rm dis})^\star$ explicitly characterized in Appendix \ref{sec:dual:running-examples}.
The dispersion bound for each bandit problem  and each arm $x$ (i.e., $\gamma(x)$ in $\mathcal P_{\text{dis}}$) is obtained as $\gamma(x) \sim 1/\abs{\setr}+u(x)$, where the random variables $u(x)$ for $x\in X$ are independent and uniformly distributed on $[0, 1/5]$. To construct a particular dispersion bandit, we first draw a distribution $Q$ uniformly from the simplex $\mc P_\Omega$. A dispersion bandit instance $P$ is then obtained as the instance in $\mc P_{{\rm{dis}}}$  closest to $Q$ according to the distance $\sum_{r\in \setr, x\in X} \abs{P(r, x)-Q(r, x)}$.
We consider $100$ such bandit instances where we run each poblem instance $5$ times over the course of $T = 10,000$ rounds.

As discussed before, the computation time of DUSA differs significantly on whether or not an exploitation round is considered. The average computation time to carry out exploration rounds is here $0.8$ seconds while the average computation time for exploitation rounds is considerably less at $0.1$ seconds. The overall average computation time per round measured $0.2$ seconds. 

As remarked in Example \ref{ex:structured-problems-2}, the class of dispersion bandits has not been studied before.
Hence, we  compare the performance of DUSA with the KL-UCB bandit policy of \citet{cappe2013kullback} in terms of their regret where we remark that the latter policy does not exploit the dispersion structure fully. In Figure \ref{fig:plot-dispersion}, we present the average cumulative regret as a function of the number of rounds over all dispersion bandit instances; see the filled curves in the figure. The shaded area depicts the variation of this average over the considered runs and instances. Unlike KL-UCB bandit policy, our policy DUSA is able to exploit the dispersion bandit structure optimally.

{\color{black}Finally, we separately depict the average regret on the nondeceitful dispersion bandit instances accumulated by both DUSA and KL-UCB as dotted curves. The empirical evidence confirms here, as we pointed out in Remark \ref{remark:nondeceitful-bandits}, that on such nondeceitful bandit instances, DUSA suffers a regret which remains bounded in the number of rounds. The KL-UCB policy is not even rate-optimal here as its regret on nondeceitful instances increases in the number of rounds.
}

\begin{figure}[ht]
  \centering
  \includegraphics[height=5.8cm]{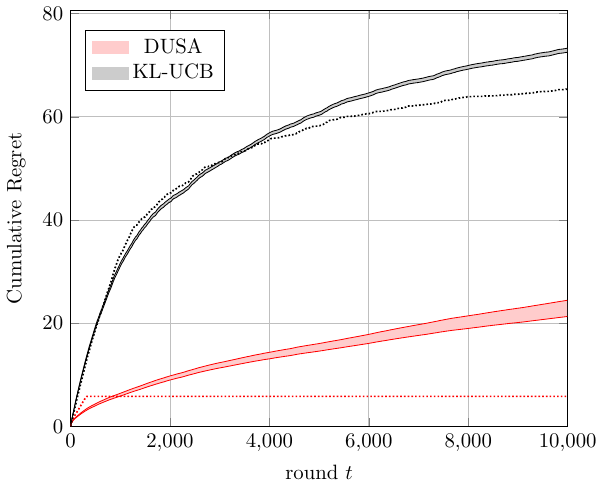}
  \caption{The filled curves show the cumulative regret of DUSA and KL-UCB for dispersion bandits, averaged over $100$ instances. The width of the curves is equal to two times the standard error of regrets. The dotted curves show the average cumulative regret of DUSA and KL-UCB for a specific instance of dispersion bandits that is non-deceitful. As expected, the regret of DUSA for the non-deceitful instance is finite and does not grow with $t$.}
  \label{fig:plot-dispersion}
\end{figure}
}

 \section{Concluding Remarks} \label{sec:conclude}
 
  In this paper, we have presented the first bandit policy that optimally exploits convex structural information in  a computationally feasible fashion.
 Our policy, DUSA, automatically exploits any known structural reward information  with the help of convex constraints on the unknown reward distribution. Hence, our powerful policy can be used in many practical problems where structural information is often available but may be peculiar to a very specific problem.
  Rather than developing  bandit policies for some idiosyncratic class of bandit problems, our paper presents a universally optimal algorithm.

  {
      We believe that our DUSA algorithm presents a useful blueprint which can be extended in several directions.
    For instance, in our setting, we assume that the set of realized rewards for each arm is finite. This assumption is made for technical reasons and we believe that it can be relaxed when the distribution of the reward is finitely parameterized (e.g., Gaussian and exponential distributions). Nonetheless, it is an interesting future research direction to explore under what conditions DUSA can be generalized to a setting with infinite reward sets. }

\section*{Acknowledgment}
N.G. was supported in part by the Young Investigator Program (YIP) Award from the Office of Naval
Research (ONR) N00014-21-1-2776 and the MIT Research Support Award.
{
  \bibliographystyle{informs2014}
  \bibliography{references}
}
\clearpage

\begin{APPENDICES}
 \renewcommand{\theHchapter}{A\arabic{chapter}}

 \section{Additional Numerical Experiments}
 \label{sec:add-num-experiments}

 {
   We remark that the asymptotic regret bound for DUSA stated in Theorem \ref{thm:main} does not explicitly depend on the number of arms. However, the non-asymptotic regret bound of for DUSA found in Appendix \ref{sec:proof} (i.e., $C_{\rm{initialize}} + C_{\text{exploit}}+C_{\text{explore}}$) is an increasing function of the number of arms. In particular, the regret $C_{\text{exploit}}$ accumulated during the exploitation phase defined in Equation \eqref{eq:c_exploit} may increase exponentially in the number of arms $\abs{X}$ even though it remains finite in the number of rounds $T$ and hence does not appear in the asymptotic regret expression of DUSA stated in Theorem \ref{thm:main}. We believe that this rather ungraceful dependence of the regret of DUSA on the number of arms can be traced back to our analysis and more specifically to  the concentration inequality in Lemma \ref{lemma:concentration}.

   To demonstrate  this, we conduct  
   additional numerical experiments in the context of linear bandit problems. We investigate how the regret of DUSA scales as the number of arms grows.  We observe that the regret of  DUSA does not seem to critically depend on 
   the number of arms $\abs{X}$ even when considering a fixed number of rounds. 
   That is, we repeat the experimental setting considered in Section \ref{sssec:linear-bandits-experiments} for linear bandits counting now $\abs{X}=\{5, 15\}$ arms. Figure \ref{fig:Linear_bandit_add} depicts the normalized cumulative regret of DUSA and GLM-UCB as a function of the number of rounds. Recall that  Figure \ref{fig:Linear_bandit} shows    the normalized cumulative regret of DUSA and GLM-UCB  with $\abs{X}=10$ arms. Similar to Figure \ref{fig:Linear_bandit},  we observe that for both $|X| =\{5, 15\}$,   GLM-UCB holds a slight edge over DUSA in terms of the median normalized cumulative regret at the expense of suffering rather severe outlier events.
}
 \begin{figure}[ht]
  \centering
  \begin{subfigure}[b]{0.5\textwidth}
    \centering
    \includegraphics[width=\textwidth]{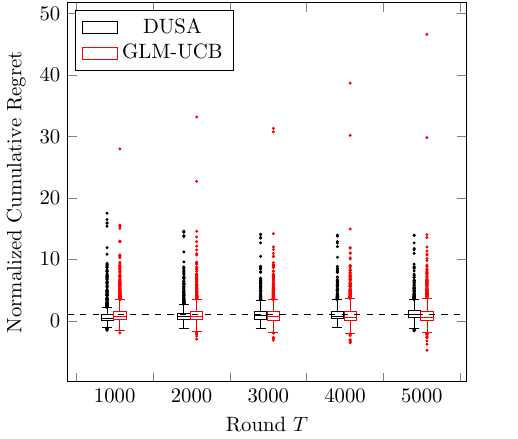}
    \caption{Linear Bandit ($\abs{X}= 5$)}
    \label{fig:Linear_bandit-5}
  \end{subfigure}%
  \hfill
  \begin{subfigure}[b]{0.5\textwidth}
    \centering
    \includegraphics[width=\textwidth]{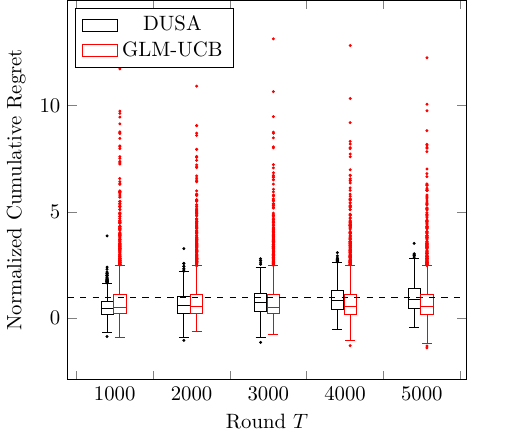}
    \caption{Linear Bandit ($\abs{X}= 15$)}
    \label{fig:Linear_bandit-15}
  \end{subfigure}
  \caption{The normalized cumulative regret of DUSA on linear bandit instances counting $\abs{X}=\{5,15\}$ arms as a function of the number of rounds $T$.}
  \label{fig:Linear_bandit_add}
\end{figure}
 
 \section{Topological Properties of the Information Distance}
\label{sec:information-topology}

With a slight abuse of notation, we denote the set of all distributions supported on a set of outcomes $\setr$ as
\begin{align}
  \label{eq:PR}
  \mathcal P_{\setr}=\tset{Q\in \real^{\abs{\setr}}_+}{\textstyle\sum_{r\in \setr} Q(r) = 1}\,.
\end{align}
In this paper, we equip the probability simplex $\mathcal P_\setr$ with the standard topology relative to its affine hull. The open sets $$\set{Q}{\textstyle\sum_{r\in \setr}Q(r)=1, ~\norm{Q-P}_\infty\defn\norm{Q-P}_\infty=\max_{r\in \setr}\,\abs{Q(r)-P(r)}< \epsilon}$$ for all $P\in \mathcal P_\setr$ and $\epsilon>0$ are a base for this topology. The set of all interior points of $\mathcal P_\setr$ in this topology is then identified with $\tset{Q\in \real^{\abs{\setr}}_{++}}{\textstyle\sum_{r\in \setr} Q(r) = 1}$. The information distance is a natural statistical notion of distance on the probability simplex $\mathcal P_{\setr}$ and satisfies several well-known key properties which we use in this paper.

\begin{lemma}[Properties of Information Distance]
  \label{prop:relative-entropy}Let $\setr$ be a discrete set and $\mathcal P_\setr$ be the set defined in (\ref{eq:PR}). 
  The information distance enjoys the following key properties:
  \begin{enumerate}
  \item \textbf{Information inequality.} $I(P, P') \geq 0$ for all $P, P'\!\in\mathcal P_{\setr}$, and  $I(P, P')= 0$ if and only if $P= P'$.
  \item \textbf{Convexity.} For all pairs $(P_1, P'_1), (P_2, P'_2)\in \mathcal P_{\setr}\times \mathcal P_\setr$ and $\lambda\in[0,1]$ we have
    \[
      I\big((1-\lambda) P_1 +  {\lambda}P_2, (1-\lambda) P'_1 +\lambda  P'_2\big)\leq (1-\lambda) I(P_1,P'_1)+\lambda I(P_2, P'_2).
    \]
  \item \textbf{Continuity in $P'$.} $I(P , P')$ is continuous in $P'$ on $\set{P'\in {\mathcal P_{\setr}}}{P\ll P'}$ for any fixed $P\in \mathcal P_\setr$.
  \item \textbf{Lower semi-continuity.} $I$ is lower semicontinuous on $\mathcal P_{\setr}\times \mathcal P_\setr$. That is, its sub-level sets defined as $\set{(P, P') \in \mathcal P_\setr\times \mathcal P_\setr}{I(P, P')\leq \alpha}$ are closed (and bounded) for any $\alpha\geq 0$.
  \end{enumerate}
\end{lemma}
\proof{Proof of Lemma \ref{prop:relative-entropy}.}
  The information inequality and convexity properties can be found in most textbooks on information theory. For instance, see Theorem 2.6.3 and Theorem 2.7.2 in \cite{cover2012elements}, respectively.  For showing the lower semi-continuity result, let us write $I(P, P')$ as
  \[
    I(P, P') = \textstyle\sum_{r\in \setr} P(r) \log (\tfrac{P(r)}{P'(r)})
  \]
  for $P$ and $P'$ in $\mathcal P_\setr$, where we use the convention $0 \log (\tfrac{0}{a})=0$ and $b \log (\tfrac{b}{0})=\infty$
  for all $a\geq 0$ and $b>0$.  
  The lower semi-continuity property follows from the fact that functions of the type $(u, v) \mapsto u\log(\tfrac{u}{v})$ are lower semi-continuous and summation preserves lower semi-continuity; see \cite[Proposition 1.1.5]{bertsekas2009convex}. 
  
  For continuity in $P'$, let $\supp(P)=\set{r\in\setr}{P(r)>0}$ denote the support of distribution $P$. We can write
  \[
    I(P, P') = \textstyle\sum_{r\in \supp(P)} P(r) \log (\tfrac{P(r)}{P'(r)}).
  \]
  Continuity of $I(P, P')$ in $P'$ when $P\ll P'$ follows from the immediate observation that the logarithm function $\log (\tfrac{P(r)}{P'(r)})$ is indeed finite and continuous at any $P'(r)>0$ for $r\in \supp(P)$. \hfill\Halmos
\endproof

\section{Duality Results}
\label{sec:proof_of_all_duality_results}

To show the duality results in this section, we make use of Section \ref{app:duality}, where supplementary materials for duality is presented.  

\subsection{Proof of Lemma \ref{thm:dual-feasibility}}
\label{sec:duality-proof}

At a high level, we will show that the maximization problem \eqref{eq:feasibility-dual} is  the dual problem associated with the primal minimization problem \eqref{eq:feasibility-cone} as discussed in Section \ref{sec:constraint-qualification}. However, the primal minimization problem \eqref{eq:feasibility-cone} does not satisfy the constraint qualification conditions required by Proposition \ref{prop:strong-duality-general} and hence establishing strong duality directly is problematic. To overcome this challenge, we will show strong duality for a slightly perturbed distance function. By letting the perturbation vanish asymptotically, we will then obtain the stated strong duality result.

We introduce the following perturbed distance functions
\begin{equation}
  \label{eq:perturbed-problem}
  \begin{array}{rl}
    \dis_{\gamma}( \eta, x', P) = \min & \sum_{x\in \tilde X(P)} \eta(x) I(P(x), Q(x)) + \gamma I(P(x^\star(P)), Q(x^\star(P)))\\
    \st  & Q \in \mathcal K, \\
                                       & \sum_{x\in X,\, r\in \setr} ~Q(r, x) = \abs{X}, \\
                                       & \sum_{r \in \setr} r Q(r, x')  \geq \rew^\star(P)
  \end{array}
\end{equation}
where the minimum is achieved, as by Lemma \ref{prop:relative-entropy}, its objective function is lower semi-continuous in $Q$ while its feasible domain is compact. The perturbed distance function can indeed be regarded as a penalty formulation of the conic minimization problem \eqref{eq:feasibility-cone} with positive penalty term $I(P(x^\star(P)), Q(x^\star(P)))\geq 0$. That is, we have $I(P(x^\star(P)), Q(x^\star(P)))=0$ for any $P(x^\star(P))=Q(x^\star(P))$ and $I(P(x^\star(P)), Q(x^\star(P)))>0$ otherwise. It is hence evident by comparing problems \eqref{eq:perturbed-problem} and (\ref{eq:feasibility-cone}) that
\begin{equation*}
  \begin{array}{rl@{\hspace{4em}}}
    \dis_\gamma(\eta, x', P)\leq \dis(\eta, x', P) = \min_{Q} & \sum_{x\in \tilde X(P)} \eta(x) I(P(x), Q(x))  \\
    \st  & Q \in \mathcal K, \\
                                                              & \sum_{x\in X,\, r\in \setr} Q(r, x) = \abs{X},  \\
                                                              & \sum_{r \in \setr} r Q(r, x')  \geq \rew^\star(P), \\
                                                              & Q(x^\star(P))= P(x^\star(P))
  \end{array}
\end{equation*}
for any $\gamma\geq 0$ and deceitful arm $x'\in \xd(P)$. From \citet[Theorem 1.3.1]{nesterov1998introductory} it also follows that $\dis(\eta, x', P)=\lim_{\gamma\to\infty}\dis_{\gamma}(\eta, x', P)$ as the sublevel sets of the objective function in the perturbed minimization problem \eqref{eq:perturbed-problem} are bounded for any $\gamma>0$.

We will identify a Lagrangian dual formulation for the perturbed minimization problem (\ref{eq:perturbed-problem}) with the help of the following perturbed dual function
\begin{align*}
  \dual_\gamma( \eta,x', P~;~ \mu) := & \hspace{-0.6em}\sum_{x\in \tilde X(P), r\in \setr}   \chi_{-\infty}( \eta(x)\geq \lambda(r, x)+\beta+\alpha \cdot r \mb 1(x=x')) + \alpha \rew^{\star}(P) +\abs{ X}\beta+\\
                                   &  \hspace{-0.5em}\sum_{x\in \tilde X(P), r\in \setr}   \eta (x)  \log\left( \frac{ \eta (x)-\lambda(r, x)-\beta-\alpha \cdot r \mb 1(x=x')}{ \eta (x)}\right) P(r, x) + \\
                                    &  \hspace{-0.5em} \sum_{x\in  x^\star(P), r\in \setr}   \gamma  \log\left( \frac{ \gamma-\lambda(r, x)-\beta}{ \gamma}\right) P(r, x) + \hspace{-0.6em}\sum_{x\in x^\star(P), r\in \setr} \hspace{-0.8em}\chi_{-\infty}( \gamma \geq \lambda(r, x)+\beta).
\end{align*}
Indeed, using this perturbed dual function we will show the following strong duality result.

\begin{lemma}[Dual Formulation of the Perturbed Distance Function]
  \label{thm:dual-feasibility-perturbed}
  For any reward distribution $\optp\in \mathcal P$ and  suboptimal deceitful arm  $x'\in  \xd(\optp)$, the following strong duality equality holds.
  \begin{equation}
    \begin{array}{rl}
      \dis_\gamma( \eta,  x', \optp) = \sup_{\mu} & \dual_\gamma( \eta,  x', \optp~; ~\mu) \\
      \st & \mu \in \real_+\times\real\times \mathcal K^\star.
    \end{array}
  \end{equation}
\end{lemma}

We indicate first that the strong duality result stated in Lemma \ref{thm:dual-feasibility-perturbed} between the perturbed distance function and perturbed dual function implies the claimed result as an asymptotic case by letting $\gamma\to \infty$. Indeed, as the functions $\dis_\gamma$ and $\dual_\gamma$ are both non-decreasing in $\gamma$, and our strong duality result, we can write
\begin{align*}
  \dis(\eta, x', P) & = \lim_{\gamma\to \infty} \dis_{\gamma}(\eta,x', P) = \sup_{\gamma\geq 0} \dis_{\gamma}(\eta, x', P) \\
                    & = \sup_{\gamma\geq 0} \sup_{\mu \in \real_+\times \real\times \mathcal K^\star}\dual_\gamma( \eta,x', P~;~ \mu ) \\
                    & = \sup_{\mu \in \real_+\times \real\times \mathcal K^\star, \gamma \ge 0} \dual_\gamma( \eta,x', P~;~ \mu )\\
                    & = \sup_{\mu \in \real_+\times \real\times \mathcal K^\star} \lim_{\gamma\to\infty} \dual_\gamma( \eta, x', P~;~\mu)
\end{align*}
where the final limit is identified with
\begin{align*}
  \dual( \eta, x', P~;~ \mu) = &  \lim_{\gamma\to\infty} \dual_{\gamma}(\eta, x', P~;~ \mu)\\
  = & \textstyle\sum_{x\in \tilde X(P), r\in \setr} \eta(x) \log\left(\frac{\left( \eta(x)-\lambda(x, r)-\beta-\alpha \cdot r \mb 1(x=x')\right)}{ \eta (x)}\right) P(r,x) \\
                                   & \quad - \textstyle \sum_{x\in x^\star(P),~r\in \setr} \lambda(r, x)P(r, x) + \alpha \rew^{\star}(P) + \beta|\tilde X(P)|\\
                                   & \quad  + \textstyle \sum_{x\in \tilde X(P), r\in \setr} \chi_{-\infty}(  \eta (x)\geq \lambda(r, x)+\beta+\alpha \cdot r \mb 1(x=x'))
\end{align*}
from which the claimed result follows. \hfill \Halmos

\subsection{Proof of  Lemma \ref{thm:dual-feasibility-perturbed}}
  We will prove Lemma \ref{thm:dual-feasibility-perturbed} as a special case of the strong duality result stated in  Proposition \ref{prop:strong-duality-general} stated in Section \ref{sec:constraint-qualification}. We first derive the associated dual function for the perturbed problem \eqref{eq:perturbed-problem}. Afterwards, we verify that the perturbed problem \eqref{eq:perturbed-problem} satisfies the required constraint qualification conditions necessary to invoke Proposition \ref{prop:strong-duality-general}.

\textbf{Perturbed dual function.} The Lagrangian function associated with the perturbed  minimization problem (\ref{eq:perturbed-problem}) which we introduce in Section \ref{sec:constraint-qualification} is
\begin{align*}
  L_{\gamma} (\eta, x', P, Q~;~ \mu)=&\sum_{x\in \tilde X(P), r\in \setr} \left[ \eta (x) P(r, x) \log\left(\frac{P(r, x)}{Q(r, x)}\right) -  \eta (x) P(r, x) +  \eta (x) Q(r, x)\right] \\
  & \qquad + \sum_{x\in x^\star(P), r\in \setr} \left[ \gamma P(r, x) \log\left(\frac{P(r, x)}{Q(r, x)}\right) -  \gamma P(r, x) +  \gamma Q(r, x)\right] \\
                                    & \qquad - \sum_{x\in X, r\in \setr}  Q(r, x) \lambda(r, x)  + \beta \Big(\abs{X}-\sum_{x\in X, r\in \setr} Q(r, x)\Big) \\
                  & \qquad                 +  \alpha\Big(\rew^{\star}(P)-\sum_{r\in \setr} r Q(r, x') \Big)
\end{align*}
where $\mu= (\alpha, \beta, \lambda)\in \real_+\times \real\times \mathcal K^\star$.
The Lagrangian function can be simplified significantly by grouping terms to
\begin{align*}
&L_{\gamma} (\eta, x', P, Q~;~ \mu)=\\
 & \qquad \sum_{x\in  \tilde X(P), r\in \setr}\left[ Q(r, x) \big( \eta(x)-\lambda(r, x)-\beta-\alpha \cdot r \mb 1(x=x')\big)+ \eta (x) P(r, x) \log\left(\frac{P(r, x)}{Q(r, x)}\right)\right]\\
                                     & \qquad + \sum_{x\in  x^\star(P), r\in \setr}\left[ Q(r, x) \big( \gamma-\lambda(r, x)-\beta\big)+ \gamma P(r, x) \log\left(\frac{P(r, x)}{Q(r, x)}\right)\right]\\
      & \qquad + \alpha \rew^{\star}(P) +\abs{X}\beta-\sum_{x\in \tilde X(P)}  \eta (x) - \abs{x^\star(P)}\gamma\,
\end{align*}
where we used the fact that we only are interested in deceitful arms which by definition can not be optimal, i.e., $x'\in \xd(P)\implies x'\not\in x^\star(P)$.
The dual function associated with the perturbed problem \eqref{eq:perturbed-problem} is 
$
\dual_\gamma( \eta, x', P~;~\mu ) \defn \min_{Q\geq 0}\,   L_{\gamma} (\eta, x', P, Q~;~ \mu)
$
 and  satisfies the weak duality inequality $\dual_\gamma( \eta, x', P~;~\mu )\leq \dis_\gamma(\eta, x', P)$. The dual function can be expressed equivalently as
\begin{align*}
  & \dual_\gamma( \eta, x', P~;~ \mu)\\
  = &  \sum_{x\in  \tilde X(P), r\in \setr}\inf_{q\geq 0}\left[ q \big( \eta (x)-\lambda(r, x)-\beta-\alpha \cdot r \mb 1(x=x')\big )+ \eta (x) P(r, x) \log\left(\frac{P(r, x)}{q}\right)\right] \\
  & \qquad + \sum_{x\in  x^\star(P), r\in \setr}\inf_{q\geq 0}\left[ q \big( \gamma-\lambda(r, x)-\beta\big )+ \gamma P(r, x) \log\left(\frac{P(r, x)}{q}\right)\right] \\
                                                       & \qquad + \alpha \rew^{\star}(P) +\abs{X}\beta-\sum_{x\in \tilde X(P)}  \eta(x) - \abs{x^\star(P)}\gamma  \\
  =&  \sum_{\{x\in \tilde X(P), r: \eta(x)P(r, x)=0\} } \chi_{-\infty}( \eta (x)\geq \lambda(r, x)+\beta+\alpha \cdot r \mb 1(x=x'))\\
  &  \qquad +\sum_{\{x\in \tilde x^\star(P), r: \gamma P(r, x)=0\} } \chi_{-\infty}( \gamma\geq \lambda(r, x)+\beta+\alpha \cdot r \mb 1(x=x')) \\
  & \qquad +\sum_{\{x\in \tilde X(P), r: \eta(x)P(r, x)>0\}}P(r, x)\inf_{\theta \geq 0}  \Big[ \theta ( \eta (x)-\lambda(r, x)-\beta-\alpha \cdot r \mb 1(x=x'))- \eta (x) \log(\theta)\Big] \\
  & \qquad +\sum_{\{x\in  x^\star(P), r: \gamma P(r, x)>0\}}P(r, x)\inf_{\theta \geq 0}  \Big[ \theta ( \gamma-\lambda(r, x)-\beta)- \gamma \log(\theta)\Big]\\
                                                       & \qquad + \alpha \rew^{\star}(P) +\abs{X}\beta-\sum_{x\in \tilde X(P)}  \eta(x)- \abs{x^\star(P)}\gamma\,,
\end{align*}
where we obtain the last equality by a change of variables $q$ to $P(r, x) \theta$. Recall that $\chi_{-\infty}(A)$ takes on the value $0$  when event $A$ happens  and $-\infty$ otherwise. 
Finally, using the convex conjugacy relationship $\max_{\theta\geq 0} u \theta + \log(\theta) = -(1+\log(-u))+\chi_{+\infty}(u\leq 0)$,
we can simplify the perturbed dual function to
\begin{align*}
  & \dual_\gamma( \eta,x', P~;~ \mu)\\
  = & \sum_{x\in \tilde X(P), r\in \setr}   \chi_{-\infty}( \eta(x)\geq \lambda(r, x)+\beta+\alpha \cdot r \mb 1(x=x'))\\
  & \qquad + \sum_{x\in x^\star(P), r\in \setr}   \chi_{-\infty}( \gamma \geq \lambda(r, x)+\beta+\alpha \cdot r \mb 1(x=x'))\\
  & \qquad +\sum_{\{x\in \tilde X(P), r: \eta(x)P(r, x)>0\}} P(r, x)  \eta (x)  \left[1+\log\left( \frac{ \eta (x)-\lambda(r, x)-\beta-\alpha \cdot r \mb 1(x=x')}{ \eta (x)}\right)\right]\\
  & \qquad +\sum_{\{x\in x^\star(P), r: \gamma P(r, x)>0\}} P(r, x)  \gamma  \left[1+\log\left( \frac{ \gamma-\lambda(r, x)-\beta}{ \gamma}\right)\right]\\
                                                      & \qquad + \alpha \rew^{\star}(P) +\abs{X}\beta-\sum_{x\in \tilde X(P)}  \eta(x)- \abs{x^\star(P)}\gamma \\
  = & \sum_{x\in \tilde X(P), r\in \setr}   \chi_{-\infty}( \eta(x)\geq \lambda(r, x)+\beta+\alpha \cdot r \mb 1(x=x'))\\
  & \qquad + \sum_{x\in x^\star(P), r\in \setr}   \chi_{-\infty}( \gamma \geq \lambda(r, x)+\beta+\alpha \cdot r \mb 1(x=x'))\\
    & \qquad +\sum_{x\in \tilde X(P), r\in \setr}   \eta (x)  \log\left( \frac{ \eta (x)-\lambda(r, x)-\beta-\alpha \cdot r \mb 1(x=x')}{ \eta (x)}\right) P(r, x) \\
  & \qquad +\sum_{x\in  x^\star(P), r\in \setr}   \gamma  \log\left( \frac{ \gamma-\lambda(r, x)-\beta}{ \gamma}\right) P(r, x) \\
  & \qquad + \alpha \rew^{\star}(P) +\abs{X}\beta. 
\end{align*}

\textbf{Strong duality perturbed problem.}
We can obtain the strong duality result claimed in Lemma \ref{thm:dual-feasibility-perturbed}, i.e., the equality 
$$\sup_{\mu \in \real_+\times \real\times \mathcal K^\star} \dual_\gamma( \eta, P, x'~;~\mu ) = \dis_\gamma(\eta,  x', P)$$ by verifying that the Slater's constraint qualification conditions  required by Proposition \ref{prop:strong-duality-general} are satisfied.
In the following, we find a Slater point for problem \eqref{eq:perturbed-problem}. Precisely, we identify  a distribution $\bar Q$ in the interior of $\mathcal P$ which  is strictly feasible, i.e.,  $\sum_{r \in \setr} r \bar Q(r, x') > \rew^{\star}(P)$.
To do so, we consider any distribution $U$ in the interior  of $\mathcal P$; that is, $U\in \interior(\mathcal P)$. Finding such an interior point is possible as we take that $\interior(\mathcal P)\neq \emptyset$ as a standing assumption.

Let the distribution $Q_{\max}$ be a distribution which satisfies  $\sum_{r\in \setr} r Q_{\max}(r, x')=\rew_{\max}(x', P)>\rew^\star(P)$ were the previous strict inequality is guaranteed by the fact that $x'$ is a deceitful arm. Define $Q_\theta \defn Q_{\max} \cdot \theta +  U \cdot (1-\theta) \in \mathcal P$ for all $\theta\in [0, 1]$. We will show that for some values of $\theta$, the reward distribution ${Q_\theta}$ can serve as a Slater point. Let
  \[
    \theta' = \frac{\rew^{\star}(P)-\sum_{r\in \setr}r  U(r, x')}{\rew_{\max}(x', P)-\sum_{r\in \setr}r  U(r, x')} \in [-\infty, 1).
  \]
  Observe that for any $\theta \in (\theta', 1]\bigcap [0, 1)$ , we have that the distribution $Q_\theta\in \interior(\mathcal P)$ satisfies $\sum_{r \in \setr} r Q_\theta(x', r) > \rew^{\star}(P)$. Hence, the particular distribution $\bar Q=Q_{\bar \theta}$ with $\bar \theta= \tfrac{(\max\{\theta', 0\}+1)}{2}$ serves as a Slater point for problem \eqref{eq:perturbed-problem}. Note  that as $\bar Q\in \interior(\mathcal P)$ is also in the interior of $\mathcal P_\Omega$ and hence in the interior of the domain of the objective function of problem \eqref{eq:perturbed-problem}.
\hfill\Halmos

\subsection{ More on the Sufficient Information Condition}\label{sec:info}

The dual-test function provides a convenient procedure to verify if enough information has been collected already to distinguish the empirical  $P_t$ from its deceitful models; that is, whether or not $\dis(\tfrac{N_{t}}{\log(t)}, x', P_t)\ge  \overline{\dual}\left(\tfrac{N_{t}}{\log(t)}, x', P_t~;~\mu_t\right)\ge 1+\epsilon$ for any deceitful arm $x'\in \xd(P_t)$.
The strong duality results in Lemma \ref{thm:dual-feasibility}, relates the information distance between $P_t$ and $\dis(x', P_t)$ and the dual function as \[\dis(\tfrac{N_{t}}{\log(t)}, x', P_t)= \max_{\mu(x')} \set{\dual(\tfrac{N_{t}}{\log(t)},  x', P_t~; ~\mu(x'))}{\mu(x')\in \real_+\times\real\times\mc K^\star}.\]
Our sufficient information test however involves instead a seemingly unrelated restricted dual problem
\begin{align*}
  & \overline{\dual}\left(\tfrac{N_{t}}{\log(t)}, x', P_t~;~\mu_t\right)\\[0.5em]
  = & \max_{\rho\geq 0,\, \mu(x')} \set{\dual(\tfrac{N_{t}}{\log(t)},  x', P_t~; ~\mu(x'))}{\rho\geq 0,\, \mu(x') = \rho\cdot \mu_t(x') \in \real_+\times\real\times\mc K^\star}\\[0.5em]
  = & ~~\max_{\rho\geq 0} ~\,\set{\dual(\tfrac{N_{t}}{\log(t)},  x', P_t~; ~\rho\cdot \mu_t(x'))}{\rho\geq 0}.
\end{align*}
Nevertheless, we will argue here that also this restricted dual problem admits a satisfying interpretation as the information distance condition between $P_t$ and a certain set of deceitful distributions. Indeed, for any feasible dual variable $\mu = (\alpha, \beta, \lambda)$, arbitrary empirical reward distribution $P_t\in \mc P$, and deceitful arm $x'\in \xd(P_t)$, we define $\mc H(x', P_t~;~\mu)$ as the following set of reward distributions:
\begin{align}
  \label{eq:hyper}
    \set{Q}{
    \begin{array}{l}
      Q(x^\star(P_t))=P_t(x^\star(P_t)), \\[0.5em]
      \sum_{x\in X, \, r\in \setr} Q(r, x)\left(  \lambda(r, x) + \beta + \alpha r \mb 1(x=x') \right) \geq \beta \abs{X} + \alpha \rew^\star(P_t)
    \end{array}
  }\,.
\end{align}
We point out that the previously defined set  contains all deceitful distributions whenever the dual variable $\mu$ is dual feasible; that is, $\mc H(x', P_t~;~ \mu)\supseteq \prob(x', P_t)$, where 
the set of deceitful distributions $\prob(x', P_t)$ is presented  as the feasible set of the convex minimization problem \eqref{eq:feasibility-cone}.
Obviously, any deceitful distribution will satisfy $Q(x^\star(P_t))=P_t(x^\star(P_t))$ by construction. That is, the reward distributions of the empirically optimal arms of $P_t$ are in full agreement with those of $Q$.  We now show that any deceitful distribution will also satisfy the second constraint $\sum_{x\in X, \, r\in \setr} Q(r, x)\left(  \lambda(r, x) + \beta + \alpha r \mb 1(x=x') \right) \geq \beta \abs{X} + \alpha \rew^\star(P_t)$ characterizing the set $\mc H(x', P_t~; ~\mu)$. The first term $\sum_{x\in X, \, r\in \setr} Q(r, x)\lambda(r, x)$ is nonnegative because $Q\in \mc K$ and dual feasible variable $\lambda \in \mc K^\star$.
The second term satisfies $\sum_{x\in X, \, r\in \setr} \beta Q(r, x)=\beta \abs{X}$ irrespective of the value of the dual variable $\beta$; see the second constraints in the minimization problem \eqref{eq:feasibility-cone}. For any dual feasible variable $\alpha\geq 0$, the final term satisfies the inequality $\sum_{x\in X, \, r\in \setr} Q(r, x) r \mb 1(x=x') \geq \rew^\star(P_t)$; see the third constraints  of minimization Problem \eqref{eq:feasibility-cone}. Thus, we have indeed the claimed inclusion $\mc H(x', P_t~;~ \mu)\supseteq \prob(x', P_t)$. The following lemma  provides an interpretation  of the sufficient information condition.

\begin{lemma}[Dual-test Function]
  \label{lemma:resolving}
  Let $P\in \interior(\mc P)$ and $\eta>0$ and consider any deceitful  arm $x'\in \xd(P)$.  For any  dual feasible variable $\mu=(\alpha, \beta, \lambda)\in\real_+\times\real\times\mc K^\star$, the dual-test function is equal to
  \begin{equation}
    \label{eq:hypothesis-test-duality}
    \begin{array}{rl}
    \overline{\dual}(\eta,  x', P~; ~\mu) =  \min_{Q} & \textstyle \sum_{x\in \tilde X(P)} \eta(x) I(P(x), Q(x)) \\[0.5em]
          \st & Q \in \mc H(x', P~;~ \mu).
    \end{array}
  \end{equation} 
\end{lemma}
\proof{}
See Appendix \ref{sec:resolving-proof}.\hfill \Halmos
\endproof

The optimization problem \eqref{eq:hypothesis-test-duality}, which states an alternative characterization of our dual-test function, bears some resemblance to the minimization problem \eqref{eq:inner-optimization}, which characterizes the  distance function $\dis$. While the objective function of two optimization problems is the same,  in the minimization Problem  \eqref{eq:inner-optimization}, we enforce the distribution $Q$ to belong to the deceitful set  $\prob (x', P)$,  whereas in the minimization problem \eqref{eq:hypothesis-test-duality}, we enforce the same distribution belong to merely belong to the superset $\mc H(x', P~;~\mu)$.  
Thus, the  condition $1\leq \overline{\dual}(\eta,  x', P~; ~\mu)$ implies that sufficient information has been obtained to distinguish the null hypothesis $P$ not only from its deceitful  distributions $\prob(x', P)$ but in fact from the entire superset $\mc H(x', P~;~ \mu)$. The simple observation that distinguishing $P$ from the larger set $\mc H(x', P~;
~\mu)$ must necessarily be harder from a statistical point of view than distinguishing $P$ merely from the smaller set of deceitful  distributions $\prob(x', P)$ explains the weak duality inequality $\overline{\dual}(\eta,  x', P~; ~\mu) \leq \dis(\eta, x', P)$ for any feasible dual variable $\mu$. Strong duality can be interpreted as the existence of a feasible dual variable $\mu^\star$ so that the distance between $P$ and the sets $\mc H(x', P; \mu^\star)$ and $\prob(x', P)$ is the same. That is, $\overline{\dual}(\eta,  x', P~; ~\mu^\star) = \dis(\eta, x', P) = \sum_{x\in \tilde X(P)}\eta(x)I(P(x), Q^\star(x))$ for some worst-case deceitful reward distribution  $Q^\star \in \prob(x', P)$. Under those circumstances, the hyperplane $\mc H(x',P~;~\mu^\star)$ separates the deceitful reward distributions $\prob(x', P)$ and the information ball $\tset{Q}{Q(x^\star(P))=P(x^\star(P)),\,\sum_{x\in \tilde X(P)}\eta(x)I(P(x), Q(x))\leq \dis(\eta, P, x')}$ as implied by  the necessary and sufficient KKT optimality conditions of minimization problem \eqref{eq:inner-optimization} and visually illustrated in Figure \ref{fig:hypothesis}.

The previous discussion also offers additional  insight into the deep and shallow update algorithms introduced before.
The deep update Algorithm \ref{alg:deep} minimizes the regret over all exploration rates $\eta$ and dual variables   so that   the reward distribution $P$ can be distinguished from its deceitful distributions $\prob(x', P)$ with sufficient statistical power; that is, the deep update ensures that
$\dis(\eta, x', P)\geq 1$. 
The shallow update, on the other hand,  merely minimizes the regret over the logarithmic rates $\eta$  so that   the reward distribution $P$ can be distinguished from the superset $\mc H(x', P~;~\mu)$ associated with the dual variable $\mu$. That is, the shallow update ensures that $\overline{\dual}(\eta,  x', P~; ~\mu)\geq 1$.

\begin{figure}[ht]
  \centering
  \includegraphics[height=4.5cm]{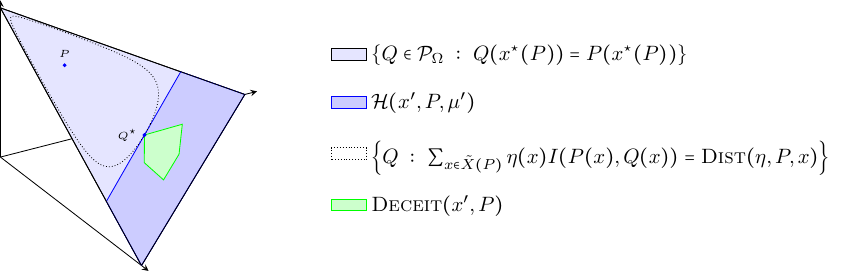}
  \caption{The sufficient information test $\overline{\dual}(\eta,  x', P~; ~\mu)\leq \dis(\eta, x', P)$ for dual feasible variable  $\mu$ quantifies the information distance between $P$ and the half-space $\mc H(x', P~; ~\mu')$ containing all deceitful distributions $\prob(x', P)$.
    When $\mc H(x', P~;~\mu^\star)$ for some $\mu^\star$ defines a separating half-space between $\prob(x', P)$ and the information ball $\tset{Q}{Q(x^\star(P))=P(x^\star(P)),\,\sum_{x\in \tilde X(P)}\eta(x)I(P(x), Q(x))\leq \dis(\eta, P, x)}$, the information distance between $P$ and either the set of deceitful distributions $\prob(x', P)$ or its superset $\mc H(x', P~; ~\mu')$ is the same. In that case, we have the strong duality equalities
    $\overline{\dual}(\eta,  x', P~; ~\mu^*) = \dis(\eta, x', P) = \sum_{x\in \tilde X(P)}\eta(x)I(P(x), Q^\star(x))$, where $Q^{\star}$ is the worst-case distribution, solving problem \eqref{eq:hypothesis-test-duality}.
  }
  \label{fig:hypothesis}
\end{figure}

\subsection{Proof of Lemma \ref{lemma:resolving}}
\label{sec:resolving-proof}

The information minimization problem \eqref{eq:hypothesis-test-duality} can be written explicitly as  
\begin{equation}
  \label{eq:hypothesis-test-duality-2}
  \begin{array}{rl@{~}l}
    \min & {\displaystyle\sum_{x\in \tilde X(P)}} \eta(x) I(P(x), Q(x)) & \\[1em]
    \st & {\displaystyle\sum_{x\in X, \, r\in \setr}} Q(r, x)\left[  \lambda(r, x) \!+\! \beta \!+\! \alpha r \mb 1(x=x') \right] \geq \beta \abs{X} + \alpha \rew^\star(P) & [{\mathrm{dual~variable\!:\,}} \rho] \\[1em]
         & Q(x^\star(P))=P(x^\star(P)) & [{\mathrm{not~dualized}}]
  \end{array}
\end{equation}
where we use the explicit characterization of $\mathcal H(x', P~;~\mu)$ given in Equation (\ref{eq:hyper}). We will claim first that the minimum in problem \eqref{eq:hypothesis-test-duality-2} is indeed attained. The remainder of the claimed results will be proven as a special strong duality result of the type discussed in Section \ref{sec:constraint-qualification}.

\paragraph{Attainment of minimum.} We first show that the minimum in \eqref{eq:hypothesis-test-duality-2} is in fact attained. The objective function of the information minimization problem has non-empty level sets
\begin{align*}
  & \set{Q\in \mathcal H(x', P~;~\mu)}{\textstyle\sum_{x\in \tilde X(P)}\eta(x) I(P(x), Q(x))\leq \gamma}
\end{align*}
for any level $\gamma\geq \dis(\eta, x', P)=\min_{Q\in \prob(x', P)}\sum_{x\in \tilde X(P)}\eta(x) I(P(x), Q(x))$ as indeed we have the set inclusion $\prob(x', P)\subseteq \mathcal H(x', P~;~\mu)$. These level sets are furthermore closed as the set $\mathcal H(x', P~;~\mu)$ is closed while the objective function is lower semi-continuous as per Lemma \ref{prop:relative-entropy}.
Recall that we assume that the logarithmic rate $\eta>0$ is positive. Hence, the level sets are also bounded as we have
\begin{align*}
  & \set{Q\in \mathcal H(x', P~;~\mu)}{\textstyle\sum_{x\in \tilde X(P)}\eta(x) I(P(x), Q(x))\leq \gamma}\\
  & \subseteq \set{Q}{\eta(x) I(P(x), Q(x))\leq \gamma ~~\forall x\in \tilde X(P), ~Q(x^\star(P))=P(x^\star(P))}.
\end{align*}
\citet[Prop. 3.2.1]{bertsekas2009convex} guarantees that the set of minima in \eqref{eq:hypothesis-test-duality-2} is non-empty and compact.
  
\paragraph{Dual function.} The associated Lagrangian to the information minimization problem \eqref{eq:hypothesis-test-duality-2} introduced in Section \ref{sec:constraint-qualification} is equal to
\[
  \begin{array}{r@{~}l}
    L(\eta, x', P, Q~;~   \rho) \defn & \sum_{x\in \tilde X(P)} \eta(x) I(P(x), Q(x))  \\
                                           & ~~ + \rho (\sum_{x\in X, \, r\in \setr} Q(r, x)\left[-\lambda(r, x) - \beta - \alpha r \mb 1(x=x') \right] + \beta \abs{X} + \alpha \rew^\star(P))
  \end{array}
\]
which can be simplified further by collecting terms to
\[
  \begin{array}{r@{~}l}
    L(\eta, x', P, Q~;~   \rho) = & \sum_{x\in \tilde X(P), r\in \setr}Q(r, x)(-\rho \lambda(r, x) - \rho \beta-\rho\alpha r \mb 1(x=x') )\\
    &\quad +\sum_{x\in \tilde X(P)}\eta(x) I(P(x), Q(x))  \\
                                           & \quad + \sum_{x\in x^\star(P), r\in \setr} Q(r, x)(-\rho \lambda(r, x) - \rho \beta)+ \rho \beta \abs{X}+\rho \alpha \rew^\star(P).
  \end{array}
\]
The dual function associated with the perturbed problem \eqref{eq:perturbed-problem} is as discussed in Section \ref{sec:constraint-qualification} identified with 
$
g( \eta, x', P~;~\rho ) \defn \min_{Q(x^\star(P))=P(x^\star(P))}\,   L (\eta, x', P, Q~;~ \rho).
$
The dual function can be expressed equivalently as
\begin{align*}
  & g( \eta, x', P~;~ \mu)\\
  = & \min_{Q(x^\star(P))=P(x^\star(P))} \sum_{x\in \tilde X(P), r\in \setr}Q(r, x)(-\rho \lambda(r, x) - \rho \beta-\rho\alpha r \mb 1(x=x') ) \\
    & \quad +\sum_{x\in \tilde X(P), r\in \setr}\eta(x) [P(r, x) \log\left(\textstyle\frac{P(r, x)}{Q(r, x)}\right) - P(r, x) +  Q(r, x)]  \\
  & \quad + \sum_{x\in x^\star(P), r\in \setr}Q(r, x)(-\rho \lambda(r, x) - \rho \beta)+ \rho \beta \abs{X}+\rho \alpha \rew^\star(P)\\
  = &    \sum_{x\in  \tilde X(P), r\in \setr}\inf_{q\geq 0}\left[ q \big( \eta (x)-\rho \lambda(r, x)-\rho \beta-\rho \alpha \cdot r \mb 1(x=x')\big )+ \eta (x) P(r, x) \log\left(\frac{P(r, x)}{q}\right)\right] \\
  & \qquad - \sum_{x\in x^\star(P)}\rho\lambda(r, x)P(r ,x)+ \rho \alpha \rew^{\star}(P) +\abs{\tilde X(P)}\beta\rho-\sum_{x\in \tilde X(P)}  \eta(x) \\
  =&  \sum_{\{x\in \tilde X(P), r: \eta(x)P(r, x)=0\} } \chi_{-\infty}( \eta (x)\geq \rho \lambda(r, x)+\rho \beta+\rho \alpha \cdot r \mb 1(x=x'))\\
  & \qquad +\sum_{\{x\in \tilde X(P), r: \eta(x)P(r, x)>0\}}\hspace{-3em}P(r, x)\inf_{\theta \geq 0}  \Big[ \theta ( \eta (x)-\rho \lambda(r, x)-\rho \beta-\rho \alpha \cdot r \mb 1(x=x'))- \eta (x) \log(\theta)\Big] \\
  & \qquad - \sum_{x\in x^\star(P)}\rho\lambda(r, x)P(r ,x)+ \rho \alpha \rew^{\star}(P) +\abs{\tilde X(P)}\rho \beta-\sum_{x\in \tilde X(P)}  \eta(x)
\end{align*}
where we obtain the last equality by a change of variables $q$ to $P(r, x) \theta$. Recall that $\chi_{-\infty}(A)$ takes on the value $0$  when event $A$ happens  and $-\infty$ otherwise. 
Finally, using the convex conjugacy relationship $\max_{\theta\geq 0} u \theta + \log(\theta) = -(1+\log(-u))+\chi_{+\infty}(u\leq 0)$,
we can simplify the perturbed dual function to
\begin{align*}
  & g( \eta,x', P~;~ \rho)\\
  = & \sum_{x\in \tilde X(P), r\in \setr}   \chi_{-\infty}( \eta(x)\geq \rho \lambda(r, x)+\rho\beta+\rho \alpha \cdot r \mb 1(x=x'))\\
  & \qquad +\sum_{\{x\in \tilde X(P), r: \eta(x)P(r, x)>0\}} P(r, x)  \eta (x)  \left[1+\log\left( \frac{ \eta (x)-\rho \lambda(r, x)-\rho \beta-\rho \alpha \cdot r \mb 1(x=x')}{ \eta (x)}\right)\right]\\
  & \qquad - \sum_{x\in x^\star(P)}\rho\lambda(r, x)P(r ,x) + \rho \alpha \rew^{\star}(P) +\rho \abs{\tilde X(P)}\beta-\sum_{x\in \tilde X(P)}  \eta(x) \\
  = & \sum_{x\in \tilde X(P), r\in \setr}   \chi_{-\infty}( \eta(x)\geq \rho \lambda(r, x)+\rho \beta+\rho \alpha \cdot r \mb 1(x=x'))\\
    & \qquad +\sum_{x\in \tilde X(P), r\in \setr}   \eta (x)  \log\left( \frac{ \eta (x)-\rho \lambda(r, x)-\rho \beta-\rho \alpha \cdot r \mb 1(x=x')}{\eta (x)}\right) P(r, x) \\
  & \qquad - \sum_{x\in x^\star(P)}\rho\lambda(r, x)P(r ,x) + \rho \alpha \rew^{\star}(P) +\abs{\tilde X(P)}\rho \beta. \\
  = &\dual(\eta, x', P ~;~ \rho\cdot \mu).
\end{align*}

\paragraph{Strong duality.}We can now obtain the strong duality equality
\[
  (\ref{eq:hypothesis-test-duality-2}) = \max_{\rho\geq 0}~ \dual(\eta, x', P ~;~ \rho\cdot \mu)
\]
via Proposition \ref{prop:strong-duality-general}. Recall that the set $\prob(x', P)\subseteq \mathcal H(x', P, \mu)$ and hence in order to apply Proposition \ref{prop:strong-duality-general}, it suffices to construct a distribution $\bar Q\in \interior(\mathcal P_\Omega)$ which satisfies
  \[
    \textstyle \sum_{r\in \setr} \bar Q(r, x')>\rew^\star(P) \quad {\mathrm{and}} \quad \bar Q(x^\star(P))=P(x^\star(P)).
  \]
  Let $Q_\theta = \theta Q^\star + (1-\theta) P\in \mathcal P$ with $Q^\star$ a minimizer in problem \eqref{eq:feasibility-cone}. We will show there for some values of $\theta$, the reward distribution ${Q_\theta}$ can serve as a Slater point. Let
  \[
    \theta' = \frac{\rew^{\star}(P)-\sum_{r\in \setr}r  P(r, x')}{\rew_{\max}(x', P)-\sum_{r\in \setr}r  P(r, x')} \in [-\infty, 1).
  \]
  Observe that as $P\in\interior(\mathcal P)$ for any $\theta \in (\theta', 1]\bigcap [0, 1)$, we have that the distribution $Q_\theta\in \interior(\mathcal P)\subseteq \interior(\mathcal P_\Omega)$ satisfies $\sum_{r \in \setr} r Q_\theta(x', r) > \rew^{\star}(P)$. Hence, the particular distribution $\bar Q=Q_{\bar \theta}$ with $\bar \theta= \tfrac{(\max\{\theta', 0\}+1)}{2}$ serves as a Slater point for problem \eqref{eq:hypothesis-test-duality-2}.

  \subsection{Dual Cones for Examples}
  \label{sec:dual:running-examples}

\begin{running}[Dual Cones]
  In what follows, we discuss the cones $\mathcal K$ and its dual $\mathcal K^\star$ in the context of some of the example bandits discussed in Section \ref{sec:struct-band-ex}. 
Detailed derivation of dual cones are deferred to Section \ref{sec:dual_cones_proof}.
  \begin{enumerate}
    \setlength\itemsep{1em}
  \item[1--2.] \textit{Generic and separable bandits.} 
    Recall that the feasible set $\mathcal{P}$ of separable bandits is $\mathcal P= \textstyle \prod_{x\in X}\mathcal P_x$, where $\mathcal P_{x}$, $x\in X$, is a closed  feasible set for the reward distribution of arm $x$.
    The cone $\mathcal K$ associated  with  feasible set $\mathcal{P}$ can be written as
    \begin{align}\label{eq:cone:sep}
      \mathcal K & = \set{Q\in \real^{\abs{\setr}\times \abs{X}}_+}{
              \begin{array}{lr}
                \sum_{r\in \setr} Q(r, x)=\theta & \forall x\in X,\\
                Q(x)\in \cone(\mathcal P_x) & \forall x\in X.
              \end{array}}\,,
    \end{align}
    The dual cone of the separable bandit cone $\mathcal K$ in Equation \eqref{eq:cone:sep} is given as
    \[
      \mathcal K^\star = \set{\lambda\in \real^{\abs{\setr}\times\abs{X}}}{
        \begin{array}{lr}
          \exists \gamma(x)\in \real & \forall x\in X, \\
          \lambda(x)-\textbf {1}\cdot \gamma(x) \in \cone(\mathcal P_x)^\star  & \forall x\in X, \\
          \sum_{x\in X}\gamma(x) = 0 &
        \end{array}
      }\,,
    \]
    where  $\textbf{1}$ is the all ones vector with a size of $|\setr|$. This implies that when  $\cone(\mathcal P_x)^\star$ is efficiently representable for each arm, then $\mathcal K^{\star}$ admits an efficient barrier function.
  
  \item[3.] \textit{Lipschitz Bernoulli bandits.} Recall that  the feasible set of Lipschitz Bernoulli Bandits is given by $
    \mathcal P_{\mathrm{Lips}} = \set{P\in \mathcal P_\Omega}{
      \begin{array}{lr}
        P(1, x) - P(1, x') \leq L \cdot d(x, x') & \forall x, x' \in X
      \end{array}
    }
$. Then, one can observe that the 
  cone $\mathcal K$ associated with $\mathcal P_{\mathrm{Lips}}$ is given by
    \begin{align}\label{eq:cone:lip}
      \mathcal K & = \set{Q\in \real^{2\times \abs{X}}_+}{
              \begin{array}{lr}
                \exists \theta \in \real, & \\
                Q(0, x)+Q(1, x)=\theta & \forall x\in X, \\
                Q(1, x) - Q(1, x') \leq \theta \cdot L \cdot d(x, x') & \forall x, x' \in X
              \end{array}}\,.
    \end{align}
    Note that the last constraint enforces the Lipschitz property of the reward distributions.   
    The dual cone of the Lipschitz bandits  cone $\mathcal K$ in Equation \eqref{eq:cone:lip} is given as
    \begin{align*}
      \mathcal K^\star & =  \set{\lambda \in \real^{2\times \abs{X}}}{
                    \begin{array}{lr}
                      \exists \Lambda(x, x')\in \real_+, ~\exists \gamma(x)\in \real & \forall  x, x'\in X,\\
                      \lambda(1, x) + \sum_{x'\in X}(\Lambda(x, x')-\Lambda(x', x)) \geq \gamma(x) & \forall x\in X,\\
                      \lambda(0, x) \geq \gamma(x)  & \forall x\in X,\\
                      \sum_{x\in X}\gamma(x) = L \sum_{x, x'\in X} d(x, x') \Lambda(x, x')
                    \end{array}}.
    \end{align*}
    The cone $\mathcal K^\star$ is here a polyhedral set implying that $\mathcal K^{\star}$ admits an efficient barrier function.    
  \item[4.] {\color{black}\textit{Parametric (linear) bandits.} Recall that the feasible set of linear bandits is given by
    \(
    \mc P_{\rm{lin}}=\set{Q\in \mc P_\Omega}{\forall x\in X, ~\exists \theta~\st~ \textstyle\sum_{r\in\setr} r Q(r, x) = c_x^T \theta}.
    \)
    Then, one can observe that the cone $\mathcal K$ associated with $\mathcal P_{\mathrm{lin}}$ is given by
    \begin{align}\label{eq:cone:linear}
      \mathcal K & = \set{Q\in \real^{\abs{\setr}\times \abs{X}}_+}{
              \begin{array}{lr}
                \exists \theta, \theta' \in \real, & \\
                \sum_{r\in\setr}r Q(r, x)=\theta' & \forall x\in X, \\
                \sum_{r\in\setr}r Q(r, x)=c_x^T \theta & \forall x \in X
              \end{array}}\,.
    \end{align}
    The dual cone of the linear bandits cone $\mathcal K$ in Equation \eqref{eq:cone:linear} is given as
    \begin{align*}
      \mathcal K^\star & =  \set{\lambda \in \real^{\abs{\setr}\times \abs{X}}}{
                         \begin{array}{lr}
                           \exists \nu(x) \in \real, ~\exists \gamma(x)\in \real & \forall  x \in X,\\
                           \lambda(r, x) \geq \gamma(x) +r \nu(x) & \forall r\in \setr, x\in X,  \\
                           \sum_{x\in X}\gamma(x) = 0,~ \sum_{x\in X} c_x \nu(x) = 0
                         \end{array}}.
    \end{align*}
    The cone $\mathcal K^\star$ is here a polyhedral set,  implying that $\mathcal K^{\star}$ admits an efficient barrier function.}
    \item[5.] {\color{black}\textit{Dispersion bandits.}
Recall that the feasible set of our dispersion bandits is given by
\(
\mc P_{\rm{dis}}=\set{Q\in \mc P_\Omega}{\tfrac{\sum_{r\in\setr} r^2Q(r, x)}{\sum_{r\in\setr }r Q(r, x) }\le \gamma(x) ~~ \forall x\in X}.
\)
Then, one can observe that the cone $\mathcal K$ associated with $\mathcal P_{\mathrm{dis}}$ is given by
\begin{align}\label{eq:cone:dispersion}
  \mathcal K & = \set{Q\in \real^{\abs{\setr}\times \abs{X}}_+}{
               \begin{array}{lr}
                 \exists \theta \in \real, & \\
                 \sum_{r\in\setr}r Q(r, x)=\theta & \forall x\in X, \\
                 \sum_{r\in\setr}r^2 Q(r, x)\leq\gamma(x) \sum_{r\in\setr}r Q(r, x) & \forall x \in X
               \end{array}}\,.
\end{align}
The dual cone of the linear bandits cone $\mathcal K$ in Equation \eqref{eq:cone:linear} is given as
\begin{align*}
  \mathcal K^\star & =  \set{\lambda \in \real^{\abs{\setr}\times \abs{X}}}{
                     \begin{array}{lr}
                       \exists \nu(x) \in \real_+, ~\exists \mu(x)\in \real & \forall  x \in X,\\
                       \lambda(r, x) \geq \mu(x) -(r^2-\gamma(x)r) \nu(x) & \forall r\in \setr, x\in X,  \\
                       \sum_{x\in X}\mu(x) = 0
                     \end{array}}.
\end{align*}
The cone $\mathcal K^\star$ is here a polyhedral set, implying that $\mathcal K^{\star}$ admits an efficient barrier function.}
  \end{enumerate}
\end{running}

\subsubsection{Detailed Derivation of Dual Cones }\label{sec:dual_cones_proof}
\textit{Dual cone of generic and separable bandits.} From the definition of dual cones, it  follows that $\lambda\in \mathcal K^\star \iff 0\leq \min_{Q\in \mathcal K}\iprod{\lambda}{Q}$, where for separable bandits, the cone $\mathcal K$ is given in Equation \eqref{eq:cone:sep}.  
Here, we have that the previous primal minimization problem $\min_{Q\in \mathcal K}\iprod{\lambda}{Q}$ can be characterized as
\begin{equation}
  \label{eq:dual-cone-separable-bandits}
  \begin{array}{rl@{\hspace{2em}}l@{\hspace{3em}}r}
    \min_{Q, \theta} & \iprod{\lambda}{Q} & & \\
    \st & \sum_{r\in \setr}Q(r, x) = \theta & \forall x\in X, & [\mathrm{dual~variable~}\gamma\in \real^{\abs{X}}] \\
                     & Q(x) \in \cone(\mathcal P_x) & \forall x\in X, & [\mathrm{dual~variable~}\Lambda(x)\in \cone(\mathcal P_x)^\star]
  \end{array}
\end{equation}
where the constraint of the above optimization problem enforces $Q$ to belong to cone $\mathcal K$; see Equation \eqref{eq:cone:sep}. {Please remark that this is not the only possible valid optimization characterization. Nevertheless, its primal characterization will be found to have a particularly simple dual.}
We remark that the special case of generic bandit problems is subsumed in the separable bandit problem and corresponds to the particular choice $\cone(\mathcal P_x)$ the positive orthant.
The associated Lagrangian function as introduced in Section \ref{sec:constraint-qualification} is
\[
  \textstyle L_\lambda(Q, \theta~;~ \gamma, \Lambda) = \iprod{\lambda}{Q} +  \sum_{x\in X}\gamma(x) (\theta-\sum_{x\in X, \, r\in \setr}Q(r, x) ) - \sum_{x\in X}\iprod{\Lambda(x)}{Q(x)}\,,
\]
which can be simplified by collecting terms to 
\[L_\lambda(Q, \theta~;~ \gamma, \Lambda) =\textstyle \sum_{x\in X,r\in \setr}(\lambda(r, x)-\gamma(x)-\Lambda(r, x))\cdot Q(r, x) + \theta \sum_{x\in X} \gamma(x)\,.\]
The associated dual function as introduced in Section \ref{sec:constraint-qualification} is hence
$$
g_{\lambda}(\gamma, \Lambda)\defn \min_{Q, \theta}L_\lambda(Q, \theta~;~ \gamma, \Lambda) = \sum_{x\in X, r\in \setr}\chi_{-\infty}(\lambda(r, x)-\gamma(x)-\Lambda(r, x)= 0)+\chi_{-\infty}(\sum_{x\in X}\gamma(x)= 0)\,.
$$
The sets $\cone(\mathcal P_x)$ have non-empty interior, as $\interior(\prod_{x\in X} \mathcal P_x) = \prod_{x\in X} \interior(\mathcal P_x)$ is non-empty.
Thus, the primal optimization problem \eqref{eq:dual-cone-separable-bandits} satisfies Slater's constraint qualification and hence strong duality  holds, per Proposition \ref{prop:strong-duality-general}. Hence, we have the equivalence 
\[
  0\leq \min_{Q\in \mathcal K}\iprod{\lambda}{Q} = \max_{\gamma, \,\Lambda} g_{\lambda}(\gamma, \Lambda) \iff
  \left\{
    \begin{array}{lr}
      \exists \gamma(x)\in \real & \forall x\in X, \\
      \lambda(x)- \textbf{1}\cdot \gamma(x) \in \cone(P_x)^\star  & \forall x\in X, \\
      \sum_{x\in X}\gamma(x) = 0. &
    \end{array}
  \right\}
  \iff \lambda \in \mathcal K^\star\,,
\]
where $\textbf{1}$ is the all ones vector of size of $|\setr|$.

\textit{Dual cone of Lipschitz bandits.} Using the definition of dual cone,  we again have  $\lambda\in \mathcal K^\star \iff 0\leq \min_{Q\in \mathcal K}\iprod{\lambda}{Q}$, where for Lipschitz bandits, the cone $\mathcal K$ is given in Equation \eqref{eq:cone:lip}. Here, we have that the previous primal minimization problem $\min_{Q\in \mathcal K}\iprod{\lambda}{Q}$ can be characterized as
\begin{equation}
  \label{eq:dual-cone-lipschitz-bandit}
  \begin{array}{rl@{\hspace{1em}}l@{\hspace{1em}}r}
    \min_{Q, \theta} & \iprod{\lambda}{Q} & & \\
    \st & Q(0, x)+Q(1, x)=\theta, & \forall x \in X & [\mathrm{dual~variable~}\gamma\in\real] \\
                     & Q(1, x) - Q(1, x') \leq \theta \cdot L \cdot d(x, x') & \forall x, x'\in X, &  [\mathrm{dual~variable~}\Lambda\in\real^{\abs{X}\times\abs{X}}_+] \\
                     & Q\geq 0,  & & [\mathrm{not~dualized}] \\
                     & \theta\in\real & & [\mathrm{not~dualized}]
  \end{array}
\end{equation}
where the constraint of the above optimization problem enforces $Q$ to belong to cone $\mathcal K$; see Equation \eqref{eq:cone:lip}. {Please remark that this is not the only possible valid characterization. Nevertheless, its primal characterization will be found to have a particularly simple dual.} The associated Lagrangian function as introduced in Section \ref{sec:constraint-qualification} is
\begin{align*}
  \textstyle L_\lambda(Q, \theta~;~ \gamma, \Lambda) &= \textstyle \iprod{\lambda}{Q} +  \sum_{x\in X}\gamma(x)(\theta-Q(0, x)-Q(1, x) ) \\
                                                     &\qquad \textstyle +\sum_{x, x'\in X}\Lambda(x, x')(Q(1, x) - Q(1, x')-\theta \cdot L \cdot d(x, x')),
\end{align*}
which can be simplified by collecting terms to 
\begin{align*}
  & L_\lambda(Q, \theta~;~ \gamma, \Lambda)\\
  &=\textstyle\sum_{x\in X}(\lambda(0, x)-\gamma(x))\cdot Q(0, x) + \textstyle \theta (\sum_{x\in X}\gamma(x)-\sum_{x, x'\in X} L \cdot d(x, x') \cdot \Lambda(x, x'))\\
  &\qquad \textstyle + \sum_{x\in X}\Big(\lambda(1,x)-\gamma(x)+\sum_{x'\in X}\Lambda(x, x')-\sum_{x'\in X}\Lambda(x', x)\Big)\cdot Q(1, x)\,.
\end{align*}
The associated dual function as introduced in Section \ref{sec:constraint-qualification} is hence
\begin{align*}
  g_{\lambda}(\gamma, \Lambda)&\defn \min_{Q\geq 0,\,\theta}L_\lambda(Q, \theta~;~ \gamma, \Lambda) \\
                              &\textstyle= \sum_{x\in X}\chi_{-\infty}(\lambda(0, x)-\gamma(x)\geq 0)\\
                              &\qquad +\sum_{x\in X}\chi_{-\infty}(\lambda(1,x)-\gamma(x)+\sum_{x'\in X}(\Lambda(x, x')-\Lambda(x', x))\geq 0)\\
                              &\textstyle \qquad +\chi_{-\infty}(\sum_{x\in X}\gamma(x)=\sum_{x, x'\in X} L \cdot d(x, x') \cdot \Lambda(x, x'))\,.
\end{align*}
As the feasible set in \eqref{eq:dual-cone-lipschitz-bandit} is a non-empty bounded polyhedron, strong linear duality  holds. Hence, we have the equivalence
\begin{align*}
  0\leq \min_{Q\in \mathcal K}\iprod{\lambda}{Q}= \max_{\gamma, \, \Lambda}&~g_{\lambda}(\gamma, \Lambda)  \iff \\[0.5em]
  & \left\{
    \begin{array}{l@{\,}r}
      \exists \Lambda(x, x')\in\real_+,~ \gamma(x) \in \real&  \forall x,x'\in X, \\
      \lambda(0, x)\geq \gamma(x) & \forall x\in X, \\
      \lambda(1,x)+\sum_{x'\in X}\Lambda(x, x')-\sum_{x'\in X}\Lambda(x', x)\geq \gamma(x) & \forall x\in X,  \\
      \sum_{x\in X}\gamma(x)= \sum_{x, x'\in X} L \cdot d(x, x') \cdot \Lambda(x, x') &
    \end{array}
  \right\}.
\end{align*}

{\color{black}

  \textit{Dual cone of linear bandits.} Using the definition of dual cone,  we again have  $\lambda\in \mathcal K^\star \iff 0\leq \min_{Q\in \mathcal K}\iprod{\lambda}{Q}$, where for linear bandits, the cone $\mathcal K$ is given in Equation \eqref{eq:cone:linear}. Here, we have that the previous primal minimization problem $\min_{Q\in \mathcal K}\iprod{\lambda}{Q}$ can be characterized as
  \begin{equation}
    \label{eq:dual-cone-linear-bandit}
    \begin{array}{rl@{\hspace{1em}}l@{\hspace{1em}}r}
      \min_{Q, \theta, \theta'} & \iprod{\lambda}{Q} & & \\
      \st & \sum_{r\in \setr}Q(r, x)=\theta' & \forall x \in X, & [\mathrm{dual~variable~}\gamma\in\real] \\
                       & \sum_{r\in\setr} r Q(r, x) = c_x^T \theta & \forall x \in X, &  [\mathrm{dual~variable~}\nu\in\real^{\abs{X}}] \\
                       & Q\geq 0,  & & [\mathrm{not~dualized}] \\
                       & \theta,\theta'\in\real & & [\mathrm{not~dualized}]
    \end{array}
  \end{equation}
  The associated Lagrangian function as introduced in Section \ref{sec:constraint-qualification} is
  \begin{align*}
    \textstyle L_\lambda(Q, \theta,\theta'~;~ \gamma, \nu) &= \textstyle \iprod{\lambda}{Q}\! + \! \sum_{x\in X}\gamma(x)(\theta'\!-\!\sum_{r\in\setr}Q(r, x))\textstyle\!+\!\sum_{x\in X}\nu(x)(c_x^T \theta\!-\!\sum_{r\in\setr} r Q(r, x)),
  \end{align*}
  which can be simplified by collecting terms to 
  \[
    L_\lambda(Q, \theta,\theta'~;~ \gamma, \nu)=\textstyle\sum_{r\in \setr, x\in X}(\lambda(r, x)\!-\!\gamma(x)\!-\!r \nu(x))\cdot Q(r, x)  \!+\! \theta' \sum_{x\in X}\gamma(x) \!+\! \sum_{x\in X}\nu(x)c_x^T \theta.
  \]
  The associated dual function as introduced in Section \ref{sec:constraint-qualification} is hence
  \begin{align*}
    g_{\lambda}(\gamma, \nu)&\defn \min_{Q\geq 0,\,\theta, \theta'}L_\lambda(Q, \theta,\theta'~;~ \gamma, \nu) \\
                                &\textstyle= \sum_{r\in \setr, x\in X}\chi_{-\infty}(\lambda(r, x)\!-\!\gamma(x)\!-\!r \nu(x)\geq 0)\\
                                &\qquad \textstyle+\chi_{-\infty}(\sum_{x\in X}\gamma(x) = 0) + \chi_{-\infty}(\sum_{x\in X}\nu(x)c_x =0).                         
  \end{align*}
  As the feasible set in \eqref{eq:dual-cone-linear-bandit} is a non-empty bounded polyhedron, strong linear duality  holds. Hence, we have the equivalence
  \begin{align*}
    0\leq \min_{Q\in \mathcal K}\iprod{\lambda}{Q}= \max_{\gamma, \, \nu}&~g_{\lambda}(\gamma, \nu)  \iff \\[0.5em]
                                                                             & \left\{
                                              \begin{array}{lr}
                           \exists \nu(x) \in \real, ~\exists \gamma(x)\in \real & \forall  x \in X,\\
                           \lambda(r, x) \geq \gamma(x) +r \nu(x) & \forall r\in \setr, x\in X,  \\
                           \sum_{x\in X}\gamma(x) = 0,~ \sum_{x\in X} c_x \nu(x) = 0
                         \end{array}                                                                                                                     \right\}.
  \end{align*}
}

{\color{black}
  \textit{Dual cone of Dispersion bandits.}
  Using the definition of dual cone,  we again have  $\lambda\in \mathcal K^\star \iff 0\leq \min_{Q\in \mathcal K}\iprod{\lambda}{Q}$, where for dispersion bandits, the cone $\mathcal K$ is given in Equation \eqref{eq:cone:dispersion}. Here, we have that the previous primal minimization problem $\min_{Q\in \mathcal K}\iprod{\lambda}{Q}$ can be characterized as
  \begin{equation}
    \label{eq:dual-cone-dispersion-bandit}
    \begin{array}{rl@{\hspace{1em}}l@{\hspace{1em}}r}
      \min_{Q, \theta, \theta'} & \iprod{\lambda}{Q} & & \\
      \st & \sum_{r\in \setr}Q(r, x)=\theta' & \forall x \in X, & [\mathrm{dual~variable~}\mu\in\real] \\
                                & \sum_{r\in\setr} r^2 Q(r, x) \leq \gamma(x) \sum_{r\in\setr} r Q(r, x) & \forall x \in X, &  [\mathrm{dual~variable~}\nu\in\real^{\abs{X}}_+] \\
                                & Q\geq 0,  & & [\mathrm{not~dualized}] \\
                                & \theta\in\real & & [\mathrm{not~dualized}]
    \end{array}
  \end{equation}
  The associated Lagrangian function as introduced in Section \ref{sec:constraint-qualification} is
  \begin{align*}
    \textstyle L_\lambda(Q, \theta~;~ \mu, \nu) &= \textstyle \iprod{\lambda}{Q}\! + \! \sum_{x\in X}\mu(x)(\theta\!-\!\sum_{r\in\setr}Q(r, x))\textstyle\!+\!\sum_{r\in\setr,x\in X}\nu(x)(r^2 Q(r, x)-\gamma(x) r Q(r, x)),
  \end{align*}
  which can be simplified by collecting terms to 
  \[
    L_\lambda(Q, \theta~;~ \mu, \nu)=\textstyle\sum_{r\in \setr, x\in X}(\lambda(r, x)\!-\!\mu(x)\!+\!(r^2 -r\gamma(x))\nu(x))\cdot Q(r, x)  \!+\! \theta \sum_{x\in X}\mu(x).
  \]
  The associated dual function as introduced in Section \ref{sec:constraint-qualification} is hence
  \begin{align*}
    g_{\lambda}(\mu, \nu)&\defn \min_{Q\geq 0,\,\theta, \theta'}L_\lambda(Q, \theta,~;~ \mu, \nu) \\
                                &\textstyle= \sum_{r\in \setr, x\in X}\chi_{-\infty}(\lambda(r, x)\!-\!\mu(x)\!+\!(r^2 -r\gamma(x))\nu(x)\geq 0)\textstyle+\chi_{-\infty}(\sum_{x\in X}\mu(x) = 0).                         
  \end{align*}
  As the feasible set in \eqref{eq:dual-cone-dispersion-bandit} is a non-empty bounded polyhedron, strong linear duality  holds. Hence, we have the equivalence
  \begin{align*}
    0\leq \min_{Q\in \mathcal K}\iprod{\lambda}{Q}= \max_{\mu, \, \nu}&~g_{\lambda}(\mu, \nu)  \iff \\[0.5em]
                                                                             & \left\{
                                                                               \begin{array}{lr}
                       \exists \nu(x) \in \real_+, ~\exists \mu(x)\in \real & \forall  x \in X,\\
                       \lambda(r, x) \geq \mu(x) -(r^2-\gamma(x)r) \nu(x) & \forall r\in \setr, x\in X,  \\
                       \sum_{x\in X}\mu(x) = 0
                                                                               \end{array}
    \right\}.
  \end{align*}
}

\section{Proof of Theorem \ref{thm:regret}: Regret Bound of DUSA}\label{sec:proof}

We first present our concentrations bound and explain how the regret of DUSA can be decomposed. 

\subsection{Concentration Bounds and Regret Decomposition}\label{sec:proof:part1}
We start with presenting our supporting concentration bounds. These bounds will be used multiple times throughout the proof and their proofs are presented in Appendix \ref{sec:proof:concentration}.
\begin{lemma}[Concentration Bound]
  \label{lemma:stopping-time-inequality}
  {\color{black}For any arm $x\in X$, let $\stop(s,x)\in [s, \dots, T+1]$ be any stopping time such that either $N_{\stop(s, x)}(x) \geq \phi(s)$ or $\stop(s, x)=T+1$. 
  Then, for any $r\in \setr$, we have 
  \[
    \Prob{\abs{  P_{\stop(s, x)}(r, x) - \optp(r, x)}\geq \kappa, ~\stop(s, x)\leq T }~\leq~ 2 \exp(-2\phi(s) \kappa^2)\,,
  \]
 where for any $\tau \in [T]$, $N_\tau(x)$ is the number of rounds within the first $\tau$ rounds that arm $x$ is played.}
\end{lemma}

The proof of Lemma \ref{lemma:stopping-time-inequality} is inspired by the proof of Lemma 4.3 in \cite{combes2014unimodal}.

\begin{lemma}[Concentration Bound for Information Distance]
  \label{lemma:concentration}
    Let $\delta\geq \abs{X}(\abs{\setr}-1)+1$. Then, for any $t\ge 1$, 
  \[
    \Prob{\sum_{x\in X} N_t(x) I(P_t(x), P(x))\geq \delta} \leq  e\cdot \left(\frac{\delta \ceil{\log(t) \delta+1} 2 e}{\abs{X}(\abs{\setr}-1)}\right)^{\abs{X}(\abs{\setr}-1)}\cdot\exp(-\delta).
  \]
\end{lemma}

The proof of concentration bound in Lemma \ref{lemma:concentration}, which  is one of our contributions, is quite involved and requires a careful decomposition of the reward distributions for which we construct a martingale sequence. Our concentration bound  
bears resemblance to the bound in \cite{pmlr-v35-magureanu14}, where they bound  $\Prob{\sum_{x\in X} N_t(x) I(P_t(x), P(x))\geq \delta}$ for Binomial  random variables. We generalize  their bounds to arbitrary discrete distributions, which may be of independent interest for non-parametric multi-armed bandit problems.

\textbf{Regret decomposition.} Having presented the concentration bounds, we now focus on upper bounding the regret of DUSA. The regret  can be broken down into three parts: Regret during (i) the initialization phase, (ii) the exploitation phase and (iii) the exploration phase. That is, for any $T> \abs{X}$, we have 
\begin{align*}
  \reg(T, \optp)\! = & \E{}{\textstyle\sum_{1\leq t\leq \abs{X}} \Delta(x_t, \optp)} \!+\! \E{}{\textstyle\sum_{\abs{X}< t\leq T} \mb 1\{\mathcal E_t \}\Delta(x_t, \optp)} \!+\! \E{}{\textstyle\sum_{\abs{X}< t\leq T} \mb 1\{\mathcal X_t \}\Delta(x_t, \optp)}\,,
\end{align*}
where $\Delta(x, \optp)$ is the suboptimality gap of arm $x$ under the actual reward distribution $\optp$. 
Here, we define  events $\mathcal{E}_t$ and $\mathcal X_t$ as  the event that DUSA is the exploitation and exploration phase, respectively, in a given round $t$.
Evidently, we have that $\mathcal{E}_t\bigcup \mathcal X_t=\{t>\abs{X}\}$. In the initialization phase, we select each arm $x\in X$ once. The regret caused in this phase is evidently finite and equal to
$
  \E{}{\textstyle\sum_{1\leq t\leq \abs{X}} \Delta(x_t, \optp)} = C_{\rm{initialize}}\defn\textstyle\sum_{x\in X} \Delta(x, \optp).
$
Note that this regret is upper bounded by $|X|$ as the regret in each round is bounded by one. In the following sections, each of the other two terms will be bounded.

\subsection{Exploitation Phase} \label{sec:proof:exploit}

To characterize the regret during the exploitation phase, we start with defining the following bad event
\(\badexploit_t(x, \kappa) = \{\norm{P_t(x)-\optp(x)}_\infty > \kappa \}\)
for some positive threshold $\kappa > 0$ and arm $x\in X$.  The  bad event occurs when the empirical reward distribution $P_t(x)$ and the true reward distribution $\optp(x)$ of the considered arm $x$ deviate significantly. In other words, the reward distribution of arm $x$ is poorly estimated. 
The following lemma characterizes the regret accumulated during the exploitation phase due to a poorly estimated empirically optimal arm; i.e., when\ the events $\mathcal{E}_t$ and $\bigcup_{x\in x^\star(P_t)}\badexploit_t(x,\kappa)$ occur simultaneously. In the following lemma, 
we show that the regret  accumulated during the first $T$ rounds under these events  is negligible for any arbitrary small $\kappa>0$. As a part of the proof, we invoke the concentration bound in Lemma \ref{lemma:stopping-time-inequality}. 

\begin{lemma}[Regret under the Bad Events in the Exploitation  Phase]\label{lem:explore_bad}
  The regret accumulated during the exploitation phase when the bad event $\bigcup_{x\in x^\star(P_t)}\badexploit_t(x,\kappa)$ occurs remains finite; that is, 
\begin{equation}
  \label{eq:regret:bad}
  \E{}{\sum_{\abs{X}< t\leq T} \mb 1\{\mathcal E_t, ~\bigcup_{x\in x^\star(P_t)}\badexploit_t(x,\kappa)\}\cdot\Delta(x_t, \optp)}
  ~\leq~  8 \abs{X}  + \abs{X} \abs{\setr} \sum_{1\leq s\leq T} 2 \exp(-s \kappa^2/(2\abs{X})) 
\end{equation}
 for any $\kappa>0$.
\end{lemma}

Next, we assume that we are in the exploitation phase and our estimate of the reward distributions of the empirically optimal arms are accurate, i.e., $\bigcup_{x\in x^\star(P_t)}\badexploit_t(x,\kappa)$ does not occur.
Recall that we assumed that $\optp$ has a unique optimal arm. By Lemma \ref{lemm:stability-optimal-arm}, presented at the end of this section, 
 this implies that we can take $\kappa$ small enough such that either $x^\star(\optp)=x^\star(P_t)$ or $x^\star(\optp)\not\in x^\star(P_t)$.  
The regret under the former event is zero. Thus, we only bound the regret under the latter event: 
\begin{align*}
  & \E{}{\textstyle\sum_{\abs{X}< t\leq T} \mb 1\{\mathcal E_t, ~\bigcap_{x\in x^\star(P_t)}\lnot\badexploit_t(x,\kappa)\}\cdot \Delta(x, \optp)} \\
  = & \E{}{\textstyle\sum_{\abs{X}< t\leq T} \mb 1\{\mathcal E_t, ~x^\star(\optp)\not\in x^\star(P_t) , ~\bigcap_{x\in x^\star(P_t)}\lnot\badexploit_t(x,\kappa)\}\cdot \Delta(x, \optp)} \\
  \leq & \sum_{\abs{X}<t\leq T} \Prob{ \mathcal E_t, ~x^\star(\optp)\not\in x^\star(P_t), ~\norm{P_t(x)-P(x)}_\infty\leq \kappa ~~\forall x\in x^\star(P_t)}\,,
\end{align*}
where the inequality holds because $\Delta(x, P) \le 1 $. Here, for any event $E$, $\lnot E$ denotes its complement.  

To further bound the regret, we take advantage of our sufficient information test. If the event $\mathcal E_t$ occurs, we must be  in the exploitation  phase and thus the sufficient information test \eqref{eq:test} should have been passed by every empirically deceitful arm in $\tilde X_d(P_t)$.
By the weak duality inequality in Lemma \ref{thm:dual-feasibility}, this implies that  $1+\epsilon\leq \overline{\dual}_t(x') \leq  \dis(\tfrac{N_t}{\log(t)},  x', P_t)$ for any empirically deceitful arm $x'$, and hence by definition of the distance function $\dis$ in \eqref{eq:inner-optimization},  the event 
\begin{align}
  \label{eq:a_1}
  \textstyle\sum_{x\in \tilde X(P_{t})} N_{t}(x) ~I( P_{t}(x), Q(x)) \geq (1+\epsilon)\log(t) \quad \forall Q \in \mathcal \prob(x', P_{t}), ~\forall x'\in \xd( P_{t}) 
\end{align}
must have occurred.
By \eqref{eq:a_1},  if we had $\optp\in \prob(x^\star(\optp), P_{t})$ and $x^\star(\optp)\in \xd(P_t)$, then 
\begin{align}
  \label{eq:a_2}
  \textstyle\sum_{x\in X} N_{t}(x) ~I( P_{t}(x), \optp(x)) \geq (1+\epsilon)\log(t)\,.
\end{align}
Under \eqref{eq:a_2},  Lemma \ref{lemma:concentration} implies that the total regret caused by such events is finite uniformly in $T$.
Unfortunately, the actual reward distribution $\optp$ may not be deceitful, i.e., $P$ may not belong to $\prob(x^\star(\optp), P_{t})$.
Nonetheless, we prove 
that the event defined in \eqref{eq:a_2}  does occur 
as long as $\kappa$ is chosen sufficiently small.

\begin{lemma}\label{lem:event_a_t}
  For small enough $\kappa>0$, if the bad event  i.e., $\bigcup_{x\in x^\star(P_t)}\badexploit_t(x,\kappa)$, does not occur, then  the event defined in Equation \eqref{eq:a_2}, must occur.
\end{lemma}

The total regret $$\sum_{1\leq t\leq T} \Prob{ \mathcal E_t, ~x^\star(P)\not\in x^\star(P_t), ~\norm{P_{t}(x)- \optp(x)}_\infty\leq \kappa ~~\forall x\in x^\star(P_t) }$$ accumulated here is upper bounded by
\(  
  \sum_{1\leq t\leq T} \Prob{ \textstyle\sum_{x\in X} N_{t}(x) ~I( P_{t}(x), \optp(x)) \geq (1+\epsilon)\log(t) }\,.
\)
Let  $t> t_0$ with  $ \log(t_0) (1+\epsilon) > \abs{X}(\abs{\setr}-1)+1$. 
Then, by Lemma \ref{lemma:concentration}, we have 
\begin{align*}
  & \Prob{\textstyle\sum_{x\in X} N_{t}(x) ~I( P_{t}(x), \optp(x)) \geq (1+\epsilon)\log(t)} \\
  \leq & {\color{black} \exp(-(1+\epsilon)\log(t))\cdot e \cdot \left(\frac{(1+\epsilon) \log(t) \ceil{\log(t) (1+\epsilon) \log(t)+1} 2 e}{\abs{X}(\abs{\setr}-1)}\right)^{\abs{X}(\abs{\setr}-1)} }\\
  =  & {\color{black}\frac{e}{t^{1+\epsilon}} \cdot \left(\frac{(1+\epsilon) \log(t) \ceil{\log(t) (1+\epsilon) \log(t)+1} 2 e}{\abs{X}(\abs{\setr}-1)}\right)^{\abs{X}(\abs{\setr}-1)}.}
\end{align*}
{\color{black}Hence, the regret in the event of interest, i.e., when the bad event  $\bigcup_{x\in x^\star(P_t)}\badexploit_t(x,\kappa)$ does not happen, is  bounded as 
\begin{align*}
  & \E{}{\textstyle\sum_{\abs{X}< t\leq T} \mb 1\{ \mathcal E_t, ~\bigcap_{x\in x^\star(P_t)}\lnot\badexploit_t(x,\kappa)\}\cdot\Delta(x_t, \optp)} \\[0.5em]
  \leq & t_0+ \sum_{t_0 \le t \le T}\frac{e}{t^{1+\epsilon}}\cdot \left(\frac{(1+\epsilon) \log(t) \ceil{\log(t) (1+\epsilon) \log(t)+1} 2 e}{\abs{X}(\abs{\setr}-1)}\right)^{\abs{X}(\abs{\setr}-1)}.
\end{align*}
Thus, the total regret in the exploitation phase is upper bounded by 
\begin{equation}
  \label{eq:c_exploit}
  \begin{aligned} 
  C_{\text{exploit}} :=  &  8\abs{X} + \abs{X} \abs{\setr} \sum_{1\leq s\leq T} 2 \exp(-s \kappa^2/(2\abs{X}))  + t_0+ \\
                         & \qquad \sum_{t_0 \le t \le T}\frac{e}{t^{1+\epsilon}}\cdot \left(\frac{(1+\epsilon) \log(t) \ceil{\log(t) (1+\epsilon) \log(t)+1} 2 e}{\abs{X}(\abs{\setr}-1)}\right)^{\abs{X}(\abs{\setr}-1)}
\end{aligned}
\end{equation}
which remains finite for any $T$. }

\begin{lemma}[Stability of the Optimal Arm and Continuity of the Maximal Reward]
  \label{lemm:stability-optimal-arm}
  Let $\mathcal P_{u}(x^\star_0)$ be the set of all reward distributions $P\in \mathcal P_\Omega$ with a unique optimal arm $x^\star_0$. Then, (i) the set $\mathcal P_{u}(x^\star_0)$ is an open set, and (ii) the maximum reward
  \(
    \rew_{\max}(x', Q)
  \)
  of any arm $x'\in X$ is a continuous function at $P\in \interior(\mc P)$. 
\end{lemma}

\subsection{Exploration Phase} \label{sec:proof:explore}

To characterize the regret during the exploration phase, we first bound the probability that our estimate of reward distributions  in two consecutive exploration rounds is bad and we then bound the regret when our estimate of reward distributions in two consecutive exploration rounds is indeed good.
Let $\kappa>0$ be defined as in the previous section.
We say our estimate at the end of round $t$ is bad, i.e., $\badexplore_t$ occurs, when $\norm{  P_{t}-P}_\infty> \kappa$. Here, $\norm{  P_{t}-\optp}_\infty = \max_{r, x} |P_{t}(r, x) -\optp(r, x)|$. We a slight abuse of notation, we define the stopping time $\tauexplore(s)$ as the first time in the first $T$ rounds in which we are in the exploration phase, i.e., $\mathcal X_t$ occurs, and $s_t=s$. If no such round exists, we set $\tauexplore(s)=T+1$.

\textbf{Regret during bad exploration rounds.} 
We now bound the regret caused by events in which either now or in the last time we were in the exploration phase our estimate of the reward distributions is bad.

\begin{lemma}\label{lem:lowerbound_N}
  {\color{black}If $\mathcal X_t\bigcap \{s_t> \tfrac{2}{\epsilon}\}$ occurs, then we must have 
$
  \min_{x\in X} ~N_t(x) \geq  \epsilon \lceil \frac{{s_t}/4}{1+\log(1+s_t)}\rceil
$.}
\end{lemma}
When 
event $\mathcal X_t\bigcap \{s_t> \tfrac{2}{(\epsilon)}\}\bigcap \badexplore_t$ happens, (i)  we are in the exploration phase, (ii) $\norm{  P_{t}-\optp}_\infty\geq \kappa$, and (iii) {\color{black}$\min_{x\in X} N_t(x)\ge \epsilon \lceil \tfrac{{s_t}}{(4(1+\log(1+s_t)))}\rceil$}. Our regret under the aforementioned events can be written as:
\begin{equation}
  \label{eq:constant_c2}
  \begin{aligned}
  & \E{}{\mb 1\{\badexplore_{\tauexplore(1)},~\tau(1)\leq T \} \cdot \Delta(x_{\tauexplore(1)}, \optp) + \textstyle\sum_{2\leq s\leq T} \mb 1\{(\badexplore_{\tau(s-1)} \bigcup \badexplore_{\tauexplore(s)}),~s\leq \tau(s)\leq T \}\cdot \Delta(x_{\tauexplore(s)}, \optp)}\\
  &\leq 1+\sum_{2\leq s\leq T} \Prob{(\badexplore_{\tau(s-1)} \bigcup \badexplore_{\tauexplore(s)}),~s\leq \tau(s)\leq T} \leq~  1+2 \sum_{1\leq s\leq T} \Prob{\badexplore_{\tauexplore(s)}, ~s\leq \tau(s)\leq T}\\
  ~&\le 1+ 2\big(\sum_{1\leq s< T} \Prob{ \mb 1\{s\leq \tfrac{2}{\epsilon} \} } + \sum_{1\leq s\leq T} \Prob{ \mb 1\{s> \tfrac{2}{\epsilon} \},~\norm{P_{\tau(s)}-\optp}_{\infty} \ge \kappa,~s\leq \tau(s)\leq T } \big)  \\
  ~&{\color{black}\le 1+ \frac{4}{\epsilon}+
  4 \abs{X}\abs{\setr}\sum_{1\leq s\leq T}   \exp(- \tfrac{s \epsilon \kappa^2}{(2(1+\log(1+s)))})
 :=C_2}\,,
\end{aligned}
\end{equation}
{\color{black}
where the third inequality follows from Lemmas \ref{lemma:stopping-time-inequality} and \ref{lem:lowerbound_N}. 
Finally, we remark that the sum in the final inequality satisfies
\begin{align*}
  & \sum_{1\leq s\leq T} \exp(- \tfrac{s \epsilon \kappa^2}{(2(1+\log(1+s)))} \leq \int_0^\infty \exp(- \tfrac{s \epsilon \kappa^2}{(2(1+\log(1+s)))} \d s\\
  \leq & \int_0^{a(\delta)} \exp(- \tfrac{s \epsilon \kappa^2}{(2(1+\log(1+s)))} \d s +  \int_{a(\delta)}^\infty \exp(- s^{1-\delta} \epsilon \kappa^2)\d s \leq  a(\delta) +  \int_0^\infty \exp(- s^{1-\delta} \epsilon \kappa^2)\d s\\
  \leq & \, a(\delta) + \epsilon^{\tfrac{-1}{(1-\delta)}} \kappa^{\tfrac{-2}{(1-\delta)}} \Gamma\left(\frac{2-\delta}{1-\delta}\right)
\end{align*}
for any $\delta>0$ and $a(\delta)$ a finite number such that $s\geq a(\delta)$ implies that $2+2\log(1+s)\leq s^\delta$. The first and third inequalities follow then from the fact that integrant $\exp(- \tfrac{s \epsilon \kappa^2}{(2(1+\log(1+s)))}$ is non-increasing in $s$ and bounded above by one. Hence, $C_2$ remains bounded for all $T$ and scales in $\epsilon$ no faster than $\epsilon^{\tfrac{-1}{(1-\delta)}}$ for all $\delta>0$. That is, $C_2$ scales with $\epsilon$ essentially as $1/\epsilon$.
}

\textbf{Upper bound on the number of good exploration rounds.} In a good exploration round, clearly,  we are in the exploration phase and the event $\badexplore_t$  does not occur now nor did occur the last time we were in the exploration phase. Recall that 
event $\badexplore_t$ does not happen when  at the end of round $t$,  $\norm{  P_{t}-P}_\infty\le \kappa$. 
From Lemma \ref{lemm:stability-optimal-arm} and Assumption \ref{assumption},  we can consider a sufficiently small $\kappa$ such that when the event $\badexplore_t$ does not occur, we have  $x^\star(P_t)=x^\star(P)$, $\xd(P_t)=\xd(P)$ and $\xn(P_t)=\xn(P)$.  Thus, under those  aforementioned  fortunate  circumstances,  we correctly identified each arm as either optimal, deceitful or undeceitful both now and in the previous exploration round. We will denote this fortunate event with $\mathcal G_t=\lnot \badexplore_t$.

We define $r(x)=\max\tset{\abs{X}< t\leq T}{x_t=x, ~\mathcal G_t}$ as the most recent exploration round before time $T$ in which the good event $\mathcal G_t$ occurred and arm $x$ is played. We set $r(x)=0$ if no such event occurred. Consider the number $\sum_{t=1}^T\mb 1\{x_t=x,~\mathcal X_t, ~\mathcal{G}_t \}$ of good exploration rounds in which some arm $x\in X$ is played. Evidently, as $r(x)$ is the most recent time before time $T$ in which the good event $\mathcal G_t$ occurred and arm $x$ is played, we must have that $\sum_{t=1}^T\mb 1\{x_t=x,~\mathcal X_t, ~\mathcal{G}_t \}=\sum_{t=1}^{r(x)}\mb 1\{x_t=x, ~\mathcal X_t, ~\mathcal{G}_t \}\leq \sum_{t=1}^{r(x)}\mb 1\{x_t=x\} = N_{r(x)}(x)$. We bound this further using the following lemma.
 
 \begin{lemma}\label{lem:expolre:bad:arm}
 During an exploration round $t$,  we have 
 $ N_t(x_t) \leq (1+\epsilon)\norm{\eta_t}_\infty \log(T)$.
\end{lemma}
 
 By Lemma \ref{lem:expolre:bad:arm}, we  have the following bound on the number of good exploration rounds:
\begin{align*}
  \sum_{t=1}^{T} \mb 1\{\mathcal X_t,~\mathcal{G}_t \} &  = \sum_{t=1}^{T} \sum_{x\in X}\mb 1\{x_t=x, ~\mathcal X_t,~\mathcal{G}_t \} =  \sum_{x\in X} \sum_{t=1}^{T} \mb 1\{x_t=x,~\mathcal X_t,~\mathcal{G}_t \} \\
  & \leq \sum_{x\in X} N_{r(x)}(x) \leq \sum_{x\in X}(1+\epsilon)\norm{\eta_{r(x)}}_\infty \log(T).
\end{align*}
Recall that in rounds $r(x)$ we have that the estimate $P_{r(x)}$ and the estimate of the exploration round preceding $r(x)$ are both close to $P$. Because of the stability result in Proposition \ref{prop:continuous-shallow-selection}, we can take $0<\kappa$ sufficiently small so that the target rates $\eta_{r(x)}$ computed by the shallow update algorithm in such good rounds satisfies
\(
\norm{\eta_{r(x)}-\eta_\epsilon'(P)}_\infty \leq \epsilon.
\)
That is, these rounds come with sufficiently accurate estimates to guarantee that the exploration rates set by the shallow update algorithm are very close to the desired exploration rate $\eta_\epsilon'(P)$. Consequently, we obtain the bound
\begin{align}
  \label{eq:bound-good-rounds}
  \sum_{t=1}^{T} \mb 1\{\mathcal X_t,~\mathcal G_t \}\leq (1+\epsilon)\abs{X}(\norm{\eta'_\epsilon(\optp)}_\infty+\epsilon) \log(T).
\end{align}
Note that the inequality holds even when $r(x) =0$. When $r(x)=0$, we have $\sum_{t=1}^T\mb 1\{x_t=x,~\mathcal X_t, ~\mathcal{G}_t \}=\sum_{t=1}^{r(x)}\mb 1\{x_t=x, ~\mathcal X_t, ~\mathcal{G}_t \}=0$. 

{\color{black}
  \begin{remark}[Nondeceitful Bandits]
    \label{rem:proof:nondeceitful-bandits}

  When the bandit $P\in \mc P'$ happens to be nondeceitful, i.e., $\tilde X_d(P)=\emptyset$, the regret lower bound function $C(P)=0$ vanishes. We will indicate that for such nondeceitful bandits the regret of DUSA remains bounded for all $T$. We have  
  \begin{align*}
     & \reg(T, \optp)\\
    = & \E{}{\textstyle\sum_{1\leq t\leq \abs{X}} \Delta(x_t, \optp)} \!+\! \E{}{\textstyle\sum_{\abs{X}< t\leq T} \mb 1\{\mathcal E_t \}\Delta(x_t, \optp)} \!+\! \E{}{\textstyle\sum_{\abs{X}< t\leq T} \mb 1\{\mathcal X_t \}\Delta(x_t, \optp)} \\
    \leq & C_{\rm{initialize}} + C_{\rm{exploit}} + \E{}{\textstyle\sum_{\abs{X}< t\leq T} \mb 1\{\mathcal X_t,~ \badexplore_t \}\Delta(x_t, \optp)} + \E{}{\textstyle\sum_{\abs{X}< t\leq T} \mb 1\{\mathcal X_t,~ \mc G_t \}\Delta(x_t, \optp)}\\
    \leq & C_{\rm{initialize}} + C_{\rm{exploit}} + C_2 + \E{}{\textstyle\sum_{\abs{X}< t\leq T} \mb 1\{\mathcal X_t,~ \mc G_t \}\Delta(x_t, \optp)}\\
    \leq & C_{\rm{initialize}} + C_{\rm{exploit}} + C_2.
\end{align*}
The first inequality follows from Inequality \eqref{eq:c_exploit}. The second inequality follows from the fact that
\begin{align*}
  & \E{}{\textstyle\sum_{\abs{X}< t\leq T} \mb 1\{\mathcal X_t,~\badexplore_t \}\Delta(x_t, \optp)}\\
  = & \E{}{\textstyle\sum_{1\leq s\leq T} \mb 1\{\badexplore_{\tau(s)}, ~s\leq \tau(s)\leq T \} \Delta(x_{\tau(s)}, \optp) }\\
  \leq & \E{}{\mb 1\{\badexplore_{\tauexplore(1)},~\tau(1)\leq T \} \cdot \Delta(x_{\tauexplore(1)}, \optp) + \textstyle\sum_{2\leq s\leq T} \mb 1\{(\badexplore_{\tau(s-1)} \bigcup \badexplore_{\tauexplore(s)}),~s\leq \tau(s)\leq T \}\cdot \Delta(x_{\tauexplore(s)}, \optp)}\\
         \leq & C_2
\end{align*}
following Inequality \eqref{eq:constant_c2}. Finally, for $\kappa>0$ chosen sufficiently small we have as discussed before that $\mc G_t\implies \xd(P_t)=\xd(P)=\emptyset$. However, in order to be in the exploration phase there needs to be at least one deceitful arm $x\in \xd(P_t)$ for which the sufficient information test ($\overline{\dual}_t(x)\ge 1+\epsilon$) failed. As this is clearly not the case we must have that $\mc G_t\implies \xd(P_t)=\xd(P)=\emptyset \implies \mathcal E_t$ and consequently $\E{}{\textstyle\sum_{\abs{X}< t\leq T} \mb 1\{\mathcal X_t,~ \mc G_t \}\Delta(x_t, \optp)}=0$,  establishing the claim.
\end{remark}
}

{\textbf{Regret under good exploration rounds.} We are now left with the regret caused by rounds in which we are in the exploration phase and in the reward model $\optp$ was estimated accurately both in the current round and the previous time we were not in the exploitation phase.
That is, we need to bound the regret caused when the good events $\mathcal G_t$ occur. Let
\[
  \begin{array}{r@{\,}l}
    W^1_T &=  \sum_{t=1}^T \Delta(x, P) \mb 1\{x_t = \underline x_t,~\mathcal X_t,~\mathcal G_t\}, \\[0.5em]
    W^2_T &= \sum_{t=1}^T \Delta(x, P) \mb 1\{x_t = \bar x_t, ~\mathcal X_t,~\mathcal G_t\}, \\[0.5em]
    W^3_T & = \sum_{t=1}^T \Delta(x, P) \mb 1\{x_t = x^\star_t, ~\mathcal X_t,~\mathcal G_t\}=0
  \end{array}
\]
{where the regret caused by pulling the empirical optimal arm $x_t^\star$ is zero as by construction $x^\star(P_t)=x^\star(P)$ when event $\mathcal G_t$ occurs.}

We first upper bound the expected regret $\E{}{W^1_T}$ when the good event $\mathcal G_t$ happens and arm $\underline  x_{t}$ is played in the exploration phase. Consider the number of times $\sum_{t=1}^T\mb 1\{x_t = \underline x_t=x, ~\mathcal X_t, ~\mathcal G_t\}$ we play arm $x$ in such events. {\color{black}Recall that when $x_t = \underline x_t$ and $\mathcal G_t$ occurs, we must have that $N_t(\underline x_t)\leq \epsilon \tfrac{s_t}{(1+\log(1+s_t))}$.
Let $\underline{r}(x)$ be the most recent time before round $T$ that the event $x_t =\underline{x}_t=x, ~\mathcal G_t$ occurred. If no such time exists,  we set $\underline{r}(x)=0$. Evidently, we have that
\begin{align*}
  \sum_{t=1}^{T}\mb 1\{x_t = \underline x_t=x, ~\mathcal X_t, ~\mathcal G_t\}
  &= \sum_{t=1}^{\underline{r}(x)}\mb 1\{x_t = \underline x_t=x, ~\mathcal X_t, ~\mathcal G_t\}\\[0.5em]
  &\leq \sum_{t=1}^{\underline{r}(x)}\mb 1\{x_t = \underline x_t=x\} \leq N_{\underline{r}(x)}(x)\\[0.5em]
  &\leq \epsilon \tfrac{s_{\underline{r}(x)}}{(1+\log(1+s_{\underline{r}(x)}))} \epsilon \leq \epsilon \tfrac{s_T}{(1+\log(1+s_T))}.
\end{align*}
We then have
\begin{align*}
  \begin{array}{rl}
    W^1_T & = \sum_{x\in X}\sum_{t=1}^T \Delta(x, P) \mb 1\{x_t = \underline x_t=x, ~\mathcal X_t, ~\mathcal G_t\}  \leq \sum_{x\in X}\sum_{t=1}^T \mb 1\{x_t = \underline x_t=x, ~\mathcal X_t, ~\mathcal G_t\}\\[0.5em]
            & \leq  \abs{X} \epsilon \tfrac{s_T}{(1+\log(1+s_T))}.
  \end{array}
\end{align*}
Jensen's inequality establishes that $\E{}{W_T^1}  \leq  \epsilon  \abs{X} \tfrac{\E{}{s_T}}{(1+\log(1+\E{}{s_T}))}$. Furthermore, we have that
\begin{align*}
  \begin{array}{rl}
    \E{}{s_T} & = \E{}{\sum_{\tau=1}^{T} \mb 1\{\mathcal X_{\tau},~\mathcal G_\tau\} + \sum_{\tau=1}^{T} \mb 1\{\mathcal X_{\tau},~\lnot\mathcal G_\tau \}} \\[0.5em]
           & \leq (1+\epsilon)\abs{X}(\norm{\eta'_\epsilon(\optp)}_\infty+\epsilon) \log(T) + C_2\,,
  \end{array}
\end{align*}
where the inequality follows from Equation \eqref{eq:bound-good-rounds} and the previous result that the expected number of exploration bounds associated with bad estimates, i.e., $\E{}{\sum_{\tau=1}^{T} \mb 1\{\mathcal X_{\tau},~\lnot\mathcal G_\tau \}}$, is bounded in expectation by the constant $C_2$ defined previously. Using monotonicity of the function $s\mapsto \tfrac{s}{(1+\log(s+1))}$ we get that
\[
  \E{}{W_T^1} \leq C_3\defn\abs{X} \epsilon \frac{(1+\epsilon)\abs{X}(\norm{\eta'_\epsilon(\optp)}_\infty+\epsilon) \log(T) + C_2}{1+\log(1+(1+\epsilon)\abs{X}(\norm{\eta'_\epsilon(\optp)}_\infty+\epsilon) \log(T) + C_2)}.
\]
Furthermore, the last inequality establishes that the regret caused by this event is negligible, i.e.,  $\lim_{T\to\infty}\E{}{W_T^1}/\log(T)=0$.}

We now upper bound the expected regret $\E{}{W^2_T}$. Consider the number of times $\sum_{t=1}^T\mb 1\{x_t = \bar x_t=x, ~\mathcal X_t, ~\mathcal G_t\}$ we play arm $x$ in such events. As we argued before, since we are  in the exploration phase and $x_t=\bar x_t$, we must have that $N_t(x_t)  \leq  (1+\epsilon)\eta_t(x_t) \log(T)$ where $\eta_{t}=\textsc{SU}(  P_{t}, \eta'_t;\mu_{t}, \epsilon)$. (Recall that when we are in the exploration phase, there exists an empirically suboptimal arm $x'\in \tilde X(P_t)$ such that $\tfrac{N_t(x')}{\eta_t(x')} \leq (1+\epsilon)\log(t)$. This implies that $\tfrac{N_t(x_t)}{\eta_t(x_t)} \leq (1+\epsilon)\log(t)$ when $x_t = \bar x_t =\arg\min_{x\in \tilde X(P_t)} \tfrac{N_{t}(x)}{ \eta_{t}(x)}$). Let $\bar{r}(x)$ be the most recent time before round $T$ that the event $x_t =\bar{x}_t=x, ~\mathcal G_t$ occurred. If no such time exists than we set $\bar{r}(x)=0$.
Evidently, we have that \begin{align*}\sum_{t=1}^{T}\mb 1\{x_t = \bar x_t=x, ~\mathcal X_t, ~\mathcal G_t\} &\leq \sum_{t=1}^{\bar{r}(x)}\mb 1\{x_t = \bar x_t=x, ~\mathcal X_t, ~\mathcal G_t\}\\
&\leq \sum_{t=1}^{\bar{r}(x)}\mb 1\{x_t = \underline x_t=x\} = N_{\bar{r}(x)}(x)\leq (1+\epsilon)\eta_{\bar{r}(x)}(x) \log(T).\end{align*}
Thus,
\[
    \textstyle \E{}{W_T^2} = \E{}{\sum_{x\in X}\Delta(x, P) \sum_{t=1}^T \mb 1\{x_t = \bar x_t=x, ~\mathcal G_t\}}
           \leq \E{}{(1+\epsilon) \sum_{x\in X}\Delta(x, P) \eta_{\bar{r}(x)}(x) \log(T)}.
\]
Because of the stability result in Proposition \ref{prop:continuous-shallow-selection}, we can take $0<\kappa$ sufficiently small so that the target rates $\eta_{r(x)}$ computed by the shallow update algorithm in such good rounds satisfies
\(
\norm{\eta_{r(x)}-\eta_\epsilon'(P)}_\infty \leq \epsilon.
\)
Hence,
\[
    \E{}{W_T^2} \leq (1+\epsilon) \sum_{x\in X}\Delta(x, P) (\eta'_\epsilon(x, P)+\epsilon) \log(T) \leq (1+\epsilon) (C(P)+\epsilon+\epsilon\sum_{x\in X}\Delta(x, P)) \log(T):=C_4\,,
\]
where we exploit the fact that $\eta_\epsilon'(P)$ is an $\epsilon$-suboptimal exploration rate. {\color{black}
  Overall, the expected regret in the exploration phase is hence bounded as
\(
  C_{\text{explore}} \leq C_2+C_3+C_4.
\)
Recall that\footnote{{\color{black}Given the fact that $C_2$ scales essentially as $1/\epsilon$, for our asymptotic regret bound to hold  $\epsilon$ can be any number in the order of $\Omega(1/\log(T))$. That being said, we observe empirically that the regret of DUSA for finite $T$ does not deteriorate  even when choosing $\epsilon$ very small. The same observation was made by \citet{combes2017minimal} who, in fact, choose $\epsilon=0$ in their numerical experiments.} } 
\begin{align*}
C_2 &=1+ \frac{4}{\epsilon}+
  4 \abs{X}\abs{\setr}\sum_{1\leq s\leq T}   \exp(- \tfrac{s \epsilon \kappa^2}{(2(1+\log(1+s)))})\\
  C_3 & = \abs{X} \epsilon \frac{(1+\epsilon)\abs{X}(\norm{\eta'_\epsilon(\optp)}_\infty+\epsilon) \log(T) + C_2}{1+\log(1+(1+\epsilon)\abs{X}(\norm{\eta'_\epsilon(\optp)}_\infty+\epsilon) \log(T) + C_2)}\\ 
  C_4 &= (1+\epsilon) (C(P)+\epsilon+\epsilon\sum_{x\in X}\Delta(x, P)) \log(T)\end{align*}
where $C_2$ bounds the regret accumulated during bad exploration rounds, $C_3$ bounds the regret accumulated while estimating in  good exploration rounds and $C_4$ bounds the regret accumulated while exploring in good exploration rounds.
We finally obtain $\reg(T)\leq C_{\rm{initialize}} + C_{\text{exploit}}+C_{\text{explore}}$, where $C_{\rm{initialize}} =|X|$ and $C_{\text{exploit}}$ is defined in Equation \eqref{eq:c_exploit}. 
 Hence, 
\begin{equation*}
  \limsup_{T\to\infty}~\frac{\reg(T)}{\log(T)} \leq \textstyle (1+\epsilon)(C(\optp)+ \epsilon\left(1+\sum_{x\in X} \Delta(x, \optp)\right). 
\end{equation*}\hfill \Halmos
}

}

\section{Proof of Concentration Bounds in Section \ref{sec:proof:part1}} \label{sec:proof:concentration}
\subsection{Proof of Lemma \ref{lemma:stopping-time-inequality}
}\label{sec:proof_concen}
The proof is inspired by the proof of Lemma 4.3 in \cite{combes2014unimodal}. In this proof, to simplify notation, we denote the stopping time $\stop(s, x)$ by $\stop(s)$.

Let $\mathcal F^\pi_t$ be the $\sigma$-algebra defined in Section \ref{sec:mabp} generated by the rewards observed and arms pulled before round $t$. Let us consider a fixed arm $x$ and reward $r\in \setr$. We then define $B_\tau$ as $\mb 1(x_\tau = x)$; that is, $B_\tau$ is one if arm $x$ is played in round $\tau$, and zero otherwise. Observe that $B_{\tau}$ is $\mathcal F_{\tau}$ measurable. Let 
\[
  S_t := \sum_{\tau=1}^{t-1}  B_\tau \Big(\mb 1(R_\tau(x)=r)-\E{}{\mb 1(R_\tau(x)=r)}\Big)\,.
\]
Note that $S_t$ can be written as 
\[
  S_t= \sum_{\tau=1}^{t-1}  B_\tau \Big(\mb 1(R_\tau(x)=r)-P(r, x)\Big) =  N_t(x)\Big(P_{t}(r, x)- P(r, x)\Big)\,,
\]
where the last equation holds because our estimate of $P(r, x)$ in round $n$, i.e., $P_t(r, x)$, is $\tfrac{\sum_{\tau=1}^{t-1} B_\tau \mb 1(R_\tau(x)=r)}{N_t(x)}$, where $N_t(x)\defn \sum_{\tau=1}^{t-1} B_t$ is the number of times in before round $t$ that arm $x$ is played.
For any $s>1$, we would like to show that \[\Prob{|S_{\stop(s)}|\ge \kappa \cdot  N_{\stop(s)}(x), ~\stop(s)\le T}\le \exp(- s \epsilon \kappa^2/2)\,,\]
where  $\stop(s)\in [s, \dots, T+1]$ is any stopping time such that either {\color{black}$N_{\stop(s)}(x) \geq \phi(s)$} or $\stop(s)=T+1$, and $\kappa$ is a positive number.
Showing the aforementioned inequality gives us the desired result.

With this goal in mind,  define $G_t = \exp(4 \kappa(S_t- \kappa N_t(x))\mb 1(t \le T)$ for all $t\in [1,\dots, T+1]$. Then, we have
\begin{align*}
  \Prob{S_{\stop(s)}\ge \kappa N_{\stop(s)},~ \stop(s)\le T} &~=~\Prob{\left(\exp(4 \kappa(S_{\stop(s)}- \kappa N_{\stop(s)}(x))\cdot \mb 1(\stop(s) \le T)\right)\ge 1}\\
&~=~\Prob{G_{\stop(s)}\ge1}~ \le~ \E{}{G_{\stop(s)}}
\end{align*}
where the ultimate inequality is due to Markov as $G_{\stop(s)}\geq 0$.
In the following, to upper bound $\Prob{S_{\stop(s)}\ge \kappa N_{\stop(s)},~ \stop(s)\le T}$, we upper bound $\E{}{G_{\stop(s)}}$. To do so, we need a few definitions.
Let 
\[\tilde G_t = \exp\left(\sum_{\tau=1}^{t-1} Y_\tau\right) \mb 1(t\le T)\quad \text{with \quad $Y_\tau = B_\tau\Big(4 \kappa \big(\mb 1(R_\tau(x)=r)-P(r, x)\big)-2\kappa^2\Big)$}\,.\]
We note that 
\[G_t = \tilde G_t \exp\left(-N_t(x)\left(4 \kappa^2 -2\kappa^2\right)\right) ~=~ \tilde G_t \exp(-2N_t(x)\kappa^2) \,.\]
{\color{black}We now use the fact that $N_{\stop(s)}(x)\ge \phi(s)$ if $\stop(s)\le T$ to upper bound $G_{\stop(s)}$ by
\[G_{\stop(s)} ~=~ \tilde G_{\stop(s)}\exp(-2N_{\stop(s)}(x)\kappa^2)~\le~ G_{\stop(s)}\exp(-2{\phi(s)\kappa^2})\,. \]
The above inequality holds even when $\stop(s) =T+1$ as $G_{T+1} =\tilde G_{T+1} =0$. Therefore,
\[\E{}{G_{\stop(s)}} \le \E{}{\tilde G_{\stop(s)}}\exp(-2\phi(s)\kappa^2)\]}

So far, we have upper bounded $\E{}{G_{\stop(s)}}$ as a function of $\E{}{\tilde G_{\stop(s)}}$. Next, we will show that $(\tilde G_n)_n$ is a super-martingale sequence. This allows us to upper bound $\E{}{\tilde G_{\stop(s)}}$. First observe that $\E{}{\tilde G_{T+1}| \mathcal F_T} =0 \le \tilde G_T$. For any $t \le T-1$, since $B_{t}$ is $\mathcal F_t$ measurable,  \[\E{}{\tilde G_{t+1}|\mathcal F_t}~=~ \tilde G_t (1-B_{t})+\tilde G_t B_{t}\E{}{\exp (Y_{t})}\,.\]
Recall that when $B_{t} =1$, we have $Y_{t} = 4 \kappa \big(\mb 1(R_{t}(x)=r)-P(r, x)\big)-2\kappa^2$. Then, by invoking Equation 4.16 in \cite{hoeffding1994probability} and considering the fact that $\mb 1(R_{t}(x) =r)\le 1$, we have
\[\E{}{\exp(4 \kappa \big(\mb 1(R_{t}(x)=r)-P(r, x)\big))}\le \exp(2\kappa^2)\,.\]
This implies that $\E{}{\exp (Y_{t})} \le 1$, and as a result, $(\tilde G_t)_t$ is a super-martingale sequence: $\E{}{\tilde G_{t+1}|\mathcal F_t}\le \tilde G_{t}$. Since $\stop(s) \le T+1$, sequence $(\tilde G_t)_t$ is super-martingale, we can apply the Doob's optional stopping theorem, to get $\E{}{\tilde G_{\stop(s)}}\le \E{}{\tilde G_{1}} =1 $. {\color{black}Putting everything together, we have 
\begin{align*}
\Prob{S_{\stop(s)}\ge \kappa N_{\stop(s)},~ \stop(s)\le T}~\le~ \E{}{G_{\stop(s)}} 
~\le~  \E{}{\tilde G_{\stop(s)}} \exp(-2\phi(s) \kappa^2) \le  \exp(-2\phi(s)  \kappa^2)\,.
\end{align*}
Then, by symmetry, we can show that 
\begin{align*}
\Prob{|S_{\stop(s)}|\ge \kappa N_{\stop(s)},~  \stop(s)\le T} \le  2\exp(-2\phi(s) \kappa^2)\,,
\end{align*}
which is the desired result. }
\hfill\Halmos

\subsection{Proof of Lemma \ref{lemma:concentration}}

Let us take $\setr= \{1, 2, 3, \ldots, k\}=[1,\dots, k]$ without loss of generality.  
 Define $\delta\geq \abs{X}(\abs{\setr}-1)+1$ and $\eta=\tfrac{1}{(\delta-1)}>0$. The core of the argument involves partitioning the event set of interest in geometrically sized slices. Define  $D\defn \ceil{\log(t)/\log(1+\eta)}$ as the number of such slices. Let now $\mc D\defn \{1, \dots, D\}^{(\abs{\setr}-1) \abs{X}}$ and $\mc Y \defn \{-1, 1\}^{(\abs{\setr}-1) \abs{X}}$. Introduce the event 
  \begin{align*}
    A   =& \{\textstyle\sum_{x\in X}N_t(x) I(P_t(x), P(x))\geq \delta \},
    \end{align*}
    and partitions
    \begin{align*}
    B_d =& \cap_{x\in X, ~r\in [k-1]} \{(1+\eta)^{d(r, x)-1}\leq N_t'(r, x) \leq (1+\eta)^{d(r, x)} \}, \\[0.5em]
    C_y =& \cap_{x\in X, ~r\in [k-1]} \{ P'_t(r, x) y(r, x) \leq P'(r, x) y(r, x) \}
  \end{align*}
  indexed in $d\in \mc D$ and $y \in \mc Y$.
  Here, 
  $$N_t'(r, x) \defn \sum_{\tau=1}^{t-1} \mb 1\{x_\tau=x, \, R_{\tau}(x)\geq r\}\quad r\in [1, \dots, k-1], x\in X$$
  is the number of times  before round $t$ in which we play arm $x$ and receive a reward of at least $r$. Furthermore, 
  \begin{align}
P'(r, x) &\defn \frac{P(r, x)}{\textstyle\sum_{r'\geq r}P(r', x)}, \quad r\in [1, \dots, k-1], x\in X,\\
P_t'(r, x) &\defn \frac{P_t(r, x)}{\textstyle\sum_{r'\geq r}P_t(r', x)}, \quad r\in [1, \dots, k-1], x\in X.
\end{align}
We will also slightly abuse notation and extent  $B_d$ to all $d\in \mc D'=\{0, \dots, D\}^{ (\abs{\setr}-1)\abs{X}}$ with the understanding that $d(r, x)=0$ implies $N'_t(r, x)=0$. The index $y(r, x)$ partitions the events into two cases:  (1)  $y(r, x)=1$ under which  $P'_t(r, x) \leq P'(r, x)$ and (2) $y(r, x)=-1$ under which  $P'_t(r, x) \ge P'(r, x)$.

  We have that $A = \cup_{d\in \mc D', \,y\in \mc Y}(A\cap B_d \cap C_y)$ and this allows us to use the union bound:
  \begin{equation}
    \label{eq:union-inequality}
    \Prob{A} \leq \sum_{d\in \mc D'}\sum_{y\in \mc Y} \Prob{A\cap B_d \cap C_y}.
  \end{equation}
  Consider a particular $d\in \mc D'$ and $y\in \mc Y$. Applying Lemma \ref{prop:slice-bound} with $\bar N'(r, x)=(1+\eta)^{d(r, x)-1}$ and $\delta\geq (1+\eta)\abs{X}(\abs{\setr}-1)$, we obtain the inequality
  \[
    \Prob{A\cap B_d \cap C_y} \leq \left(\frac{\delta e}{\abs{X}(\abs{\setr}-1)}\right)^{\abs{X}(\abs{\setr}-1)}\cdot \exp(-\tfrac{\delta}{(1+\eta)}).
  \]
  We have $\abs{\mc D'\times \mc Y }=(2 (D+1))^{\abs{X}(\abs{\setr}-1)}$ and thus from Inequality \eqref{eq:union-inequality}, we obtain
  \[
    \Prob{A} \leq \left(\frac{2 (D+1) \delta e }{\abs{X}(\abs{\setr}-1)}\right)^{\abs{X}(\abs{\setr}-1)} \exp(-\tfrac{\delta}{(1+\eta)}).
  \]
  With the choice $\eta=\tfrac{1}{(\delta-1)}$ and using $\log(1+\eta)=-\log(1/(1+\eta))\geq 1-1/(1+\eta)=1/\delta$, we obtain
  \[
    \Prob{A} \leq \left(\frac{\delta \ceil{\log(T)\delta+1}2e}{\abs{X}(\abs{\setr}-1)}\right)^{\abs{X}(\abs{\setr}-1)} \exp(-\tfrac{\delta}{(1+\eta)}).
  \]
  Furthermore, as $1/(\eta+1)=(\delta-1)/{\delta}$, we get the desired result
  \[
    \Prob{A} \leq \left(\frac{\delta \ceil{\log(T) \delta+1} 2 e}{\abs{X}(\abs{\setr}-1)}\right)^{\abs{X}(\abs{\setr}-1)}e\cdot\exp(-\delta).
  \]
  \hfill \Halmos
  
  \begin{lemma}
  \label{prop:slice-bound}
  For any $r\in [1, \dots, k-1]$ and $x\in X$, let $\bar{{N}'}(r, x)\in [0, \dots, t]$, where $k=|\setr|$, $\setr= \{1, 2, \ldots, k\}$. For any $\eta >0$, define the event
  \begin{align*}
    B\defn & \cap_{x\in X}\cap_{r=1}^{k-1} \{ \bar N'(r, x) \leq N_t'(r, x) \leq (1+\eta) \bar N'(r, x) \} \\[0.5em]
    C\defn & \cap_{x\in X}\cap_{r=1}^{k-1} \{ P'_t(r, x) y(r, x) \leq P'(r, x) y(r, x) \}
  \end{align*}
  for $y(r, x)\in \{+1, -1\}$, $r\in [1, \dots, k-1]$ and $x\in X$.
  For $\delta\geq (1+\eta) (\abs{\setr}-1)\cdot \abs{X}$,  we have the inequality
  \[
    \Prob{
      B\cap  C \cap \left(\sum_{x\in X} N_t(x) I(P_t(x), P(x))\geq \delta\right) } \leq  \left(\frac{\delta e}{(\abs{\setr}-1)\cdot \abs{X}}\right)^{(\abs{\setr}-1)\cdot \abs{X}} \exp(-\tfrac{\delta}{(1+\eta)}).
  \]
\end{lemma}\medskip

\proof{Proof of Lemma \ref{prop:slice-bound}.}
The proof of this lemma takes advantage of the decomposition result presented in Lemma \ref{lemma:decomposition-P}. This lemma allows us to decompose $N_t(x) I(P_t(x), P(x))$ as follows
\[N_t(x) I(P_t(x), P(x)) = \sum_{r=1}^{k-1} N'_t(r, x) I_B(P'_t(r, x), P'(r, x))\,, \]
where $N_t'(r,x) := \sum_{r'\geq r}N_t(r',x)$, $r\in [1, \dots, k-1]$, $x\in X$, and 
$$
I_B( P_t'(r,x), P'(r,x))\defn  I(( P_t'(r,x), 1- P_t'(r,x)), (P'(r,x), 1-P'(r,x)))
$$
denotes the information distance between two Bernoulli distributions with mean $P_t'(r,x)$ and $P'(r,x)$. Recall that for any $r\in [1, \dots, k-1]$, $P'(r,x) = \tfrac{P(r,x)}{\sum_{r'\geq r}P(r',x)}$ and  $P_t'(r,x) = \tfrac{ P_t(r,x)}{\sum_{r'\geq r} P_t(r',x)}$; that is, $P'(r,x)$ is the probability of receiving a reward of $r$ by pulling arm $x$, conditioned on the reward being greater than or equal to $r$. Having this decomposition in mind, we will show that for any $\xi(r, x)\geq 0$, $r\in [1, \dots, k-1]$ and $x\in X$,  the following inequality holds
   \begin{align} \label{eq:IB}
     \begin{split}
    & \Prob{\cap_{x\in X}\cap_{r=1}^{k-1} \{B~\cap ~ C~ \cap ~ \left(N'_t(r, x) I_B(P'_t(r, x), P'(r, x))\geq \xi(r, x)\right)\}} \\
    & \leq   \exp(\textstyle\sum_{x\in X}\sum_{r=1}^{k-1} \xi(r, x)/(1+\eta)).
  \end{split}
   \end{align}
  Then, by applying the stochastic dominance bound, stated in Lemma \ref{prop:stochastic-dominance},  where we use $Z(r, x)=\mb 1\{B\cap C\} \cdot \mb 1\{N'_t(r, x) I_B(P'_t(r, x), P'(r, x))\}$ and $a=\tfrac{1}{(1+\eta)}$,  we get
  \begin{align*}
   \Prob{
      B\cap  C \cap \left(\sum_{x\in X} N_t(x) I(P_t(x), P(x))\geq \delta\right) } \leq & \left(\frac{\delta e}{\abs{X}(\abs{\setr}-1) (1+\eta)}\right)^{\abs{X}(\abs{\setr}-1)}\hspace{-0.5em} \exp(-\delta/(1+\eta)) \\[0.5em]
            \leq & \left(\frac{\delta e}{\abs{X}(\abs{\setr}-1)}\right)^{\abs{X}(\abs{\setr}-1)} \exp(-\delta/(1+\eta))
  \end{align*}
  establishing the claim. 

{It remains to show Inequality \eqref{eq:IB}. We have that
  \begin{align*}
    & \Prob{\cap_{x\in X}\cap_{r=1}^{k-1}  \{B~\cap~ C ~\cap ~(N'_t(r, x) I_B(P'_t(r, x), P'(r, x))\geq \xi(r, x))\}} \\[0.5em]
    ~\leq~ & \Prob{\cap_{x\in X}\cap_{r=1}^{k-1}  \{B~\cap~ C ~\cap ~(\bar N'(r, x) I_B(P'_t(r, x), P'(r, x))\geq \xi(r, x)/(1+\eta) )\}}.
  \end{align*}
  We define now $Q'(r, x)$ for any $r\in [1,\dots, k-1]$ and $x\in X$ such that if there exists a $q'\in [0, 1]$ such that $q' y(r, x) \leq P'(r, x) y(r, x)$ and  $\bar N'(r, x)\cdot I_B(q', P'(r, x))=\xi(r, x)/(1+\eta)$, then $Q'(r, x)=q'$. Evidently, if no such $q'$ exists,  then the event $(\bar N'(r, x) I_B(P'_t(r, x), P'(r, x))\geq \xi(r, x)/(1+\eta) )$ can also not occur and Inequality \eqref{eq:IB} holds trivially. We now assume that we are not in this trivial case and we have
  \begin{equation}
    \label{eq:nontrivial-case}
    \bar N(r, x) I_B\left(Q'(r, x)~,~ P'(r. x)\right)= \xi(r, x)/(1+\eta) \quad \forall r\in [1, \dots, k-1], x\in X.
  \end{equation}  
  With this definition, we have
\begin{align}
  \label{eq:implication:deviation}
  \begin{split}
    \bar N'(r, x)I_B(P'_t(r, x), P'(r, x))\geq  \xi(r, x)/(1+\eta) ~\&~  P'_t(r, x) y(r, x)\leq & P'(r, x) y(r, x) \\
    \qquad \implies  P'_t(r, x) y(r, x)  \leq  Q'(r, x) y(r, x)&
  \end{split}
\end{align}
where this holds because of the convexity and continuity of the function $I_B$.
Hence,
\begin{align*}
      & \Prob{\cap_{x\in X}\cap_{r=1}^{k-1}  \{B~\cap~ C ~\cap ~(N'_t(r, x) I_B(P'_t(r, x), P'(r, x))\geq \xi(r, x))\}} \\[0.5em]
  ~\leq~ & \Prob{\cap_{x\in X}\cap_{r=1}^{k-1} \{ B~\cap ~ (P'_t(r, x) \cdot y(r, x) \leq Q'(r, x)\cdot  y(r, x))\}} \\[0.5em]
    ~\leq~ & \textstyle\prod_{x\in X}\prod_{ r=1}^{k-1} \exp\left(- \bar N(r, x) I_B\left(Q'(r, x)~,~ P'(r. x)\right)\right)\\[0.5em]
    ~ \leq~ & \exp(-\textstyle\sum_{x\in X, r\in [k-1]} \xi(r, x)/(1+\eta)).
  \end{align*}
  The first inequality follows immediately from our implication in Equation (\ref{eq:implication:deviation}). To obtain the penultimate inequality, we apply Lemma \ref{prop:deviation-bound}. The ultimate inequality follows from Equation (\ref{eq:nontrivial-case}).\hfill\Halmos
  }
\endproof

\begin{lemma}\label{lemma:decomposition-P}
  Let $\widehat P$ and $P$ be two distributions on the event space $\setr=[1, \dots, k]$.  
  Let $P'$ and $\widehat P'$, defined on the event space $[1, \dots, k-1]$,   be their Bernoulli transformations:
    \begin{align}P'(r) = \frac{P(r)}{1-\sum_{r'<r}P(r')}, \qquad \widehat P'(r) = \frac{\widehat P(r)}{1-\sum_{r'<r}\widehat P(r')}, \quad  \forall r \in [1, \dots, k-1]\,. \label{eq:decomposition-P}\end{align}
  Let $N\geq 0$ and with a slight abuse of notation, define $N(r):=N\cdot \widehat P(r)$ for all $r\in [1, \dots, k]$. Further define $N'(r) := N-\sum_{r'<r}N(r')$ for all $r\in [1, \dots, k-1]$. Then,
  \[
    N \cdot I(\widehat P, P) = \sum_{r=1}^{k-1} N'(r)\cdot I_B(\widehat P'(r), P'(r))\,,
  \]
  where we use the shorthand notation $$I_B(\widehat P'(r), P'(r))\defn  I((\widehat P'(r), 1-\widehat P'(r)), (P'(r), 1-P'(r)))$$
  to denote the information distance between two Bernoulli distributions.
\end{lemma}
\proof{Proof of Lemma \ref{lemma:decomposition-P}.} We will use a proof by induction on the support size $\abs{\setr}=k$.
  The case where the support size $k=2$ is trivial.
  For the sake of induction, we assume that the statement is true for distributions supported on $k-1$ points.
  By the chain rule, we  have
  \begin{align} \label{eg:decomp}
    \begin{split}
      & N \cdot I(\widehat P, P) \\
      & =  N\cdot I_B(\widehat P(1), P(1)) + N (1-\widehat P(1)) \cdot I\left( \frac{(\widehat P(2), \dots, \widehat P(k))}{(1-\widehat P(1))}, \frac{(P(2), \dots, P(k))}{(1- P(1))}\right).
    \end{split}
  \end{align}
  Define $\widehat P_s=(\widehat P(2), \dots, \widehat P(k))/(1-\widehat P(1))$ and $P_s=(P(2), \dots, P(k))/(1- P(1))$ and observe that both are distributions supported on $k-1$ points.   That is,  $\sum_{r=2}^{k}\widehat P_s(r)=\sum_{r=2}^{k}P_s(r)=1$. Further, Let $N_s \defn N \cdot (1-\widehat P(1))$ and define  $N_s(r) \defn N_s \widehat P_s(r)$ and $N'_s(r) \defn  N_s- \sum_{2\leq r'<r}N_s(r')$ for $r\in \{2, \dots, k\}$. With these definitions, and by applying the induction assumption, Equation \eqref{eg:decomp} can be written as 
  \[ N \cdot I(\widehat P, P)
    =  N\cdot I_B(\widehat P(1), P(1)) + \sum_{r=2}^{k-1} N'_s(r) \cdot I_B(\widehat P'_s(r), P_s'(r))\]
 where 
  \begin{align*}
    \widehat P'_s(r) & = \frac{\widehat P_s(r)}{1-\sum_{2\leq r'<r}\widehat P_s(r')}
                  = \frac{\widehat P(r)}{(1-P(1)) (1-\sum_{2\leq r'<r}\tfrac{\widehat P_s(r)}{(1-P(1))})} \\[0.5em]
                 & = \frac{\widehat P(r)}{(1-P(1)-\sum_{2\leq r'<r}\widehat P_s(r))}  = \frac{\widehat P(r)}{1-\sum_{r'<r}\widehat P_s(r))} = \widehat P'(r)
  \end{align*}
  for all $r\in [2, \dots, k]$. Here, the last equality follows from Equation \eqref{eq:decomposition-P}. The same argument can be made to argue that  $P'_s(r)=P'(r)$ for all $r\in [2, \dots, k]$.
  So far, we established that
    \begin{align*} N \cdot I(\widehat P, P)
    =  N\cdot I_B(\widehat P(1), P(1)) + \sum_{r=2}^{k-1} N_s'(r) \cdot I_B(\widehat P'(r), P'(r))\,,\end{align*}
    To complete the proof, it suffices to show that for any $r\in [2, \ldots, k]$, $N_s'(r)= N'(r)$:
  \begin{align*}
    N'_s(r) = & N_s- \sum_{2\leq r'<r}N_s(r') \\[0.5em]
    = &  (1-\widehat P(1)) \cdot (N - N \sum_{2\leq r' < r}\widehat P_s(r')) \\[0.5em]
    = & N (1-\widehat P(1)) \cdot (1 - \sum_{2\leq r' < r}\widehat P_s(r')) \\[0.5em]
    = & N (1-\widehat P(1)) \cdot (1 - \sum_{2\leq r' < r}\widehat P(r')/(1-\widehat P(1))) \\[0.5em]
    = & N \cdot (1 - \widehat P(1) -  \sum_{2\leq r' < r}\widehat P(r')) \\[0.5em]
    = & N \cdot (1- \sum_{r' < r}\widehat P(r')) = N'(r)\,. 
  \end{align*}\hfill\Halmos
\endproof

 \begin{lemma} [Lemma 8 of \cite{pmlr-v35-magureanu14}] \label{prop:stochastic-dominance}
  Let $a>0$, $d\geq 2$, and  $Z\in\mathbb{R}^{d}_+$ be a random variable such that for any $\xi\in {\mathbb R}^d_+$, we have that
  \[
    \Prob{\cap_{i=1}^{d}(Z_i>\xi_i)} \leq \exp(-a \cdot \iprod{\textbf{1}_d}{\xi}).
  \]
  Then, for all $\delta\geq \frac{d}{a}>0$ we have 
  \[
    \Prob{\iprod{\textbf{1}_d}{Z} \geq \delta}\leq \left(\frac{a \delta e}{d}\right)^d \exp(-a \delta)\,,
  \]
  where $\textbf{1}_d$ is a $d$-dimensional all-ones vector. 
\end{lemma}
\proof{Proof of Lemma \ref{prop:stochastic-dominance}.} See the proof of Lemma 8 in \cite{pmlr-v35-magureanu14}.\hfill \Halmos
\endproof

\begin{lemma}
  \label{prop:deviation-bound}
  For any $x\in X$ and $r\in [1, \dots, k-1]$, let $\bar N'(r, x) \in [0, \dots, t]$ and $y(r,x)\in \{-1,1\}$. Consider the events
  \begin{align*}
    B\defn & \cap_{x\in X}\cap_{r=1}^{k-1} \{ \bar N'(r, x) \leq N_t'(r, x) \leq (1+\eta) \bar N'(r, x) \} \\[0.5em]
    C\defn & \cap_{x\in X}\cap_{r=1}^{k-1} \{ P'_t(r, x) y(r, x) \leq Q'(r, x) y(r, x) \}
  \end{align*}
  for $y(r, x)\in \{+1, -1\}$ for all $r\in [1, \dots, k-1]$ and $x\in X$. Assume that $Q'(r, x) y(r, x)\leq P'(r, x) y(r, x)$ for all $r\in [1, \dots, k-1]$ and $x\in X$.
  Then, 
  \[
    \Prob{B\cap C}  \leq \textstyle \prod_{x\in X}\prod_{r=1}^{k-1} \exp(- \bar N'(r, x) I_B(Q'(r, x), P'(r, x)))\,.
  \]
\end{lemma}

\proof{Proof of Lemma \ref{prop:deviation-bound}.}
 We start by introducing a log moment-generating function related to the rewards for each arm.
  In particular, for all $x\in X$, $r\in [1, \dots, k-1]$, and $\lambda$, we define
  \[
    \phi(r, x; \lambda) \defn \log\left( \E{}{\exp(\lambda \cdot \mb 1\{R_t(x)=r\})|R_t(x)\geq r}\right).
  \]
  Evidently, we have that these functions $\phi(r, x)$ are not a function of time $t$ as the reward distributions are identically distributed over time.
  Recall that
  \begin{align*}
    \Prob{R_t(x)=r|R_t(x)\geq r} = &  \frac{\Prob{R_t(x)=r ~ \& ~ R_t(x)\geq r}}{\Prob{R_t(x)\geq r}} =\frac{\Prob{R_t(x)=r}}{\Prob{R_t(x)\geq r}} = \frac{P(r, x)}{\sum_{r'\geq r} P(r', x)} \\[0.5em]
    = &  P'(r, x)
  \end{align*}
  for all $r\in [1, \dots, k-1]$ and $x\in X$. Therefore, we have that $\phi(r, x; \lambda)$ is the moment-generating function of a Bernoulli variable with success parameter $P'(r, x)$. That is,
  \[
    \phi(r, x; \lambda) =  \log\left(P'(r, x) \exp(\lambda) + (1-P'(r, x))\right).
  \]
  It is also easy to show that
  \begin{align} \label{eq:IB_2}
    I_B(Q'(r, x), P'(r, x)) = & \max_{\lambda: \lambda y(r, x)\leq 0}\,\left\{Q'(r, x) - \phi(r, x; \lambda)\right\}
  \end{align}
  as indeed by assumption $Q'(r, x) y(r, x) \leq P'(r, x)y(r, x)$.  
  That is, the relative entropy of a Bernoulli distribution is the Fenchel dual of its log moment-generating function.  We then define
  \begin{align} \label{eq:lambda*}
    \lambda^\star(r, x) \defn & \arg\max_{\lambda: \lambda y(r, x)\leq 0}\,\left\{Q'(r, x) \lambda - \phi(r, x; \lambda)\right\}\,,
  \end{align}
  and let
  \begin{align}
  \begin{split}
    & G_t\\
    \defn & \exp\Big(\sum_{x\in X}\sum_{r=1}^{k-1} \big(\lambda^\star(r, x) \sum_{\tau=1}^{t-1} \mb 1\{x_\tau=x, R_\tau(x)=r\} - \sum_{\tau=1}^{t-1} \mb 1\{x_\tau=x, R_\tau(x)\geq r\} \phi(r, x;\lambda^\star(r, x))\big)\Big)\\
    = & \exp\Big(\textstyle\sum_{x\in X}\sum_{r=1}^{k-1} \big(\textstyle\lambda^\star(r, x)  N_t(r,x) - N_t'(r,x)\phi(r, x;\lambda^\star(r, x))\big)\Big)\,,
    \end{split}\label{eq:G}
  \end{align} 
  where $N_t'(r,x)= \sum_{\tau=1}^{t-1} \mb 1\{x_\tau=x, R_\tau(x)\geq r\}$, $r\in [1, \dots, k-1]$, $x\in X$, is the number of times before round $t$ that arm $x$ is played and we receive of reward of at least $r$.
  We will argue later that $\E{}{G_t}=1$. 
   Now assume that is indeed the case, then 
  \begin{align*}
     \Prob{B\cap C} 
    = & \Prob{\cap_{x\in X}\cap_{r=1}^{k-1} \left\{ P_t'(r, x) y(r, x) \leq  Q'(r, x) y(r, x), ~B \right\} } \\[0.5em]
    = & \Prob{\cap_{x\in X}\cap_{r=1}^{k-1} \left\{\frac{N_t(r,x)}{N'_t(r,x)} y(r, x)\leq  Q'(r, x) y(r, x), ~B \right\} } \\[0.5em]
        = & \Prob{\cap_{x\in X}\cap_{r=1}^{k-1} \left\{N_t(r,x) y(r, x)\leq  Q'(r, x) y(r, x)  N'_t(r,x), ~B \right\} }\,. 
  \end{align*}
  Note in particular that the last characterization of $\Prob{B\cap C}$ in fact allows for the event $N'_t(r, x)=0$. 
        Using the fact that, we have $\lambda^\star(r, x) y(r,x) \le 0$, we have
        \begin{align*}
          & \Prob{B\cap C}\\[0.5em]
          \leq & {\rm{Prob}}\Bigg[\sum_{x\in X}\sum_{r=1}^{k-1} \lambda^\star(r, x) y^2(r,x) N_t(r,x) \geq \sum_{x\in X}\sum_{r=1}^{k-1} \lambda^\star(r, x) y^2(r, x) Q'(r, x) N'_t(r,x), ~B \Bigg] \\[0.5em]
    = &  {\rm{Prob}}\Bigg[ \mb 1\{B\}\exp\Big(\sum_{x\in X}\sum_{r=1}^{k-1} \lambda^\star(r, x) N_t(r,x)\Big) \geq  \exp\Big(\sum_{x\in X}\sum_{r=1}^{k-1} \lambda^\star(r, x) Q'(r, x) N_t'(r,x)\Big) \Bigg]\,, 
    \end{align*}
    where the last equations follows from the monotonously of the exponential function and the fact that $y^2(r,x)=1$ as $y(r, x)=\{-1, 1\}$ for all $x\in X$ and $r\in [1,\dots, k-1]$. Then,  the definition of $G_t$ gives us
  \begin{align*}
    & \Prob{B\cap C}  \\
    \leq & \Prob{B \cap \left(G_t\geq \exp\Big(\textstyle\sum_{x\in X}\sum_{r=1}^{k-1} \big(\lambda^\star(r, x) Q'(r, x)-\phi(r, x;\lambda^\star(r, x)) \big)N_t'(r, x)\Big)\right)}\\
    = & \Prob{B \cap  \Big(G_t\geq \exp\left(\textstyle\sum_{x\in X}\sum_{r=1}^{k-1} I_B(Q'(r, x), P'(r, x)) \cdot N_t'(r,x)\right)\Big)}\\
   =&\Prob{G_t \cdot \mb 1\{B\}\geq \exp\left(\textstyle\sum_{x\in X}\sum_{r=1}^{k-1} I_B(Q'(r, x), P'(r, x)) \cdot N_t'(r,x)\right)}\\
   \le & \Prob{G_t \cdot \mb 1\{B\}\geq \exp\left(\textstyle\sum_{x\in X}\sum_{r=1}^{k-1} I_B(Q'(r, x), P'(r, x)) \cdot \bar N'(r,x)\right)}
  \\   \le & \E{}{G_t \cdot \mb 1\{B\}}\exp\left(-\textstyle\sum_{x\in X}\sum_{r=1}^{k-1} I_B(Q'(r, x), P'(r, x)) \cdot \bar N'(r,x)\right)
    \,,\\
     \le & \exp\left(-\textstyle\sum_{x\in X}\sum_{r=1}^{k-1} I_B(Q'(r, x), P'(r, x)) \cdot \bar N'(r,x)\right)\,,
    \end{align*}
   where the second equation follows from definition of $\lambda^*$, given in Equation \eqref{eq:lambda*}, and the fact that   $I_B(Q'(r, x), P'(r, x)) = \max_{\lambda: \lambda y(r, x)\leq 0}\,\left\{Q'(r, x) - \phi(r, x; \lambda)\right\}$ per Equation \eqref{eq:IB_2}.  The first inequality holds because under event $B$, the random variable $N'_t(r,x)$ is greater than or $\bar N'(r,x)$, and the second inequality  follows from Markov's inequality. Finally, the last inequality follows from our earlier claim that ${\E{}{{G_t}}}=1$ which implies $\E{}{G_t \cdot \mb 1\{B\}}\leq 1$.

 We still need to show that $\E{}{G_t}=1$. Recall that $G_t$ is a function of random variables $R_{\tau}(x_\tau)$ for $\tau\in [1, \dots, t]$; see Equation \eqref{eq:G}. In this following, for every $x\in X$, and $\tau\in [1, \dots, t]$ we define another random variable $R_{\tau}'(x)$ whose distribution is the same as  $R_{\tau}(x)$. To construct $R_{\tau}'(x)$, we first define $k-1$ independent Bernoulli random variables, $b_{\tau}(r,x)\sim \text{Bernoulli}(P'(r, x))$, $r\in [1, \dots, k-1]$, where $\text{Bernoulli}(p)$ is a  Bernoulli random variable with a success rate of $p$. Then,  if $b_{\tau}(r,x) =0$ for all $r\in [1, \dots, k-1]$, we set $R_{\tau}'(x)$ to $k$. Otherwise, $R_{\tau}'(x) =\min\{r: b_{\tau}(r,x)= 1\}$. Note that for any $\tau_1$ and $\tau_2$ and $x\in X$, distribution of $R'_{\tau_1}(x)$ is the same as that of $R'_{\tau_2}(x)$, and the random variables $R'_{\tau_1}(x)$ and $R'_{\tau_2}(x)$ are independent of each other. Lemma \ref{lem:dist:R_R'}, presented at the end of this section shows that for any $x\in X$ and $\tau\ge 0$, the distribution of $R'_{\tau}(x)$ is the same as $R_{\tau}(x)$.
  Consequently, the random variable
  \begin{align*}
    & G'_t\\
    = & \exp\left(\sum_{x\in X}\sum_{r=1}^{k-1} \left(\lambda^\star(r, x) \sum_{\tau=1}^{t-1} \mb 1\{x_\tau=x,\, R'_\tau(x)=r\} - \sum_{\tau=1}^{t-1} \mb 1\{x_\tau=x, R'_\tau(x)\geq r\} \phi(r, x; \lambda^\star(r, \, x))\right)\right)
  \end{align*}
  has the same distribution as $G_t$. Hence, we will show  that $\E{}{G_t}=\E{}{G'_t}=1$ using a martingale argument.
  
  Define  the following filtrations 
  \begin{align*}
    \mc F^{\pi}_\tau{'}(0) & \defn \sigma((x_1, R_1'(x_1)), \dots, (x_{\tau-1}, R'_{\tau-1}(x_{\tau-1}))) \quad \forall \tau\in [1, \dots, t],\\[0.5em]
    \mc F^{\pi}_\tau{'}(r) & \defn \sigma((x_1, R_1'(x_1)), \dots, (x_{\tau-1}, R'_{\tau-1}(x_{\tau-1})), (x_\tau, b_{\tau}(1, x_{\tau}), \dots, b_{\tau}(r, x_{\tau})) )\\
                       & \hspace{4em} \forall \tau\in [1, \dots, t], ~r\in [1, \dots, k-1],
  \end{align*}
  where $b_{\tau}(r,x)$, $r\in [1, \dots, k-1]$, are the Bernoulli variables used to construct the rewards $R'_{\tau}(x)$ in round $\tau$. Observe that $\mc F^\pi_\tau{'}(r)$ is independent from the Bernoulli variables $b_{\tau}(r',x_{\tau})$ for $r'>r$. 
  Evidently, we have the inclusions
  \begin{align*}
    \mc F^\pi_\tau{'}(r-1) \subseteq & \mc F^\pi_\tau{'}( r) \quad~~ \forall \tau \in [1, \dots, t], ~r \in [1, \dots, k-1], \\[0.5em]
    \mc F^{\pi}_{\tau}{'}( k-1) = & \mc F^{\hspace{0.5em}\pi}_{\tau+1}{'}( 0) \quad \forall \tau\in [1, \dots, t-1].
  \end{align*}
  The first inclusion follows immediately by construction. The second equality follows from the fact that $x_{\tau}$ is $\mc F^{\pi}_{\tau}{'}(0)\subseteq \mc F^{\pi}_{\tau}{'}(k-1)$ measurable and $R'_{\tau}(x_{\tau})$ is revealed once the random variables $b_\tau(r, x_\tau)$, $r\in [1, \dots, k-1]$, are revealed and hence is $\mc F^{\pi}_{\tau}{'}( k-1)$-measurable as well. Hence, the filtration sequence $\mc F^{\pi}_{\tau}{'}(r)$ for $\tau\in [1, \dots, t]$ and $r\in [0, \dots, k-1]$ is increasing; more precisely we have
  \begin{equation}
    \label{eq:martingale-sequence}
    \mc F^{\pi}_1{'}(0) \subseteq \mc F^{\pi}_1(1){'} \subseteq \dots \subseteq \mc F^{\pi}_1{'}(k-1 ) = \mc F^{\pi}_2{'}(0 ) \subseteq \mc F^{\pi}_2{'}(1 )\subseteq \dots \subseteq \mc F^{\pi}_t{'}(k-2) \subseteq \mc F^{\pi}_t{'}(k-1).
  \end{equation}
  
  Define for any round $\tau\in[1, \dots, t]$ the random variables
  \begin{align*}
    G_\tau'(0) = & G_\tau',\\
    G_\tau'(r) = & G'_{\tau} \cdot \prod_{x\in X, r'\leq r}\exp\left(\mb 1\{x_{\tau}=x\}  \Big[\lambda^\star(r', x) \mb 1\{R'_{\tau}(x)=r'\}-\mb 1\{R'_{\tau}(x)\geq r'\} \phi(r',x; \lambda^\star(r', x) )\Big]\right) \\
    & \hspace{4em} \forall r\in [1, \dots, k-1].
  \end{align*}
  We will show that $G_\tau'(r)$ is a martingale with respect to the filtration sequence stated in  Equation (\ref{eq:martingale-sequence}) and thus
  \(
    \E{}{G'_t( r)}=\E{}{G'_1( 0)}=1\,,
  \) which is the desired result. 
 
 First observe that for any round $\tau\in [1, \dots, t]$ we have  \[\E{}{G'_\tau(0)|\mc F^{\hspace{0.5em}\pi}_{\tau-1}{'}( k-1)}=\E{}{G'_{\tau}|\mc F^{\pi}_{\tau}{'}(0)}=\E{}{G'_{\tau-1}(k-1)|\mc F^{\pi}_{\tau}{'}(0)} =G'_{\tau-1}(k-1).\]
  Furthermore, for any round $\tau\in [1, \dots, t]$ and $r\in [1, \dots, k-1]$, we have 
  \begin{align*}
     & \E{}{G'_\tau(r)|\mc F^\pi_\tau{'}(r-1)}\\[0.5em]
    = & G'_{\tau}\cdot\mb E\Big[\prod_{r'\leq r} \exp\Big(\lambda^\star(r', x_{\tau}) \mb 1\{R'_{\tau}(x_{\tau})=r'\}\\ & \qquad -\mb 1\{R'_{\tau}(x_{\tau})\geq r'\} \phi(r',x_{\tau} ; \lambda^\star(r', x_{\tau}))\Big)|\mc F^\pi_\tau{'}(r-1)\Big] \\[0.5em]
    = & G'_{\tau} \textstyle \prod_{r'\leq r-1} \exp(\lambda^\star(r', x_\tau) \mb 1\{R'_\tau(x_\tau)=r'\}-\mb 1\{R'_\tau(x_\tau)\geq r'\} \phi(r',x_\tau; \lambda^\star(r', x_\tau) )) \\
      & \qquad \cdot\E{}{ \exp(\lambda^\star(r, x_\tau) \mb 1\{R'_\tau(x_\tau)=r\}-\mb 1\{R'_\tau(x_\tau)\geq r\} \phi(r,x_\tau; \lambda^\star(r, x_\tau) ))~|~\mc F_\tau^\pi{'}( r-1)}  \\[0.5em]
    = & G'_\tau(r-1)\cdot \frac{\E{}{\exp(\lambda^\star(r, x_\tau) \mb 1\{R'_\tau(x_\tau)=r\})~|~\mc F^\pi_\tau{'}(r-1)}}{\exp(\mb 1\{R'_\tau(x_\tau)\geq r\} \phi(r, x_\tau; \lambda^\star(r, x_\tau)))} \\[0.5em]
    = & G'_\tau(r-1)\,. 
  \end{align*}
  Here, the first equation follows from the fact that $x_\tau$ is $\mc F^{\pi}_\tau{'}(0)\subseteq \mc F^{\pi}_\tau{'}(r-1)$ measurable. The second equation follows from the fact that both
  $\mb 1\{R'_\tau(x_\tau)=r'\}$ and $\mb 1\{R'_\tau(x_\tau)\geq r'\}$ are $\mc F^{\pi}_\tau{'}( r-1)$-measurable for $r'\leq r-1$. The third equation  can be established as by remarking that $\mb 1\{R'_\tau(x_\tau)\geq r\}$ is also $\mc F^{\pi}_\tau{'}( r-1)$ measurable. To establish the last equality we consider two cases. Conditioned on the event $\mb 1\{R'_\tau(x_\tau)\geq r\} =0$,  we have $\mb 1\{R'_\tau(x_\tau)= r\} =0$ as well, and thus
  \begin{align*}
    \mb 1\{R'_\tau(x_\tau)\geq r\} =0 \implies \frac{\E{}{\exp(\lambda^\star(r, x_\tau) \mb 1\{R'_\tau(x_\tau)=r\})~|~\mc F^{\pi}_\tau{'}(r-1)}}{\exp\Big(\mb 1\{R'_\tau(x_\tau)\geq r\} \phi(r, x_\tau; \lambda^\star(r, x_\tau))\Big)}=1.
  \end{align*} Conditioned on the event $\mb 1\{R'_\tau(x_\tau)\geq r\} =1$, we have that $\mb 1\{R'_\tau(x_\tau)= r\}$ is by construction distributed as $b_{r, x_\tau}$ a Bernoulli random variable independent from $\mc F^{\pi}_\tau{'}(r-1)$. Hence, 
  \begin{align*}
   \mb 1\{R'_\tau(x_\tau)\geq r\} =1 \implies & \E{}{\exp(\lambda^\star(r, x_\tau) \mb 1\{R'_\tau(x_\tau)=r\}) ~|~\mc F^{\pi}_\tau{'}(r-1)} \\
  &= \E{}{\phi(r,x_\tau; \lambda^\star(r,x_\tau))~|~\mc F^{\pi}_\tau{'}(r-1)}
  \end{align*} and we have $$\tfrac{\E{}{\exp(\lambda^\star(r, x_\tau) \mb 1\{R'_\tau(x_\tau)=r\})~|~\mc F^{\pi}_\tau{'}(r-1)}}{\exp\Big(\mb 1\{R'_\tau(x_\tau)\geq r\} \phi(r, x_\tau; \lambda^\star(r, x_\tau))\Big)}=1\,.$$
  Hence,  $G_\tau'(r)$ is a martingale for the sequence $\mc F^{\pi}_\tau{'}( r)$ as stated in Equation (\ref{eq:martingale-sequence}).\hfill\Halmos
\endproof

\begin{lemma}\label{lem:dist:R_R'}
  Let  $R\sim P$ be a random variable with finite support $\setr=\{1, 2, \ldots, k\}$. Define distribution $ P'(r) = \frac{P(r)}{\sum_{r'\geq r}P(r')}$, $r\in [1, \dots, k-1]$, and  $k-1$ independent Bernoulli random variables, $b(r)\sim \text{Bernoulli}(P'(r))$, $r\in [1, \dots, k-1]$, where $\text{Bernoulli}(p)$ is a  Bernoulli random variable with a success rate of $p$. Let, the random variable $R'$ be 
\[ R' = \left\{ \begin{array}{ll}
         k & \mbox{if  $\sum_{r=1}^{k-1}b(r)=0$};\\
        \min\{r : b(r)=1\} & \mbox{otherwise}.\end{array} \right. \] 
Then, distribution of $R'$ is the same as that of $R$.
\end{lemma}
\proof{Proof of Lemma \ref{lem:dist:R_R'}.} The proof makes use of the following claim. 

\textbf{Claim.} For any $r\in [1, \dots, k]$, we have
\begin{align*}
  \prod_{r'< r} (1-P'(r'))   = 1-\sum_{r'<r} P(r')\qquad   \text{and} \qquad
  P(r)  = \prod_{r'<r} (1-P'(r')) \cdot P'(r)~ .
\end{align*}
This claim completes the proof because for any $r\in [k]$, we have
\[  \Prob{R'=r} = \prod_{r'<r} (1-P'(r')) \cdot P'(r) = (1-\sum_{r'< r}P(r'))P'(r) = P(r)= \Prob{R=r}\,,\]
where the first equation holds because 
$R'=r$ when $b(r')=0$, $r'<r$, and $b(r)=1$; the second equation follows from the claim;  the third equation holds because by definition,  $P'(r)= \frac{P(r)}{(1-\sum_{r'< r}P(r'))}$. 

\proof{Proof of claim.}   We will prove the claim using induction on $r$.
Consider  the base  case of $r=1$. We have $\prod_{r'< 1} (1-P'(r')) = 1-\sum_{r'<r} P(r')=1 $ and   $\prod_{r'< 1} (1-P'(r')) \cdot P'(r) = P'(1) = P(1)$.

  Assume for the sake of induction that  $\prod_{r'< r-1} (1-P'(r'))   = 1-\sum_{r'<r-1} P(r')$ and $P(r-1) = \prod_{r'<r-1} (1-P'(r')) \cdot P'(r-1)$ . First, we have that
  \begin{align*}
     \textstyle \prod_{r'< r} (1-P'(r')) 
    = & \textstyle \left(1-\sum_{r'<r-1} P(r')\right) (1-P'(r-1)) \\[0.5em]
    = & \textstyle 1-\sum_{r'<r-1} P(r') - \left(1-\sum_{r'<r-1} P(r')\right) P'(r-1) \\[0.5em]
    = & \textstyle 1-\sum_{r'<r-1} P(r') - P(r-1) = 1-\sum_{r'<r} P(r)\,,
  \end{align*}
  where the first equality follows from the induction hypothesis and the third equality follows  from definition of $P'(\cdot)$.
  We have then also have that
  \begin{align*}
      & \textstyle \prod_{r'<r} (1-P'(r')) P'(r) = \textstyle (1-\sum_{r'<r} P(r)) P'(r) 
    = \textstyle P(r)\,,
  \end{align*}
  where the first equality follows from the already established earlier relationship $\prod_{r'<r} (1-P'(r'))=(1-\sum_{r'<r} P(r))$ and the second equality follows again from definition of $P'$.\hfill\Halmos
\endproof
\endproof

\section{Proof of Statements in Section \ref{sec:proof:exploit}}
\label{sec:proofs:exploit}

\subsection{Proof of Lemma \ref{lem:explore_bad}}
By definition of the suboptimality gap of arms, when we are in the exploitation phase, i.e.,  $\mathcal E_t$ holds, and the bad events $\bigcup_{x\in x^\star(P_t)}\badexploit_t(x,\kappa)$ happen,  the regret can be bounded as 
\begin{align*}
  \sum_{\abs{X} < t \leq T} \mb 1\{ \mathcal E_t,~\bigcup_{x\in x^\star(P_t)}\badexploit_t(x,\kappa)\} \cdot \Delta(x, \optp)
  \leq & \sum_{\abs{X} <t \leq T} \sum_{x\in X}\mb 1\{ \mathcal E_t, ~x\in x^\star(P_t), ~\badexploit_t(x, \kappa)\} \cdot \Delta(x, \optp)\\
  = & \sum_{x\in X} \sum_{\abs{X} < t \leq T} \mb 1\{ \mathcal E_t, ~x\in x^\star(P_t), ~\badexploit_t(x, \kappa) \} \cdot \Delta(x, \optp) \\
  = & \sum_{x\in X}\sum_{1\leq s\leq T} \mb 1\{ \badexploit_{\tau(s, x)}(x, \kappa), ~\tau(s, x)\leq T\} \cdot \Delta(x, \optp)\,,
\end{align*}
where the first inequality holds because we pull one of empirically optimal arms, i.e., $x\in x^{\star}(P_t)$.
 Here, the stopping time $\tau(s, x) = \min\set{\tau\in(\abs{X},\dots, T]}{\sum_{\abs{X}< t\leq \tau} \mb 1\{ \mathcal E_t, ~x\in x^\star(P_t)\}\geq s}$ indicates the first time in the exploitation phase that we deemed a certain arm $x$ empirically optimal at least $s$ times.
When such an event never happens,  we set $\tau(s, x)=T+1$.
The expected regret caused during the event of interest is hence bounded by
\begin{align}\begin{split}\label{eq:bound}
 & \E{}{\textstyle\sum_{\abs{X}< t\leq T} \mb 1\{\mathcal E_t, ~\bigcup_{x\in x^\star(P_t)}\badexploit_t(x,\kappa)\}\cdot\Delta(x_t, \optp)}\\
 &\quad \leq \sum_{x\in X, \, r\in\setr} \sum_{1\leq s\leq T} \Prob{ \abs{P_{\tau(s, x)}(r, x)- P(r, x)}>\kappa} \cdot \Delta(x, \optp). \\
 &\quad \leq \sum_{1\leq s\leq T} \sum_{x\in X, \, r\in\setr}  \Prob{ \abs{P_{\tau(s, x)}(r, x)- P(r, x)}>\kappa} \cdot \Delta(x, \optp). 
  \end{split}
\end{align}
Now consider a stopping time $\tau(s, x)$ for some $1\leq s\leq T$.
We know that by construction of the stopping time $\tau(s, x)$, arm $x$ was empirically optimal at least $s-1$ times prior to round $t$.  
We will argue that $N_{\tau(s, x)}(x)> \tfrac{s}{(4\abs{X})}$ for $s\geq 8 \abs{X}$.
For the sake of contradiction assume that $N_{\tau(s, x)}(x)\leq \tfrac{s}{(4\abs{X})}$. This implies that for at least $E\leq  s-1-\tfrac{s}{(4\abs{X})}$ rounds $\{ t_i \}$ prior to time $\tau(s, x)$ we had that there was an empirically optimal arm $x'\neq x$ such that $N_{t_i}(x')\leq N_{t_i}(x) \leq N_{\tau(s, x)}(x)\leq \tfrac{s}{(4\abs{X})}$. As in each of these rounds $N_{t_i}(x')$ is incremented by one there can only have been at most $E \leq (\abs{X}-1)\lceil \tfrac{s}{(4\abs{X})} \rceil\leq \tfrac{s}{4}+\abs{X}$ of such rounds. Hence, we must have
\begin{align*}
  \tfrac{s}{4}+\abs{X} \geq & E\geq s-1-\tfrac{s}{(4\abs{X})} 
\end{align*}
which is contradiction with $\abs{X}\geq 1$ and $s\geq 8 \abs{X}$.
Considering this and per Equation \eqref{eq:bound}, we further upper bound our regret as follows:
\begin{align*}
  & \E{}{\textstyle\sum_{\abs{X}< t\leq T} \mb 1\{\mathcal E_t, ~\bigcup_{x\in x^\star(P_t)}\badexploit_t(x,\kappa)\}\cdot\Delta(x_t, \optp) } \\
  \leq & 8\abs{X}+\abs{X}\abs{\setr} \cdot \sum_{1\leq s\leq T}  \Prob{ \abs{P_{\tau(s, x)}(r, x)- P(r, x)}>\kappa, ~N_{\tau(s, x)}(x)> \tfrac{s}{(4\abs{X}) }}
\end{align*}
Applying the concentration bound in Lemma \ref{lemma:stopping-time-inequality} gives
\begin{align*}
  &  \abs{X}\abs{\setr} \cdot \sum_{1\leq s\leq T}  \Prob{ \abs{P_{\tau(s, x)}(r, x)- P(r, x)}>\kappa, ~N_{\tau(s, x)}(x)> \tfrac{s}{(4\abs{X}) }} \\
 \leq & \abs{X} \abs{\setr} \cdot \sum_{1\leq s\leq T} 2 \exp(-s  \kappa^2/(2\abs{X}))
\end{align*}
from which the claimed result follows. \hfill\Halmos

\subsection{Proof of Lemma \ref{lem:event_a_t}}

Too show the result, we consider a reward distribution $\bar Q$ such that $\bar Q(x^\star(P_t)) = P_{t}(x^\star(P_t))$ and $\bar Q(\tilde X(P_{t}))=P(\tilde X(P_{t}))$ as visualized in Figure \ref{fig:proof-construction}. 
Observe that the reward distribution $\bar Q$ agrees with the actual reward distribution $\optp$ at all the arms except for the empirically optimal arms. 
Now, considering our assumption that $\optp\in \interior (\mathcal P)$ with a unique optimal arm, we can choose $\kappa$ sufficiently small so that $\norm{\bar Q(x)- \optp(x)}_\infty<\kappa$ for all $x\in x^\star(P_t)$ implies (i) $\bar Q\in \interior(\mathcal P)$ and (ii) $x^\star(\optp)=x^\star(\bar Q)$ given that $\bar Q(x) =\optp(x)$ for any $x\in \tilde X(P_t)$; see Lemma \ref{lemm:stability-optimal-arm}. %
That is, if two distributions $\optp$ and $\bar Q$ are equal everywhere except at arms $x^\star(P_t)$ and $\norm{\bar Q(x)- \optp(x)}_\infty<\kappa$ for all $x\in x^\star(P_t)$, then they are sufficiently close to have the same optimal arm and both be in the interior of $\mathcal P$.
Given this, it is easy to see that  distribution $\bar Q$  is indeed deceitful, i.e.,  $\bar Q \in \prob(x^\star(\optp), P_{t})$, and that $x^\star(\optp)\in \xd(P_t)$ as $\prob(x^\star(\optp), P_{t})$ is non-empty.
Note that by Equation \eqref{eq:problematic-distributions},  $\bar Q \in \prob(x', P_{t})$ if (i) $\bar Q\in \mathcal P$, (ii) $\bar Q(x^\star(P_t))=P_{t}(x^\star(P_t))$, and (ii) $\textstyle\sum_{r \in \setr} r \bar Q(r, x^*(\optp))  > \sum_{r \in \setr} r \bar Q(r, x)$ for some $x\in x^\star(P_t)$.
The two first conditions hold by construction and the final condition holds because the optimal arm under $\bar Q$ is $x^*(\optp)\not\in x^\star(P_t)$; again see Figure \ref{fig:proof-construction}. 
The constructed distribution $\bar Q\in \prob(x^*(P), P_{t})$ guarantees that \eqref{eq:a_2} occurred as
\begin{align*}
  \textstyle\sum_{x\in X} N_{t}(x) ~I( P_{t}(x), P(x)) &\geq \sum_{x\in \tilde X(P_t)} N_{t}(x) ~I( P_{t}(x), P(x))\\
  &\geq  \sum_{x\in X} N_{t}(x) ~I( P_{t}(x), \bar Q(x)) \geq (1+\epsilon)\log(t)\,,
\end{align*}
where the first inequality follows from the construction of $\bar Q$ (i.e., $\bar Q(x^\star(P_t)) = P_{t}(x^\star(P_t))$) and the second inequality follows from Equation \eqref{eq:a_1} and the fact that  $\bar Q\in \prob(x^*(P), P_{t})$. \hfill\Halmos

\begin{figure}[ht]
  \centering
 {\includegraphics[height=4.5cm]{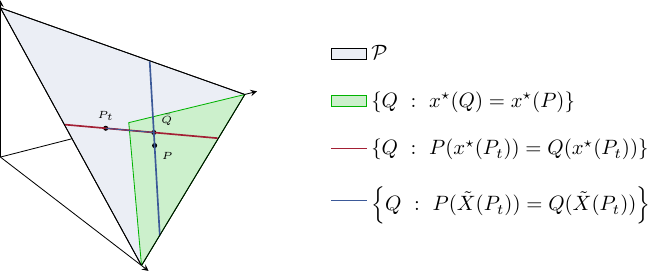}}
  {  \caption{Construction of the reward distribution $\bar Q$. By assumption $P$ is in the interior of both $\mathcal P$ as well as $\set{Q\in \mathcal P}{x^\star(Q)=x^\star(P)}$ as its optimal arm is unique.} \label{fig:proof-construction}}
\end{figure}

\subsection{Proof of Lemma \ref{lemm:stability-optimal-arm}}
We first show that  $\mathcal P_{u}(x^\star_0)$ is an open set. By definition, for a given $x^\star_0$, evidently we have that $$\mathcal P_{u}(x^\star_0)=\set{Q}{0<\sum_{r\in \setr} Q(r, x^\star_0) - \sum_{r\in \setr} Q(r, x) ~~\forall x\neq x^\star_0\in X}$$ is open as the functions $\sum_{r\in \setr} Q(r, x^\star) - \sum_{r\in \setr} Q(r, x)$ are linear and hence continuous in $Q$.

Next, we show the second result of the lemma. Using the result we just showed,  it follows that there exists a neighborhood $\mathcal N$ such  that
  \[
  Q\in \mathcal N\implies x^\star(Q)=x^\star(P)=:x^\star_0.
  \]
  whenever we have $P \in\mathcal P_u(x_0^\star)$.
  The maximum reward of arm $x'$ for any $P\in \mathcal N\bigcap \mathcal P$ in this neighborhood hence simplifies to
  \begin{equation}
    \label{eq:maximal-reward-special}
    \begin{array}{rl}
      \rew_{\max}(x', Q) \defn \max_{Q'} & \sum_{r \in \setr} r Q'(r, x') \\
      \st & Q' \in \mathcal P, \\
                                    & Q'(x^\star_0)=Q(x^\star_0).
    \end{array}
  \end{equation}  
  We wish to prove that $\rew_{\max}(x',Q)$ is continuous at $P$ by invoking Proposition \ref{prop:continuity-function-general} to its maximization characterization in stated in Equation \eqref{eq:maximal-reward-special}. In order to do so we verify that all required conditions are met.
  
  We first remark that our feasible set mapping is indeed lower semi-continuous. 
  
  \begin{lemma}
  \label{lemm:lsc:feasible-mapping}
    The mapping $M(Q)\defn \set{Q'\in \mc P}{Q'(x_0^\star)=Q(x_0^\star)}$ is lower semi-continuous at any $P\in\interior(\mc P)$.
  \end{lemma}

  Evidently, the set $M(Q)$ is always convex as it is an affine set for any $Q\in \mathcal P$. The mapping $M(Q)$ is closed at $P$ as its graph
  $\set{(Q, Q')}{Q\in \mathcal N', ~Q'\in \mathcal P,~Q'(x^\star_0)=Q(x^\star_0)}$ is closed for all closed neighborhoods $\mathcal N'$ containing $P$.
  
  The objective function in the maximization problem stated in Equation \eqref{eq:maximal-reward-special} is continuous at $(Q', Q)$. The set of maximizers is always non-empty and bounded as its feasible set at $Q= P$, i.e., $M(P)$, is furthermore compact and non-empty. As the objective function of the maximization problem \eqref{eq:maximal-reward-special} is also convex in $Q'$, Proposition \ref{prop:continuity-function-general} implies the continuity of its maximum $\rew_{\max}(x', Q)$ at $P$. \hfill\Halmos
  
 \proof{Proof of Lemma \ref{lemm:lsc:feasible-mapping}.}
    The mapping is clearly nonempty as $Q\in M(Q)$. We can prove the lower-semicontinuity of this mapping at $P$ directly from its definition. Consider indeed an arbitrary open set $V$ such that $M(P)\cap V\neq \emptyset$. It remains to show that the set $\set{Q}{M(Q)\cap V\neq \emptyset}$ contains an open neighborhood around $P$. From Section \ref{sec:information-topology}, we know that without loss of generality we may assume that $$V=\set{Q'}{\textstyle\sum_{r\in \setr}Q'(r)=1, ~\norm{Q'-\bar Q}_\infty< \epsilon}$$ for some $\epsilon>0$ and $\bar Q\in \mc P_\Omega$. Hence, there exists $U$ with
      \[
        U \in \mc P,\quad \norm{U(x)- \bar Q(x)}_\infty<\epsilon ~\forall x\in X \quad{\rm{and}}\quad U(x^\star_0)=P(x^\star_0).
      \]
      In fact as $P\in M(P)$ and $P\in \interior(\mc P)$ we may assume without loss of generality that $U\in \interior(\mc P)$ as well. Indeed, assume this is not the case then we may simply consider the perturbation $(1+\theta) U+\theta P$ instead for a sufficiently small $\theta>0$ exploiting the convexity of $M(P)$ and the fact that $V$ is an open set.
      Define now the parametric distribution $U_Q(x)=U(x)$ for $x\in \tilde X_0$ and $U_Q(x_0^\star)=Q(x_0^\star)$ for any $Q$. Clearly, $U_P=U$. By construction, we hence have $P\in \set{Q}{U_Q \in M(Q) \cap V}$. Furthermore,
      \begin{align*}
         & \set{Q}{U_Q \in M(Q) \cap V} \\
        = & \set{Q}{U_Q \in \mc P} \cap\set{Q}{\norm{ Q(x^\star_0)- \bar Q(x^\star_0)}_\infty<\epsilon}
      \end{align*}
      The first set in the previous intersection contains an open neighborhood around $P$ as $U_P=U \in \interior(\mc P)$ and $U_Q$ is a linear and hence continuous function in $Q$. The second set is an open set as $\norm{Q(x^\star_0)- \bar Q(x^\star_0)}_\infty$ is a continuous function in $Q$. As clearly we have $\set{Q}{U_Q \in M(Q) \cap V}\subseteq \set{Q}{M(Q)\cap V\neq \emptyset}$ the claim follows immediately.\hfill\Halmos
  \endproof

\section{Proof of Statements in Section \ref{sec:proof:explore}}
\label{sec:proofs:explore}  

\subsection{Proof of Lemma \ref{lem:lowerbound_N}}
{\color{black}
  For the sake of contradiction,  assume that this is false and hence $\min_{x\in X} ~N_t(x) < \epsilon \lceil \tfrac{(s_t/4)}{(1+\log(1+ s_t ))} \rceil$. When $\mathcal X_t$ occurs then $\tau(s_t)= t$. We now argue that there must exist  $\lceil{3s_t}/4\rceil$ rounds in the first $t$ rounds  in which  $\min_{x\in X} N_t(x)\leq \epsilon \tfrac{s_{\tauexplore(i)}}{(1+\log(1+s_{\tauexplore(i)}))}$.
We have that $\min_{x\in X} N_{\tauexplore(i)}(x) \le \epsilon \tfrac{s_{\tauexplore(i)}}{(1+\log(1+s_{\tauexplore(i)}))}$ for any $i\in \{\lceil s_t/4 \rceil, \ldots, s_t\}$.
This is because for such rounds $i\in \{\lceil (s_t/4) \rceil, \ldots, s_t\}$ we have (i) $\min_{x\in X} N_{\tauexplore(i)}(x) \le \min_{x\in X} N_{t}(x)$, (ii) $\epsilon \tfrac{s_{\tauexplore(i)}}{(1+\log(1+s_{\tauexplore(i)}))} = \epsilon \tfrac{i}{(1+\log(1+i))} \ge \epsilon \tfrac{i}{(1+\log(1+s_t))} \ge\epsilon \lceil \tfrac{(s_t/4)}{(1+\log(1+ s_t ))} \rceil$, and (iii) by our assumption, $\min_{x\in X} ~N_t(x) < \epsilon \lceil \tfrac{(s_t/4)}{(1+\log(1+ s_t ))}\rceil $.
So far we established that there  exist at least $\lfloor{ 3s_t}/4\rfloor \leq s_t-\lceil s_t/4\rceil$ rounds before time $t$ in which  $\min_{x\in X} N_t(x)\leq \epsilon \tfrac{s_{\tauexplore(i)}}{(1+\log(1+s_{\tauexplore(i)}))}$. 
After every $\abs{X}$ of such rounds $\min_{x\in X} N_t(x)$ is incremented at least by one.
Hence, $$\min_{x\in X} N_t(x)\geq \lfloor \tfrac{\lfloor{3s_t}/4\rfloor}{\abs{X}}\rfloor\,.$$
As we choose $\epsilon < \tfrac{1}{\abs{X}}$,  we have that $\min_{x\in X} N_t(x)\geq \lfloor \epsilon\lfloor{3s_t}/4\rfloor\rfloor$. This is in contradiction with $\min_{x\in X} ~N_t(x) < \epsilon \lceil \tfrac{(s_t/4)}{(1+\log(1+ s_t ))} \rceil\leq \epsilon \lceil (s_t/4) \rceil$ whenever $s_t> \tfrac{2}{\epsilon}$.  \hfill\Halmos
}

\subsection{Proof of Lemma \ref{lem:expolre:bad:arm}}
We first  show that
when we are in the exploration phase, there exists an empirically suboptimal arm $x'\in \tilde X(P_t)$ such that $N_t(x') \leq (1+\epsilon)\eta_t(x') \log(t) \leq  (1+\epsilon)\norm{\eta_t}_{\infty}\log(T)$ where the target exploration rate $\eta_t=\mathrm{SU}(P_t, \eta'_t~;~\mu_t, \epsilon)$ is determined with the help of the shallow update Algorithm \ref{alg:shallow}. Here,
 $\eta'_t$ is the reference logarithmic rate in round $t$.
Assume that this was not the case, then $N_t(x) \geq  (1+\epsilon)\eta_t(x) \log(t)$ for all $x\in X$. 
All empirically deceitful arms $x\in \xd(P_t)$ would
\begin{align*}
    \overline{\dual}(\tfrac{N_t}{\log(t)},x,    P_t~;~  \mu_{t}(x)) 
  ~&\geq~  \overline{\dual}\left((1+\epsilon) \eta_t, x,  P_{t}~;~\mu_{t}(x)\right) \\
  ~&=~ (1+\epsilon)\cdot \overline{\dual}\left( \eta_t,  x,  P_{t}~;~ \mu_{t}(x) \right)
  \ge  1+\epsilon
\end{align*}
pass the sufficient information test; a contradiction. Note that the first inequality follows from the fact that the dual-test function $\overline{\dual}(\eta,x,    P_t;~  \mu_{t}(x'))$ is an increasing function of the target rates $\eta$; see Lemma \ref{lemm:res-function-properties}, stated below.
The subsequent equality follows from the fact that the dual-test function $\overline{\dual}(\eta,x,    P_t;~  \mu_{t}(x'))$ is positively homogeneous in the target rates $\eta$; see again Lemma \ref{lemm:res-function-properties}.
The final inequality holds as $\eta_{t}=\textsc{SU}(  P_{t}, \eta'_t~;~\mu_{t}, \epsilon)$ and consequently $\overline{\dual}( \eta_t,  x,  P_{t}~;~  \mu_{t}(x) ) \ge 1$ for all empirically deceitful arms $x\in \xd(P_t)$; see the shallow update Algorithm \ref{alg:shallow}.

So far we established that when we are in the exploration phase, there exists an empirically suboptimal arm $x'\in \tilde X(P_t)$ such that $N_t(x') \leq (1+\epsilon)\eta_t(x') \log(t) \leq  (1+\epsilon)\norm{\eta_t)}_{\infty}\log(T)$. 
Next, we use this to show the desired result. 
Recall that whenever we are in the exploration round,  arm $x_t$ we play is either $\underline x_t$, $x_t^\star$ or $\bar x_t$.
If $x_t=\underbar{x}_t =\arg\min_{x\in X} N_t(x)$, then our earlier observation implies that   $N_t(x_t)\leq N_t(x') \leq (1+\epsilon) \norm{\eta_t}_{\infty}\log(T)$.
We now assume that $x_t=\bar{x}_t$, where $\bar{x}_t= \arg\min_{x\in \tilde X(  P_{t})} \tfrac{N_{t}(x)}{\eta_{t}(x)}$. Again, by our earlier observation, there must exist at least one empirically suboptimal arm $x'$ such that $\tfrac{N_t(x')}{\eta_t(x')} \leq (1+\epsilon) \log(t)$. This implies that 
$N_t(x_t) \leq (1+\epsilon) \eta_t(x_t) \log(t) \leq (1+\epsilon) \norm{\eta_t}_{\infty}\log(T)$ as well.
{\color{black} By definition $x_t=x_t^\star$ only if $N_t(x_t)\leq N_t(\bar x_t)\leq (1+\epsilon) \norm{\eta_t}_{\infty}\log(T)$.}
Thus, in all cases, we have the inequality
\begin{equation*}
  N_t(x_t) \leq (1+\epsilon)\norm{\eta_t}_\infty \log(T)\,,
\end{equation*}
as desired.
\hfill\Halmos

\begin{lemma}[Dual-test Function]
  \label{lemm:res-function-properties}
  The dual-test function
  \(
    \overline{\dual}\left(\eta,  x, P~; ~\mu\right)
  \)
  is non-decreasing and positively homogeneous in the exploration rate $\eta$. That is, we have
  \[
    \begin{array}{l@{~}r@{~}l@{\hspace{3em}}r}
      \overline{\dual}(\eta_1,  x', P~; ~\mu)&\leq& \overline{\dual}(\eta_2,  x', P~; ~\mu)& \forall 0\leq \eta_1\leq \eta_2,\\
      \overline{\dual}(a\cdot \eta,  x', P~; ~\mu)&=& a \cdot \overline{\dual}(\eta,  x', P~; ~\mu) & \forall 0\leq a,
    \end{array}
  \]
  for any dual feasible $\mu \in \real_+\times\real\times \mathcal K^\star$. 
\end{lemma}
\proof{Proof of Lemma \ref{lemm:res-function-properties}.}
  The dual function $\dual\left(\eta,  x', P~; ~\mu\right)$ in $\eta$ as defined in Equation \eqref{eq:dual} is a non-decreasing function of $\eta$ for all arms $x'\in X$, reward distributions $P \in \mathcal P_\Omega$, and dual variables $\mu\in \real_+\times\real\times\mathcal K^\star$; see Lemma \ref{lemm:dual-function-properties}. Evidently, it follows that $\max_{\rho\geq 0} \dual\left(\eta,  x', P~; ~\rho\cdot \mu\right)$ is a non-decreasing in $\eta$ for all $x'\in X$ and $P \in \mathcal P_\Omega$ too. Positive homogeneity follows because
  \begin{align*}
    \overline{\dual}(a\cdot \eta,  x', P~; ~\mu)=& \max_{\rho\geq 0} ~\dual\left(a \cdot \eta,  x', P~; ~\rho\cdot \mu\right)
    =  \max_{\rho\geq 0} ~\dual\left(a \cdot \eta,  x', P~; ~\rho\cdot a \cdot \mu\right)\\
    = & \max_{\rho\geq 0} ~a \cdot \dual\left(\eta,  x', P~; ~\rho \cdot \mu\right)
        = a \cdot \overline{\dual}(\eta,  x', P~; ~\mu)\,,
  \end{align*}
  where the third equality again follows from  Lemma \ref{lemm:dual-function-properties}.\hfill\Halmos
\endproof

\begin{lemma}[Dual Function]
  \label{lemm:dual-function-properties}
 For any $x'\in \tilde X$, the dual function
  \(
    \dual(\eta, x', P~; ~\mu)
  \)
  is nondecreasing in the exploration rate $\eta$ and positively homogeneous in $(\eta, \mu)$. We have indeed for any dual feasible $\mu\in \real_+\times\real\times \mathcal K^\star$ that
  \[
    \begin{array}{l@{~}r@{~}l@{\hspace{3em}}l}
      \dual(\eta_1, x', \optp~;\mu)&\leq& \dual(\eta_2, x', \optp~;~\mu)& \forall 0\leq \eta_1\leq \eta_2,\\
      \dual(a \cdot \eta, x', \optp~;a\cdot \mu)&=& a \cdot \dual(a \cdot \eta, x', \optp~;a\cdot \mu) & \forall 0\leq a.
    \end{array}
  \]
\end{lemma}
\proof{Proof of Lemma \ref{lemm:dual-function-properties}.}
  Both properties can be proved simply by verification on the explicit representation
  \begin{align*}
      \dual(\eta,x', Q~;~\mu) = &  \textstyle\sum_{x\in \tilde X(Q),\, r\in \setr} \eta(x) \log\left( \frac{\eta (x)- \lambda(r, x)-\beta-\alpha \cdot r \mb 1(x=x')}{\eta(x)}\right) Q(r,x) \\
      &\quad - \textstyle \sum_{x\in x^\star(Q),r\in \setr} \lambda(r, x)Q(r, x)  \\
      & \quad  + \alpha \rew^{\star}(Q) + \beta|\tilde X(Q)| + \chi_{-\infty}(\eta (x)\geq \lambda(r, x)+\beta+\alpha \cdot r \mb 1(x=x')).     
  \end{align*}
  The term $\chi_{-\infty}(\eta (x)\geq \lambda(r, x)+\beta+\alpha \cdot r \mb 1(x=x'))$ is clearly nondecreasing in $\eta$. The nonlinear terms $\eta(x) \log\left( \tfrac{(\eta (x)- \lambda(r, x)-\beta-\alpha \cdot r \mb 1(x=x'))}{\eta(x)}\right)Q(r, x)$ can easily be observed to be nondecreasing in $\eta$ by noting that the functions of the form $u\log(\tfrac{(u - v)}{u})$ are increasing in $u\in\real$ for any given $u\in \real$. This proofs the first part of the lemma. The second part is trivial and follows from
  \begin{align*}
    & \dual(a \eta,x', Q~;~a \mu)\\
     = &  \textstyle\sum_{x\in \tilde X(Q),\, r\in \setr} a \eta(x) \log\left( \frac{a \eta (x)- a \lambda(r, x)-a \beta-a \alpha \cdot r \mb 1(x=x')}{a \eta(x)}\right) Q(r,x)  - \textstyle \sum_{x\in x^\star(Q),r\in \setr} a \lambda(r, x)Q(r, x)  \\
                                  & \quad  + a \alpha \rew^{\star}(Q) + a \beta|\tilde X(Q)| + \chi_{-\infty}(a \eta (x)\geq a \lambda(r, x)+a \beta+a \alpha \cdot r \mb 1(x=x')) \\
    = &  \textstyle\sum_{x\in \tilde X(Q),\, r\in \setr} a \eta(x) \log\left( \frac{ \eta (x)-  \lambda(r, x)- \beta- \alpha \cdot r \mb 1(x=x')}{ \eta(x)}\right) Q(r,x)  - \textstyle \sum_{x\in x^\star(Q),r\in \setr}  a \lambda(r, x)Q(r, x)  \\
                                  & \quad  + a \alpha \rew^{\star}(Q) + a \beta|\tilde X(Q)| +a  \chi_{-\infty}( \eta (x)\geq  \lambda(r, x)+ \beta+ \alpha \cdot r \mb 1(x=x')) \\
    = & a \cdot \dual( \eta,x', Q~;~ \mu).
  \end{align*}\hfill\Halmos
\endproof

\section{Continuity Results}\label{sec:Continuity}
We suggest the readers who are not familiar with  Slater’s Constraint Qualification Condition to review Section \ref{app:continuity} before reading the proofs presented in this section.   
\subsection{Proof Proposition \ref{prop:continuity-projection}}
\label{proof:projection-continuity}
The proof has two parts. In the first part, we show that the  distance function $\dis(\eta, x', P)$, defined in Equation \eqref{eq:inner-optimization}, is finite if and only if $x'$ is deceitful, i.e., $\rew^{\star}(P)<\rew_{\max}(x', P)$. In the second part, we show that  problem \eqref{eq:inner-optimization} and problem  \eqref{eq:dist-function-continuous} are identical when arm $x'$ is deceitful.

\textbf{First part.} By  definition of the  distance function $\dis$ in Equation \eqref{eq:inner-optimization},   a necessary requirement for the distance function $\dis( \eta, x', P)$ to be finite  is that $\rew^{\star}(P)<\rew_{\max}(x', P)$. To see why, note that the set of deceitful distributions  $\prob(x', P)$ is empty and hence $\dis(\eta, x', P)$ infinite if $\rew^{\star}(P)\geq \rew_{\max}(x', P)$.
We will show that this condition, i.e., $\rew^{\star}(P)<\rew_{\max}(x', P)$, is in fact sufficient too.  Therefore, we have  that the effective domain of the information distance function is given by 
 $\set{P\in \mathcal P}{\dis(\eta, x', P)<\infty}=\tset{P}{\rew^{\star}(P)<\rew_{\max}(x', P)}$, as desired. 
  
 When $\rew^{\star}(P)<\rew_{\max}(x', P)$, we will argue that the distance function is finite, i.e.,
 \[
   \dis( \eta, x', P) \leq \textstyle\sum_{x\in \tilde X(P)} \eta(x) I(P(x), \bar Q(x)) < \infty
 \]
 by constructing a distribution $\bar Q\in \prob(x', P)$ which satisfies $\bar Q\ll P$; see Lemma \ref{prop:relative-entropy}. 
 By definition of $\rew_{\max}(x', P)$ stated in Equation \eqref{eq:maximum-reward}, there exists a maximizing distribution  $\hat Q \in \mathcal P$ such that $\hat Q(x^\star(P))=P(x^\star(P))$ and $$\textstyle\rew^{\star}(P) < \sum_{r\in \setr} r \hat Q(r, x')$$ which does ensure that $\hat Q\in \prob(x', P)$. Unfortunately, we are not ensured that $\hat Q\ll P$. However, we may simple perturb $\hat Q$ as $\bar Q = (1-\theta) \hat Q + \theta P$ for some $\theta \in (0, 1)$ and observe that from convexity we have $\bar Q\in \mathcal P$ and $\bar Q(x^\star(P))=P(x^\star(P))$. Furthermore, for  sufficiently small perturbations $\theta$ we have from continuity of the expected reward of arm $x'$ in the reward distribution that
 $\rew^{\star}(P) < \sum_{r\in \setr} r \bar Q(r, x')$ as well. Thus, $\bar Q\in \prob(x', P)$ and $\bar Q\ll P$.

  \textbf{Second part.} Here, we show that  problem \eqref{eq:inner-optimization} and problem  \eqref{eq:dist-function-continuous} are identical when arm $x'$ is deceitful.
  It is easy to see  that the optimal value of problem \eqref{eq:inner-optimization}, i.e.,  $\dis( \eta, x', P)$, is greater than or equal to that of problem \eqref{eq:dist-function-continuous} as the latter is characterized as a minimization problem of the same objective function over a larger feasible set. Next, we show that the optimal value of problem \eqref{eq:inner-optimization} is less than or equal to that of problem \eqref{eq:dist-function-continuous}. 
  Let $Q^\star$ be the optimal solution to problem \eqref{eq:dist-function-continuous}. Note that an optimal solution of this problem exists because its objective function  is lower semicontinuous  per Lemma \ref{prop:relative-entropy} and its feasible set is a compact set; see for instance \citet[Prop. 3.2.1]{bertsekas2009convex}. 
   Now, consider the perturbed reward distributions  $Q'_\theta = \theta \cdot \bar Q  + (1-\theta)\cdot Q^\star  \in \mathcal P$ parameterized in $\theta\in [0, 1]$. Again we have that $Q'_\theta\in \mathcal P$ as the set $\mathcal P$ is convex and contains both reward distribution $Q^\star$ and $\bar Q$. Furthermore, we have that $\sum_{r \in \setr} r Q'_\theta(r, x') > \rew^{\star}(P)$ for all $\theta\in (0, 1]$ as we have both $\sum_{r \in \setr} r Q^\star(r, x') \geq \rew^{\star}(P)$ and $\sum_{r \in \setr} r \bar Q(r, x') > \rew^{\star}(P)$. Hence, the reward distributions $Q'_\theta$ are feasible in the minimization problem \eqref{eq:inner-optimization} and consequently for any $\theta\in (0, 1]$ we have
  \begin{align*}
     \eqref{eq:inner-optimization} &\textstyle\leq \textstyle\sum_{x\in \tilde X(P)} \eta(x) I(P(x), Q'_\theta(x))\\
     &\textstyle\leq \theta\cdot\sum_{x\in \tilde X(P)} \eta(x)  I(P(x), \bar Q(x)) + (1-\theta)\cdot \sum_{x\in \tilde X(P)} \eta(x) I(P(x), Q^\star(x) ).
   \end{align*}
   Considering the fact that the term $\textstyle\sum_{x\in \tilde X(P)} \eta(x) I(P(x), \bar Q(x))$ is finite, and by 
  letting $\theta$ go to zero, we have \begin{align*} &\lim_{\theta\downarrow 0} \textstyle\theta\sum_{x\in \tilde X(P)} \eta(x)  I(P(x), \bar Q(x)) + (1-\theta) \sum_{x\in \tilde X(P)} \eta(x) I(P(x), Q^\star(x) ) \\
  &\quad= \textstyle \sum_{x\in \tilde X(P)} \eta(x) \cdot I(P(x), Q^\star(x) )\,.\end{align*}
   The proof is then completed by observing that the left hand side is greater than  or equal to the optimal value of problem \eqref{eq:inner-optimization}, and the right hand side is equal to that of problem \eqref{eq:dist-function-continuous}. \hfill\Halmos
   
   \subsection{Proof of Proposition \ref{prop:continuity-regret-lower-bound}}

Here, we first show that the lower regret bound function $C(Q)$ is continuous at any reward distribution $P\in \mathcal P'$.
We show that the lower regret bound function $C(Q)$ is continuous  with the help of  Lemmas \ref{lemm:stability-optimal-arm} and Lemma \ref{lemm:continuity-distance-function}, which are found at the end of this section and two classical continuity results (Propositions \ref{prop:continuity-mapping-general} and \ref{prop:continuity-function-general}) regarding the minimizers of parametric optimization problems restated in Section \ref{app:duality} for completeness. 

From Lemma \ref{lemm:stability-optimal-arm}, we can find a neighborhood $\mathcal N\subseteq \mathcal P'$ around any $P\in \mathcal P'$ such that
\[
  Q\in \mathcal N \implies x^\star(Q)=x^\star(P):=x^\star_0, ~\tilde X(Q)=\tilde X(P):=\tilde X_0 {\mathrm{~and~}} \xd(Q)= \xd(P):=\xd^0
  \,.\]
That is, all reward distributions $Q$ in the neighborhood $\mathcal N$ share their optimal, nondeceitful and deceitful arms with the reward distribution $P$. 
  { From Proposition \ref{prop:continuity-projection}, for any $Q \in \mathcal N$,  the regret lower bound function is given as
  \begin{equation}
    \label{eq:special-lower-regret}
    \begin{array}{rl}
      C(Q):= {\displaystyle\inf_{\eta\geq 0}} & \textstyle\sum_{x\in \tilde X_0} \eta(x) \Delta(x, Q)  \\[0.4em]
      \st       & 1 \leq  \dis(\eta, x, Q)  \quad \forall x\in \xd^0 \\[0.4em]
                                              & \eta(x)=0 \quad \forall x\in x^\star_0
    \end{array}
  \end{equation}
  where the value of $\eta(x^\star_0)$ for the optimal arm is in fact arbitrary and can be set to zero.}
  We would like to establish that the lower regret bound function $C(Q)$ is continuous at $P$ by applying Proposition \ref{prop:continuity-function-general} to the minimization problem \eqref{eq:special-lower-regret}. To apply this proposition, we need to show that its feasible set mapping is a lower semi-continuous mapping. One can hope to establish lower semi-continuity by verifying the conditions of  Proposition \ref{prop:continuity-mapping-general}. One of these conditions demands 
  the inequality constraint functions $\dis(\eta,x, Q)$ to be upper semi-continuous at points $(\eta_0, P)$, which may not be the case when $\eta_0$ is not strictly positive. {To circumvent this challenge, we introduce the following  family of perturbed regret lower bound problem 
  \begin{equation}
    \label{eq:special-lower-regret-perturbed}
    \begin{array}{rl}
      C_\delta(Q):= {\displaystyle\inf} & \textstyle\sum_{x\in \tilde X_0} \eta(x) \Delta(x, Q) \\[0.4em]
      \st       & 1 \leq  \dis(\eta, x, Q)  \quad \forall x\in \xd^0, \\
                                        & \eta(x)\geq \delta \quad \forall x\in X, \\
                                        & \eta(x)= \delta \quad \forall x\in x^\star_0
    \end{array}
  \end{equation}
  and show that continuity of $C_{\delta}(Q)$ at $Q= P$ implies  continuity of the regret lower bound at $P$. We then proceed to show $C_{\delta}(Q)$ is indeed continuous  at $P$.}

   \textbf{Continuity of the perturbed regret lower bound implies that of regret lower bound.} 
  By definition of the perturbed regret lower bound, we have the following inequalities  \[
  C_\delta(Q)-\delta |\tilde X_0| \leq C_\delta(Q)-\delta \textstyle\sum_{x\in \tilde X_0}\Delta(x, Q) \leq C(Q) \leq C_\delta(Q)
  \]
  for any $\delta>0$ uniformly in $Q\in \mathcal N$. The first inequality holds because $\Delta (x, Q)\in [0,1]$ for any arm $x$ and reward distribution $Q$. {The second inequality follows from the fact that we have
  \begin{equation}
    \begin{array}{rl}
      C_\delta(Q)= {\displaystyle\inf_{\eta}} & \textstyle\sum_{x\in \tilde X_0} (\eta(x)-\delta) \Delta(x, Q) +  \delta \sum_{x\in \tilde X_0} \Delta(x, Q) \\[0.4em]
      \st       & 1 \leq  \dis(\eta, x, Q)  \quad \forall x\in \xd^0\\
                                       & \eta(x)-\delta \geq 0 \quad \forall x\in X, \\
                                              & \eta(x)-\delta= 0 \quad \forall x\in x^\star_0 \\[1em]
      \leq {\displaystyle\inf_{\eta}} & \textstyle\sum_{x\in \tilde X_0} (\eta(x)-\delta) \Delta(x, Q) +  \delta \sum_{x\in \tilde X_0} \Delta(x, Q) \\[0.4em]
      \st       & 1 \leq  \dis(\eta-\delta\cdot \textbf{1}, x, Q)  \quad \forall x\in \xd^0\\
                                       & \eta(x)-\delta \geq 0 \quad \forall x\in X, \\
                                              & \eta(x)-\delta= 0 \quad \forall x\in x^\star_0 \\[1em]
      = & C(Q) + \delta \sum_{x\in \tilde X_0} \Delta(x, Q).
    \end{array}
  \end{equation}
}

Hence, we have that $C(Q)=\lim_{\delta\downarrow 0} C_\delta(Q)$ uniformly in $Q\in \mathcal N$.  It remains to be shown that the perturbed regret lower bound functions $C_\delta(Q)$ are continuous. The claimed result then indeed follows as uniform limits of continuous functions are themselves continuous. 
  
  \textbf{The perturbed regret lower bound is continuous.} We show the perturbed regret lower functions are continuous by applying Proposition \ref{prop:continuity-function-general} to the minimization problem \eqref{eq:special-lower-regret-perturbed}. Let $M_\delta(Q)$ be the feasible set of the perturbed minimization problem \eqref{eq:special-lower-regret-perturbed} where we use as suggested in Section \ref{sec:constraint-qualification} the decomposition
  \begin{align*}
    M_\delta(Q)  =: &~  \Gamma(Q) \cap \set{\eta}{g_i(\eta, Q)\leq 0} \\
    =&~ \tset{\eta}{\eta(x)\geq \delta ~~\forall x\in X, ~\eta(x)=\delta ~~ \forall x\in x^\star_0} \cap \tset{\eta}{1 \leq  \dis(\eta, x, Q)~~ \forall x\in \xd^0}.
  \end{align*}
  We will invoke Proposition \ref{prop:continuity-function-general} with parameter $\theta = Q$, decision variable $y=\eta$,  $\Gamma(Q) = \tset{\eta}{\eta(x)\geq \delta ~~\forall x\in X, ~\eta(x)=\delta ~~ \forall x\in x_0^{\star}}$, and $\set{\eta}{g_i(\eta, Q)\leq 0}= \tset{\eta}{1-\dis(\eta, x, Q) \leq 0 ~~ \forall x\in \xd^0}$.
To invoke Proposition \ref{prop:continuity-function-general}, and establish our claim, it  simply remains to verify that all conditions of this proposition are met:

 \begin{enumerate}
 \item \textbf{Feasible region $M_\delta (Q)$ is a lower semi-continuous mapping at $P$.} We establish this result by invoking Proposition \ref{prop:continuity-mapping-general}.
     \begin{enumerate}
     \item \textit{The Slater's constraint qualification condition is satisfied ($M_{0, \delta}(P)\ne \emptyset$).} This condition, which is presented in Definition \ref{def:slater}, holds if we can show that the set of Slater points
       $$M_{0, \delta}(P)\defn \tset{\eta}{\eta(x)\geq \delta ~~\forall x\in X, ~\eta(x_0^\star)=\delta } \cap \tset{\eta}{1 <  \dis(\eta, x, P)~~ \forall x\in \xd^0}$$
       is non-empty.  To show this, we will first point out that $\dis(\hat \eta, x, P)>0$ for any $\hat \eta \ge \delta > 0$ and any deceitful arm $x\in \tilde X_d^0$. We can then use this result to construct a Slater point in $M_{0, \delta}(P)$. 
         
         By Proposition \ref{prop:continuity-projection}, for any deceitful arm $x'\in \tilde X_d^0$, we can characterize the distance function as the minimum
  \[
    \begin{array}{r@{~~}l@{~}r@{~~}l}
      \dis(\hat \eta, x', P)=\min_{Q\in \mathcal P}  & \sum_{x\in \tilde X^0} \hat \eta(x) I(P(x), Q(x))  \\
      \st & Q(x^\star_0)=P(x^\star_0), \\
                         & \sum_{r \in \setr} r Q(r, x')  \geq \rew^\star(P).
    \end{array}
  \]
  By the classical extreme value theorem, the minimum of the above optimization problem must be achieved at some $Q^\star$ as its feasible set is compact and its objective function is  lower semi-continuous due to Lemma \ref{prop:relative-entropy}. That is, $\dis(\hat \eta, x', P) =\sum_{x\in \tilde X^0} \hat \eta(x) I(P(x), Q^\star(x))$ where $\eta(x)\geq 0$ for all $x\in \tilde X_0$. Because there is only one unique optimal arm $x^\star_0$, which is trivially not in $\xd^0$, it follows that $Q^\star\neq P$ and hence $\dis(\hat \eta, x', P)>0$ due  to the first result of Lemma \ref{prop:relative-entropy}.
  This allows us to define the following Slater (feasible) points for the perturbed problem \eqref{eq:special-lower-regret-perturbed} at $Q= P$:
  $$
  \bar \eta(x)=  \delta+\tfrac{2\hat \eta(x)}{\textstyle\min_{x'\in \tilde X_d^0} \dis(\hat \eta, x', P)}~~~\forall x\in \tilde X^0 \quad {\mathrm{and}} \quad \bar \eta(x^\star_0)= \delta .
  $$
  {To see why these are Slater points, first of all note that $\bar \eta \geq \delta$  and second of all, for any $x\in \tilde X_d^0$, we have  
  \begin{align*}
    \dis(\bar \eta , x, P)= &~ \dis(\delta\textbf{1} +\tfrac{2\hat \eta}{\textstyle\min_{x'\in \tilde X_d^0} \dis(\hat \eta, x', P)} , x, Q) \\
    \ge &~ \dis(\tfrac{2\hat \eta}{\textstyle\min_{x'\in \tilde X_d^0} \dis(\hat \eta, x', P)}, x, P)\\
          =&~\frac{2 \dis(\hat \eta, x, P)}{\min_{x'\in \tilde X_d^0} \dis(\hat \eta, x', P)} \ge 2\,,
  \end{align*}
  where the first equality follows from the fact that the value of $\eta(x^\star_0)$ is immaterial and the first inequality follows from Lemma \ref{lemm:dist-function-properties} where we show  $0<\dis(\alpha \eta, x, Q)=\alpha \dis(\eta, x, Q)$ for $\alpha> 0$. Thus, $\bar \eta \in M_{0, \delta}(P)$.}

\item \textit{The distance function $\dis(\eta, x, Q)$ is upper semi-continuous on $\{M_{0, \delta}(P) \}\times \{P\}$.}
  First remark that $M_{0, \delta}(P)\subseteq \set{\eta}{\eta\geq \delta \cdot \textbf{1}}$.
         In Lemma \ref{lemm:continuity-distance-function}, presented at the end of this section, we show that $\dis(\eta, x, Q)$ is in fact continuous on $\set{\eta}{\eta\geq \delta \cdot \textbf{1}}\times \{P\}$ for any deceitful arm $x\in \tilde X_d^0$ and $P\in\interior(\mc P)$.

       \item  \textit{The distance function $\eta\mapsto \dis(\eta, x, P)$ is concave on $\Gamma(P)$.}
         First remark that $\Gamma(P)\subseteq \set{\eta}{\eta\geq 0}$.
         From its definition, it is also immediately clear (c.f., \citet[Section 3.2.5]{boyd2004convex}) that the distance function $\dis(\eta, x, P)$ is concave on $\set{\eta}{\eta\geq 0}$  for any deceitful arm $x\in \tilde X_d^0$.
     \end{enumerate}
     
   \item \textbf{Feasible region $M_\delta(Q)$ is convex for all $Q\in \mathcal P$ and the mapping $M_\delta(Q)$ is closed at $P$.}
     The feasible region $M_{\delta}(P)$  is evidently convex as the distance functions $\dis(\eta, x, Q)$ are concave in $\eta$.
     The mapping $M_\delta(Q)$ is closed at $P$ if its graph $\set{(\eta, Q)}{Q\in \mathcal N', ~\eta\in M_{\delta}(Q)}$ is a closed set for some closed neighborhood $\mathcal N'$ containing $P$ and itself sufficiently small to be a subset of $\mc P'$. Closedness of this graph follows immediately from Lemma \ref{lemm:continuity-distance-function} in which we indicate that the distance function $\dis(\eta, x, Q)$ is continuous at any point $\eta_0>0$ and $Q\in \mathcal P'$ for any deceitful arm $x\in \tilde X_d^0$.
     
     \item \textbf{The set of minimizers of problem \eqref{eq:special-lower-regret-perturbed} at $P$ is  nonempty and compact.} The set of minimizers in the optimization problem \eqref{eq:special-lower-regret-perturbed} for $P$ is nonempty and compact as its associated objective function is continuous and its sublevel sets restricted to the feasible domain are bounded, c.f., \citet[Prop. 3.2.1]{bertsekas2009convex}.
     
     \item \textbf{The objective function of problem \eqref{eq:special-lower-regret-perturbed} is continuous and convex.} The objective function of the perturbed optimization problem \eqref{eq:special-lower-regret-perturbed}, i.e., $\sum_{x\in \tilde X_0} \eta(x) \Delta(x, Q)$, is continuous at any point $(\eta, Q)$. This is so because the term $\Delta(x, Q)\defn \sum_{r\in \setr}r Q(r, x)-\rew^\star(Q)$ is continuous at any $Q$. Observe that the first term in $\Delta(x, Q)$, i.e., $\sum_{r\in \setr}r Q(r, x)$ is linear in $Q$, and hence continuous. The second term, i.e., $\rew^\star(Q)=\max_{x}\sum_{r\in \setr}Q(r, x)$, is also continuous as it is the maximum of continuous (linear) functions. The product $\sum_{x\in \tilde X_0} \eta(x) \Delta(x, Q)$ will consequently be continuous at any point $(\eta, Q)$. Evidently, the objective function is linear and hence convex in $\eta$.

       The proof of the existence of an $0<\epsilon$-suboptimal rate selection $\eta^c_{\epsilon}(Q)$ is referred to Proposition \ref{prop:continuous-deep-selection}.       
       \hfill\halmos
 \end{enumerate}
 
 \begin{lemma}[Distance Function]
  \label{lemm:dist-function-properties}
 For any $x'\in \tilde X$, the distance function
  \(
    \dis(\eta, x', P)
  \)
  is non-decreasing and positively homogeneous in the exploration rate $\eta$. That is, we have
  \[
    \begin{array}{l@{~}r@{~}l@{\hspace{3em}}l}
      \dis(\eta_1, x', \optp)&\leq& \dis(\eta_2, x', \optp)& \forall 0\leq \eta_1\leq \eta_2,\\
      \dis(a\cdot \eta, x', \optp)&=& a \cdot \dis(\eta, x', \optp) & \forall 0\leq a.
    \end{array}
  \]
\end{lemma}
\begin{proof}[Proof of Lemma \ref{lemm:dist-function-properties}.]
  Both statements are an immediate consequence of the fact that the objective function in the information minimization problem \eqref{eq:inner-optimization}, defining our distance function $\dis(\eta, x', P)$, is nondecreasing and positively homogeneous in the logarithmic rate $\eta$ for any reward distribution $Q$.
\end{proof}

 \begin{lemma}[Continuity of the Distance Function]
   \label{lemm:continuity-distance-function}
  Assume that the reward distribution $P \in\interior(\mathcal P)$ has a unique optimal arm $x^\star(P)$ and consider any of its deceitful arms $x'\in \xd(P)$. The distance function
  \(
    \dis( \eta, x', Q)
  \)
  is continuous at any point $(\eta_0, P)$ with $\eta_0>0$.  
\end{lemma}

\subsubsection{Proof of Lemma \ref{lemm:continuity-distance-function}}
  Recall that an arm $x'\in \xd(Q)$ is denoted as deceitful for some reward distribution $Q$ if we have that $\rew_{\max}(x', Q)> \rew^{\star}(Q)$.
  From Lemma \ref{lemm:stability-optimal-arm}, we know there exists a neighborhood $\mathcal N\subseteq \mathcal P$ around $P$ so that 
  \[
  Q\in \mathcal N\implies x^\star(Q)=x^\star(P):=x^\star_0,~\tilde X(Q)=\tilde X(P):=\tilde X_0  {\mathrm{~and~}} \xd(P) = \xd^0\subseteq \xd(Q).
  \]
  Note that in general we may observe $\xd(P)\subset \xd(Q)$ as an nondeceitful arm $x\in \xn(P)$ for which $\rew^\star(x, P)=\rew^\star(P)$ may become deceitful $x\in \xd(Q)$ with $Q\in \mathcal N$.
  For each reward distribution $Q\in \mathcal N$ in a neighborhood around $P$,  arm $x^\star(P)$ remains optimal, arms in set $\tilde X(P)$ remain suboptimal, while any deceitful arm in $\xd(P)$ remains deceitful. Let $\eta_0> 0$ be any arbitrary positive logarithmic exploration rate.
  From Proposition \ref{prop:continuity-projection}, it follows that for any $Q\in \mathcal N$, we have that the distance function for any deceitful arm $x'\in \xd^0$ simplifies  to
  \begin{equation}
    \label{eq:distance-function-special}
    \begin{array}{r@{~~}l@{~}r@{~~}l}
      \dis( \eta, x', Q) =\min_{Q'\in \mathcal P}  & \sum_{x\in \tilde X_0} \eta(x) I(Q(x), Q'(x))  \\
      \st    & \sum_{r \in \setr} r Q'(r, x')  \geq \sum_{r \in \setr} r Q(r, x_0^\star), \\
                                                              &Q'(x_0^\star)=Q(x_0^\star).
    \end{array}
  \end{equation}
  We would like to show that  function $\dis( \eta, x', Q)$ is continuous at $(\eta_0, P)$ by applying Proposition \ref{prop:continuity-function-general} to its minimization characterization stated in Equation \eqref{eq:distance-function-special}. To do so, we will verify all the conditions in this proposition. Before that, let us decompose as suggested in Section \ref{sec:constraint-qualification} the feasible set mapping of problem \eqref{eq:distance-function-special} as follows: 
  \begin{align*}
    M(\eta, Q) =:&~ \Gamma(\eta, Q) \cap \set{Q'}{g(Q', (\eta, Q))\leq 0} \\
    = &~ \tset{Q'\in \mathcal P}{Q'(x_0^\star)=Q(x_0^\star)} \bigcap \tset{Q'}{\textstyle\sum_{r \in \setr} r Q'(r, x')  \geq \sum_{r \in \setr} r Q(r, x_0^\star)}.
  \end{align*} 
We now proceed to verifying all the conditions of Proposition \ref{prop:continuity-function-general}.

\begin{enumerate}
    \item \textbf{Feasible region $M(\eta, Q)$ is a lower semi-continuous mapping at $(\eta_0, P)$.} We show this result by invoking Proposition \ref{prop:continuity-mapping-general}. We need to verify the following conditions to invoke Proposition \ref{prop:continuity-mapping-general}.

      {
    \begin{enumerate}
    \item \textit{The mapping $\Gamma(\eta, Q)$ is lower semi-continuous at $(\eta_0, P)$.} Clearly, the mapping $\Gamma(\eta, Q)\defn \set{Q'\in \mc P}{Q'(x_0^\star)=Q(x_0^\star)}$ is independent of $\eta_0$. Hence, this claim is already established in Lemma \ref{lemm:lsc:feasible-mapping}.
      
    \item \textit{Set $\Gamma(\eta_0, P)$ is convex and non-empty.} Solutions sets of linear systems are well known to be affine and hence convex. Evidently, we have $P\in \Gamma(\eta_0, P)$ and hence $\Gamma(\eta_0, P)$ must be non-empty.

        \item \textit{The Slater's constraint qualification condition is satisfied ($M_0(\eta_0, P)\ne \emptyset$).}  
        Consider a deceitful arm $x'\in \tilde X_d(P)$ and let 
  \begin{equation*}
    \begin{array}{rl}
      Q^\star \in \arg\max & \sum_{r \in \setr} r Q(r, x') \\
      \st & Q\in\mathcal P, \\
                            & Q(x^\star_0)=P(x^\star_0).
                                      
    \end{array}
  \end{equation*}
  Note that the maximum of the above optimization problem is attained at some distribution $Q^\star$ because its objective function is continuous and its feasible region is a compact set. Because  arm $x'$ is deceitful, i.e., $\sum_{r \in \setr} r Q^\star(r, x') > \sum_{r \in \setr} r P(r, x^\star_0)$, it  follows immediately that $Q^\star\in M_0(\eta_0, P) = \tset{Q'\in \mathcal P}{Q'(x_0^\star)=Q(x_0^\star)} \bigcap \tset{Q'}{\textstyle\sum_{r \in \setr} r Q'(r, x')  > \sum_{r \in \setr} r Q(r, x_0^\star)}$, which is the desired result.
\end{enumerate}
}
    \item  \textbf{Feasible region $M(\eta, Q)$ is convex for all $Q\in \mathcal P$ and $M(\eta, Q)$ is a closed mapping at $(\eta_0, P)$.} The set $\mathcal P$ is assumed to be convex and closed. The $\tset{Q'}{~Q'(x_0^\star)=Q(x_0^\star), ~\textstyle\sum_{r \in \setr} r Q'(r, x')  \geq \sum_{r \in \setr} r Q(r, x_0^\star)}$ is a closed and convex polyhedral set. Consequently, the intersection $M(\eta, P) = \mathcal P \cap \tset{Q'}{~Q'(x_0^\star)=P(x_0^\star), ~\textstyle\sum_{r \in \setr} r Q'(r, x')  \geq \sum_{r \in \setr} r P(r, x_0^\star)}$ is a closed convex set for any $(\eta, Q)$. The mapping $M(\eta, Q)$ is closed at $(\eta_0, P)$ as its  graph  $$\set{(Q', \eta, Q)}{Q'\in \mathcal P,~Q'(x_0^\star)=Q(x_0^\star), ~\textstyle\sum_{r \in \setr} r Q'(r, x')  \geq \sum_{r \in \setr} r Q(r, x_0^\star),~(\eta, Q)\in \mathcal N'}$$ is closed for any closed neighborhood $\mathcal N'$ containing $(\eta_0, P)$.
    
    \item \textbf{The set of minimizers of problem \eqref{eq:distance-function-special} at $(\eta_0, P)$ is nonempty and compact.} The set of minimizers of problem  \eqref{eq:distance-function-special} is nonempty and compact as its feasible set $P\in M(\eta_0, P)$ is compact (recall that $\mathcal P$ is a compact set) and its objective function is lower semi-continuous because of the last result of Lemma \ref{prop:relative-entropy}, stated in Section \ref{sec:information-topology}; see also \citet[Proposition 3.2.1]{bertsekas2009convex}.

    \item    \textbf{The objective function of problem \eqref{eq:distance-function-special} is continuous at $(\eta_0, P)$ and convex in $Q'$.} With a slight abuse of notation, let $Q^\star$ be an arbitrary minimizer to problem  \eqref{eq:distance-function-special}. As the target rate satisfies $\eta_0>0$, we will now argue that we must have that $P\ll Q^\star \in \interior(\mathcal P)$. Indeed, otherwise $\dis( \eta_0, x', P)=\sum_{x\in \tilde X_0} \eta_0(x) I(P(x), Q^\star(x))=+\infty$ which would contradict Proposition \ref{prop:continuity-projection} in which we show that for any deceitful arm $x'\in \xd(P)$ the distance function  $\dis( \eta_0, x', P)$ is always finite. Hence, the objective function is continuous at the point $(\eta_0, P, Q^\star)\in \real_{++}\times \interior(\mathcal P)\times \interior(\mathcal P)$.
  Finally,  the objective function of problem  \eqref{eq:distance-function-special} is also convex in $Q'$, per Lemma \ref{prop:relative-entropy}. \hfill\Halmos
\end{enumerate}

 \subsection{Proof of Proposition \ref{prop:continuous-deep-selection}}
\label{sec:continuous-selections}

We now proceed to show that the selection functions $(\eta'_\epsilon(Q), \mu_\epsilon(Q))$, defined in Equation \eqref{eq:minimal-primal-dual-projection-1}, are continuous at any reward distribution $P\in \mathcal P'$ satisfying Assumption \ref{assumption}.
We will prove that the non-empty, closed and convex mapping representing the feasible set $M(Q)$ of problem \eqref{eq:minimal-primal-dual-projection-1} is lower semi-continuous at $P$. By Michael's theorem, c.f., \citet[Corollary 9.1.2]{aubin2009set}, there exists a selection $(\eta^c_\epsilon(Q), \mu^c_\epsilon(Q))$ which is continuous on $\mathcal N'\subseteq \mathcal P'$, a sufficiently small closed neighborhood of $P$. The existence of such neighborhood is guaranteed by Lemma \ref{lemm:stability-optimal-arm}. Remark that by our minimal selection definition, we must have for all $Q\in \mathcal N'$ that
\[
  \sum_{x\in X}\eta'_\epsilon(x, Q)^2 + \sum_{x\in X} \norm{\mu_\epsilon(x, Q)}^2 \leq \max_{Q\in \mathcal N'}  \sum_{x\in X}\eta^c_\epsilon(x, Q)^2 + \sum_{x\in X} \norm{\mu^c_\epsilon(x, Q)}^2 <\infty
\]
where the last inequality follows from the fact that continuous functions on compact sets attain a finite maximum.  The minimum norm selection $(\eta'_\epsilon(Q), \mu_\epsilon(Q))$ is unique and continuous at such $P$ following \citet[Corollary 9.3.2]{aubin2009set}.

  From Lemma \ref{lemm:stability-optimal-arm}, there exists an open neighborhood $\mathcal N\subseteq \mathcal P'$ around $P$ satisfying Assumption \ref{assumption} so that 
  \[
    Q\in \mathcal N \implies x^\star(Q)=x^\star(P):=x^\star_0,~\tilde X(Q)=\tilde X(P):=\tilde X_0  {\mathrm{~and~}}  \xd(Q)=\xd(P):=\xd^0.
  \]
  That is, for each reward distribution $Q\in \mathcal N $ in a neighborhood around $P$, the arm $x^\star(P)$ remains optimal,  the arms $\tilde X(P)$ remain suboptimal, and any (non)deceitful arm remains (non)deceitful. Hence, the feasible set for $Q\in \mathcal N$ of the minimization problem \eqref{eq:minimal-primal-dual-projection-1} is captured by the following mapping:
  \[
    M(Q)=
    \set{(\eta,\mu)}{
      \begin{array}{l}
        \eta(x)\in \real_+, ~\mu(x)\in \mathcal K^\star\times\real_+\times\real \quad\forall x\in X,\\
        1\leq \dual(\eta, x, Q~;~ \mu(x))  \quad\forall x\in \xd^0,
        \\
        \sum_{x\in \tilde{X}_0}\eta(x) \Delta(x, Q) \leq C(Q)+\epsilon,
        \\
         \eta(x)\geq \epsilon/(2\sum_{x\in \tilde{X}_0} \Delta(x, Q))\quad\forall x\in \tilde X_0, \\
        \eta (x)\geq \lambda(r, x, x')+\beta(x')+\alpha(x') \cdot r \mb 1(x=x')+\frac{\epsilon}{2\sum_{x\in \tilde{X}_0} \Delta(x, Q)}\\
        \quad\forall x\in\xd^0,~x\in \tilde X_0
      \end{array}
    }.
  \]
  
  \textbf{Mapping $M$ is lower-semicontinuous at $P$.} 
  We will establish that the mapping $M$ is  lower semi-continuity by verifying  the following conditions of Proposition \ref{prop:continuity-mapping-general} where we use the decomposition
  \[
    \begin{array}{rl}
      M(Q) & =: \Gamma(Q) \cap \set{(\eta,\mu)}{g_i((\eta,\mu), Q)\leq 0 ~~\forall i\in [1, \dots, m_2]}\\
           & = \{ \mu(x)\in \mathcal K^\star\times\real_+\times\real \quad\forall x\in X\}\cap \\
           & \quad  \set{(\eta,\mu)}{
             \begin{array}{l}               
               1\leq \dual(\eta, x, Q~;~ \mu(x))  \quad\forall x\in \xd^0,
               \\
               \sum_{x\in \tilde{X}_0}\eta(x) \Delta(x, Q) \leq C(Q)+\epsilon,
               \\
               \eta(x)\geq \epsilon/(2\sum_{x\in \tilde{X}_0} \Delta(x, Q))\quad\forall x\in \tilde X_0, ~\eta(x)\geq 0\quad \forall x\in X,\\
               \eta (x)\geq \lambda(r, x, x')+\beta(x')+\alpha(x') \cdot r \mb 1(x=x')+\frac{\epsilon}{2\sum_{x\in \tilde{X}_0} \Delta(x, Q)} \\ \quad\forall x\in\xd^0,~ x\in \tilde X_0
             \end{array}
      }.
    \end{array}
  \]
  
  \begin{itemize}
      \item \textbf{The Slater's constraint qualification condition.} To verify this condition, we need to show that 
      the set of Slater points
  \[
    M_0(P)=
    \set{(\eta,\mu)}{
      \begin{array}{l}
      \eta(x)>0, ~\mu(x)\in \mathcal K^\star\times\real_{+}\times\real \quad\forall x\in X, \\ 
      \dual(\eta, x, P~;~ \mu(x))> 1  \quad\forall x\in \xd^0,\\
       \sum_{x\in \tilde{X}_0}\eta(x) \Delta(x, P) < C(P)+\epsilon,\\
        \eta(x)> \epsilon/(2\sum_{x\in \tilde{X}_0} \Delta(x, P))\quad\forall x\in \tilde X_0,\\
        \eta (x)> \lambda(r, x, x')+\beta(x')+\alpha(x') \cdot r \mb 1(x=x')+\frac{\epsilon}{2\sum_{x\in \tilde{X}_0} \Delta(x, P)}\\ \quad\forall x\in\xd^0, ~x\in \tilde X_0
      \end{array}
    }
  \]
  is nonempty. Consider to that end first the following $\epsilon/4$-suboptimal logarithmic exploration rates 
  \[
    \set{(\eta,\mu) }{
      \begin{array}{l}
        \eta \geq 0, ~\mu(x)\in \mathcal K^\star\times\real_+\times\real \quad\forall x\in \xd^0,\\
        \dual(\eta, x, P~;~ \mu(x))\geq 1  \quad\forall x\in \xd^0,
        \\
        \sum_{x\in \tilde{X}_0}\eta(x) \Delta(x, P) \leq C(P)+\epsilon/4,\\
        \eta (x)\geq \lambda(r, x, x')+\beta(x')+\alpha(x') \cdot r \mb 1(x=x')  \quad\forall x\in \xd^0, x\in \tilde X_0
      \end{array}
    }
  \]
  in the optimization problem \eqref{eq:silo-dual}. From the definition of the infimum in problem \eqref{eq:silo-dual}, the previous set is nonempty for all $\epsilon>0$. Consider an arbitrary element $(\eta_{\epsilon/4}, \mu_{\epsilon/4})$ in the previous set. We can easily verify that the point $(\eta_{\epsilon/4}+\epsilon/(2\sum_{x\in \tilde{X}_0} \Delta(x, P))\cdot \textbf{1},\mu_{\epsilon/4}) $ is in $M(P)$, where $\textbf{1}$ is the all one vector. Indeed, from Lemma \ref{lemm:dual-function-properties}, we have that $\dual(\eta_{\epsilon/4}+\epsilon/(2\sum_{x\in \tilde{X}_0} \Delta(x, P))\cdot \textbf{1}, x, P~;~ \mu_{\epsilon/4})\geq \dual(\eta_{\epsilon/4}, x, P~;~ \mu_{\epsilon/4})\geq 1$ and $\sum_{x\in \tilde{X}_0}(\eta_{\epsilon/4}(x)+\epsilon/(2\sum_{x\in \tilde{X}_0} \Delta(x, P))) \Delta(x, P)= \sum_{x\in \tilde{X}_0}\eta_{\epsilon/4}(x) \Delta(x, P)+\epsilon/2\leq \epsilon/4 + \epsilon/2 = 3\epsilon/4<\epsilon$.

{We show that a similar construction can be used to construct a Slater point in  $M_0(P)$ as well. Consider a point $(\eta'', \mu_{0}=(\alpha_0, \beta_0, \lambda_0))=(1+\delta)\cdot (\eta_{\epsilon/4}(P)+\epsilon/(2\sum_{x\in \tilde{X}(P)} \Delta(x, P)\cdot \textbf{1}),\mu_{\epsilon/4}) \in M(P)$ for some $\delta>0$ small enough so that $\delta/(1-3\delta)<\tfrac{\epsilon}{(4 C(P))}$. Using again Lemma \ref{lemm:dual-function-properties}, it is trivial to verify that we have
  \[
    \begin{array}{l}
      \eta''\geq (1+\delta)\epsilon/(2\sum_{x\in \tilde{X}_0} \Delta(x, P)) \cdot \textbf{1} > \epsilon/(2\sum_{x\in \tilde{X}_0} \Delta(x, P)) \cdot \textbf{1}>0,\\
      \dual(\eta'', x, P~;~ \mu_{0}(x))\geq 1+\delta > 1 \quad\forall x\in \xd^0,\\
      \sum_{x\in \tilde{X}_0}\eta''(x) \Delta(x, P) \leq (1+\delta)\cdot (C(P)+3\epsilon/4) < C(P)+\epsilon,\\
      \eta'' (x)> \lambda_{0}(r, x, x')+\beta_{0}(x')+\alpha_{0}(x') \cdot r \mb 1(x=x')+\epsilon/(2\sum_{x\in \tilde{X}_0} \Delta(x, P)) \\ \quad\forall x'\in \xd^0, x\in \tilde X_0.
    \end{array}
  \]
 }
 
 \item \textbf{The rest of the conditions.} We first verify the required convexity conditions. It is clear that the constraint functions $g_i((\eta, \mu), P)$ which appear in the feasible set mapping $M$ are linear and hence convex jointly in $\eta$ and $\mu$ with the exception of those related to the constraints $\dual(\eta, x, P~;~ \mu(x))\geq 1$ for all $\forall x\in \xd^0$. As pointed out previously, the dual function $\dual(\eta, x, P~;~ \mu(x))$ is concave jointly in $\eta$ and $\mu$ and consequently all constraint functions are convex. Finally, we verify all required continuity conditions. As $\eta(x)\geq \epsilon/(2\sum_{x\in \tilde{X}_0} \Delta(x, Q))>{0}$ and $$\eta (x)- \lambda(r, x, x')-\beta(x')-\alpha(x') \cdot r \mb 1(x=x')\geq \epsilon/(\textstyle 2\sum_{x\in \tilde{X}_0} \Delta(x, Q))>0$$ for all $x'\in \xd^0$ and $x\in \tilde X_0$, the constraint functions $-\dual(\eta, x, Q~;~ \mu(x))$ for all $\forall x\in \xd^0$ are continuous on $M_0(P)\times \{P\}$; see Proposition \ref{prop:continuity-dual-function} below. The constraint functions related to all other constraints characterizing the feasible set mapping $M$ are linear and hence continuous.  \hfill\Halmos \end{itemize}
 
 \begin{proposition}[Continuity of the Dual Function]
  \label{prop:continuity-dual-function} Consider any distribution $P\in \mathcal P'$ which satisfies Assumption \ref{assumption}.
  Then, the dual function $\dual(\eta,x', P~;~\mu)$ for a fixed arm $x'$ and reward distribution $P$ is upper semi-continuous at any point $(\eta, \mu)$ in its  domain. Furthermore, the  dual function is also continuous for any fixed arm $x'$ at any point $(\eta_0, P, \mu_0)$ with 
  $\eta_0(x)>0$ and $\eta_0 (x)- \lambda_0(r, x)-\beta_0-\alpha_0 \cdot r \mb 1(x=x')>0$ for all  $x\in \tilde X(P)$ and $r\in\setr$.
\end{proposition}

\subsubsection{Proof of Proposition \ref{prop:continuity-dual-function}}
  We first argue that  the  dual function is upper semi-continuous\footnote{A function $y\mapsto f(y)$ is upper semi-continuous at any point in its domain if and only if its superlevel sets $\set{y}{f(y)\geq\gamma}$ are closed for any $\gamma\in\real$.} at any point in its domain.
  We decompose the dual function into
    \begin{align*}
      \dual(\eta,x', Q~;~\mu) & = \sum_{x\in \tilde X(Q),\, r\in \setr} \eta(x) \log\left( \frac{\eta (x)- \lambda(r, x)-\beta-\alpha \cdot r \mb 1(x=x')}{\eta(x)}\right) Q(r,x) \\
                              & \quad - \textstyle \sum_{x\in x^\star(Q),r\in \setr} \lambda(r, x)Q(r, x)  + \alpha \rew^{\star}(Q) + \beta|\tilde X(Q)| \\
                              & \quad + \chi_{-\infty}(\eta (x)\geq \lambda(r, x)+\beta+\alpha \cdot r \mb 1(x=x'))
    \end{align*}
  and will show that each of the terms is an upper semi-continuous function at any point of their respective domains. The first part of proposition than follows from the fact that the sum of upper semi-continuous functions is itself upper semi-continuous on the intersection of the domains.
  
  It is clear that the characteristic function $\chi_{-\infty}(S)$ of a closed set $S$ is upper semi-continuous on its entire domain as its $\gamma$-superlevel sets are either the closed set $S$ for $\gamma\leq 0$ or the empty set when $\gamma>0$. Hence, the term  $$\chi_{-\infty}(\eta (x)\geq \lambda(r, x)+\beta+\alpha \cdot r \mb 1(x=x'))$$ as an affine composition of $\chi_{-\infty}(\real_+)$ and $\eta (x)- \lambda(r, x)-\beta-\alpha \cdot r \mb 1(x=x')$ is an upper semi-continuous function, c.f., \citet[Proposition 1.1.4]{bertsekas2009convex}.

  The log perspective function $(u, v)\mapsto u \log(\tfrac{v}{u})$ is upper semi-continuous on its entire domain; see \citet[Proposition 2.3]{combettes2018perspective}.  Hence, the affine composition function
  $$\eta(x) \log\left( \tfrac{(\eta (x)- \lambda(r, x)-\beta-\alpha \cdot r \mb 1(x=x'))}{\eta(x)}\right)$$ of the log perspective function and $(\eta(x), \eta (x)- \lambda(r, x)-\beta-\alpha \cdot r \mb 1(x=x'))$ is upper semi-continuous too; c.f., \citet[Proposition 1.1.4]{bertsekas2009convex}.
  As we have that $P(r,x)\geq 0$, the second term
  \begin{align}\label{eq:problem}
    \eta(x)\log\left( \frac{(\eta (x)- \lambda(r, x)-\beta-\alpha \cdot r \mb 1(x=x'))}{\eta(x)}\right)P(r, x)
  \end{align}
  is upper semi-continuous as well, c.f., \citet[Proposition 1.1.5]{bertsekas2009convex}. The remaining terms  $\beta |\tilde X(P)|$, $\alpha \rew^{\star}(P)$ and $-\sum_{x\in x^\star(P),r\in \setr} \lambda(r, x)P(r, x)$ are linear and thus upper semi-continuous functions as well. 
  
  Next, we argue that the dual function is continuous as well. This  follows immediately from the fact that the problematic terms, presented in Equation \eqref{eq:problem},
  are finite and continuous when $\eta(x)>0$ and $\eta (x)- \lambda(r, x)-\beta-\alpha \cdot r \mb 1(x=x')>0$ for all  $x\in X$ and $r\in\setr$. Indeed, the terms $\alpha \rew^{\star}(Q)$ and $\beta|\tilde X(Q)|$ are continuous via Lemma \ref{lemm:stability-optimal-arm}.
  \hfill\Halmos

  \subsection{Proof of Proposition \ref{prop:continuous-shallow-selection}}\label{sec:proof:shallow}
Let $(\eta'_\epsilon(Q), \mu_\epsilon(Q))$ be the $\epsilon$-suboptimal selection from the feasible set of problem  presented in Equation \eqref{eq:minimal-primal-dual-projection-1}.
The selection $(\eta'_\epsilon(Q), \mu_\epsilon(Q))$ is continuous on any sufficiently small neighborhood $\set{Q}{\norm{Q-P}_\infty\leq \kappa}$ around $P$ due to Proposition \ref{prop:continuous-deep-selection} and the fact that $\mathcal P'$ is an open set. 
We can  hence assume $\kappa>0$ is sufficiently small so that $\max \tset{\norm{\eta'_\epsilon(Q)-\eta'_\epsilon(P)}}{\norm{Q-P}_\infty\leq \kappa}\leq \delta$, where  $\delta>0$ is an arbitrary positive number. Then, for any $Q_2$ so that $\norm{Q_2-P}_\infty\leq \kappa$, we have   
\begin{align}
  \begin{split}
    \label{eq:shallow-bound}
    & \norm{\text{\sf SU}(Q_1, \eta_\epsilon'(Q_2) ~;~\mu_\epsilon(Q_2), \epsilon)-\eta'_\epsilon(P)}_2\\
    \leq & \norm{\text{\sf SU}(Q_1, \eta_\epsilon'(Q_2) ~;~\mu_\epsilon(Q_2), \epsilon)-\eta'_\epsilon(Q_2)}_2+\norm{\eta'_\epsilon(Q_2)-\eta'_\epsilon(P)}_2\\
    \leq & \norm{\text{\sf SU}(Q_1, \eta_\epsilon'(Q_2) ~;~\mu_\epsilon(Q_2), \epsilon)-\eta'_\epsilon(Q_2)}_2+\delta.
  \end{split}
\end{align}
  It remains to be shown  that the term $\norm{\text{\sf SU}(Q_1, \eta_\epsilon'(Q_2) ~;~\mu_\epsilon(Q_2), \epsilon)-\eta'_\epsilon(Q_2)}_2$ diminishes to zero for all $Q_1$ and $Q_2$ such that $\norm{Q_1-P}\leq \kappa$ and $\norm{Q_2-P}\leq \kappa$. From Lemma \ref{lemm:stability-optimal-arm}, we may assume that $\kappa>0$ is sufficiently small to guarantee
  \[
    \xd(Q_2)=\xd(Q_1)=\xd(P).
  \]
  It should be remarked that the term $\norm{\text{\sf SU}(Q_1, \eta_\epsilon'(Q_2) ~;~\mu_\epsilon(Q_2), \epsilon)-\eta'_\epsilon(Q_2)}_2$ coincides with the minimum in problem \eqref{eq:minimal-primal-projection-1}. Hence, the objective value of any feasible selection from problem \eqref{eq:minimal-primal-projection-1} immediately serves as an upper bound on $\norm{\text{\sf SU}(Q_1, \eta_\epsilon'(Q_2) ~;~\mu_\epsilon(Q_2), \epsilon)-\eta'_\epsilon(Q_2)}_2$ as we discuss next. We will require the following supporting Lemma.

  \begin{lemma}
    \label{lemma:cont-dual-2}
    Let $\overline{\dual} ( \eta'_\epsilon(Q_2), x', Q_2~;~ \mu_{\epsilon}(x',Q_2))\ge 1$ for some arm $x'\in X$ and all $Q_2$ with $\norm{Q_2-P}\leq \kappa$. There exists $\delta(\kappa, x, P)$ (independent of $Q_2$) such that
    $$\overline{\dual} ( \eta'_\epsilon(Q_2), x', Q_1~;~ \mu_{\epsilon}(x',Q_2))\ge 1-\delta(\kappa,x', P)$$ for all $Q_1$ with $\norm{Q_1-P}_\infty\leq \kappa$ and $\lim_{\kappa\to 0}\delta(\kappa, x', P)=0$.
  \end{lemma}

  Since $(\eta'_\epsilon(Q_2), \mu_{\epsilon}(Q_2))$ is an $\epsilon$-suboptimal solution to problem \eqref{eq:silo-dual} with $Q=Q_2$, we have the information condition
  \(\overline{\dual} ( \eta'_\epsilon(Q_2), x, Q_2~;~ \mu_{\epsilon}(x, Q_2))\geq \dual ( \eta'_\epsilon(Q_2), x, Q_2~;~ \mu_{\epsilon}(x, Q_2))\ge 1 ~~ \forall x\in \xd(P)\) for all $Q_2$ such that $\norm{Q_2-P}\leq \kappa$ where we select $\kappa$ to be sufficiently small to ensure that $\xd(Q_2)=\xd(P)$ which is possible as $P\in \mathcal P'$.
  From positive homogeneity of the test dual function and Lemma \ref{lemma:cont-dual-2} it follows that we can consider for any $P$ a small enough $\kappa>0$ such that 
\[
  \overline{\dual}(\eta'_\epsilon(Q_2), x, Q_1~;~ \mu_\epsilon(x, Q_2))\geq \textstyle 1-\delta(\kappa, P)  \quad \forall x\in \xd(Q_1)
\]
for all $Q_1$ and $Q_2$ with $\norm{Q_1-P}\leq \kappa$ and $\norm{Q_2-Q_1}\leq \kappa$ where $\delta(\kappa, P)=\max_{x\in \xd(P)} \delta(\kappa, x, P)$ and $\lim_{\kappa\to 0}\delta(\kappa, P)=0$. Suppose that we consider $\kappa$ small enough so that
$\delta(\kappa, P)\leq \delta$, where  $\delta>0$ is taken to be such that $(C(P)+\epsilon)/(1-\delta)\leq C(P)+2\epsilon$. Then, we have that $\eta'_\epsilon(Q_2)/(1-\delta)$ is feasible in minimization problem \eqref{eq:minimal-primal-projection-1}. Hence, the term of interest is bounded by $\norm{\text{\sf SU}(Q_1, \eta_\epsilon'(Q_2) ~;~\mu_\epsilon(Q_2), \epsilon)-\eta'_\epsilon(Q_2)}_2\leq \norm{\eta'_\epsilon(Q_2)/(1-\delta)-\eta'_\epsilon(Q_2)}_2\leq  \norm{\eta_\epsilon'(Q_2)}_2(1/(1-\delta) - 1)$.
We thus have from Equation \eqref{eq:shallow-bound} finally that 
\begin{align*}
  \norm{\text{\sf SU}(Q_1, \eta_\epsilon'(Q_2) ~;~\mu_\epsilon(Q_2), \epsilon)-\eta'_\epsilon(Q_2)}_2\leq & \norm{\eta_\epsilon'(Q_2)}_2(1/(1-\delta) - 1)+\delta \\
  \leq & (\norm{\eta_\epsilon'(P)}_2+\delta)(1/(1-\delta) - 1)+\delta.
\end{align*}
The result follows as $\delta>0$ can be made arbitrarily small and $\norm{\eta_\epsilon'(P)}_2$ is bounded.\hfill\Halmos

\subsubsection{Proof of Lemma \ref{lemma:cont-dual-2}}
Fix some arm $x'\in X$. For $Q$ such $\norm{Q-P}\leq \kappa$, we can simplify the dual function
\begin{align*}
 \dual ( \eta, x', Q~;~ \mu) =& \sum_{x\in \tilde X(Q),\, r\in \setr} \eta(x) \log\left( \frac{\eta (x)- \lambda(r, x)-\beta-\alpha \cdot r \mb 1(x=x')}{\eta(x)}\right) Q(r,x) \\
                              & \quad - \textstyle \sum_{x\in x^\star(Q),r\in \setr} \lambda(r, x)Q(r, x)  + \alpha \rew^{\star}(Q) + \beta|\tilde X(Q)| \\
                                                                 & \quad + \chi_{-\infty}(\eta (x)\geq \lambda(r, x)+\beta+\alpha \cdot r \mb 1(x=x'))\\
             = & \sum_{x\in \tilde X_0,\, r\in \setr} \eta(x) \log\left( \frac{\eta (x)- \lambda(r, x)-\beta-\alpha \cdot r \mb 1(x=x')}{\eta(x)}\right) Q(r,x) \\
                              & \quad  \sum_{r\in \setr} (\alpha r -\lambda(r, x^\star_0))Q(r, x_0^\star) + \beta|\tilde X(P)|
\end{align*}
using the fact that $|\tilde X(Q)| = |\tilde X(P)| = \abs{X}-1 $ and $\rew^{\star}(Q) = \sum_{r\in \setr} r Q(r, x^\star_0)$. As a function of the reward distribution the dual is hence recognized to be affine in $Q$ for $\norm{Q-P}\leq \kappa$. We may use the standard inequality $u\tpose v\geq -\norm{u}_1 \norm{v}_\infty$ to bound the dual function as
\begin{align*}
  & \overline{\dual}(\eta'_\epsilon(Q_2), x, Q_1~;~\mu_\epsilon(x,Q_2))\\
  & \geq \dual(\eta'_\epsilon(Q_2), x, Q_1~;~\mu_\epsilon(x, Q_2))\\
  & \geq \min_{\norm{Q-Q_2}_\infty\leq 2\kappa} \dual(\eta'_\epsilon(Q_2), x, Q~;~\mu_\epsilon(x, Q_2))\\
  & \geq \dual(\eta'_\epsilon(Q_2), x, Q_2~;~\mu_\epsilon(x, Q_2)) - \\
  & \qquad \sum_{x\in \tilde X_0, r\in \setr} 2\kappa \abs{\eta'_{\epsilon}(x, Q_2) \log\left( \frac{\eta'_\epsilon (x, Q_2)- \lambda_\epsilon(r, x, Q_2)-\beta_{\epsilon}(Q_2)-\alpha_{\epsilon}(Q_2) \cdot r \mb 1(x=x')}{\eta'_\epsilon(x,Q_2)}\right)}\\
  & ~\quad - \quad\, \sum_{r\in \setr}~~~2\kappa \abs{\alpha_{\epsilon}(Q_2)-\lambda_{\epsilon}(r, x^\star_0, Q_2)}\\
                                                           & \geq 1 -
                                                             \sum_{x\in \tilde X_0, r\in \setr} 2\kappa \abs{\eta'_{\epsilon}(x,Q_2) \log\left( \frac{\eta'_\epsilon (x, Q_2)- \lambda_\epsilon(r, x,Q_2)-\beta_{\epsilon}(Q_2)-\alpha_{\epsilon}(Q_2) \cdot r \mb 1(x=x')}{\eta'_\epsilon(x, Q_2)}\right)} \\
  & ~\quad - \quad\, \sum_{r\in \setr}~~~2\kappa \abs{\alpha_{\epsilon}(Q_2)-\lambda_{\epsilon}(r, x^\star_0, Q_2)},
\end{align*}
where we exploit the triangle inequality $\norm{Q_2-Q_1}_\infty\leq \norm{Q_2-P}_\infty + \norm{P-Q_1}_\infty\leq 2\kappa$. We will now indicate that absolute values appearing in the previous inequality can be bounded uniformly for all $Q_2$ with $\norm{Q_2-P}_\infty\leq\kappa$ from which the claim follows. That is, we take
\begin{align*}
   &\delta(\kappa, x, P) \\
   & \defn \max_{\norm{Q_2-P}_\infty\leq \kappa}  \sum_{x\in \tilde X_0, r\in \setr} \hspace{-1em}2\kappa \abs{\eta'_{\epsilon}(x,Q_2) \log\left( \frac{\eta'_\epsilon (x, Q_2)\!-\! \lambda_\epsilon(r,x,Q_2)\!-\!\beta_{\epsilon}(Q_2)\!-\!\alpha_{\epsilon}(Q_2) r \mb 1(x=x')}{\eta'_\epsilon(x, Q_2)}\right)} \\
  & \qquad + \sum_{r\in \setr}~~~2\kappa \abs{\alpha_{\epsilon}(Q_2)-\lambda_{\epsilon}(r, x^\star_0, Q_2)}.
\end{align*}
We have indeed that $\lim_{\kappa\to 0}\delta(\kappa, x, P)=0$ if we can show $$\min_{\norm{Q_2-P}_\infty\leq \kappa} \eta'_{\epsilon}(x,Q_2)>0\,,$$ $$\min_{\norm{Q_2-P}_\infty\leq \kappa} \eta'_\epsilon (x, Q_2)- \lambda_\epsilon( r, x, Q_2)-\beta_{\epsilon}(Q_2)-\alpha_{\epsilon}(Q_2) \cdot r \mb 1(x=x')>0\,,$$ $$\max_{\norm{Q_2-P}_\infty\leq \kappa} \eta'_\epsilon (x,Q_2)- \lambda_\epsilon(r, x, Q_2)-\beta_{\epsilon}(Q_2)-\alpha_{\epsilon}(Q_2) \cdot r \mb 1(x=x')<\infty\\,$$ and $\max_{\norm{Q_2-P}_\infty\leq \kappa}\abs{\alpha_{\epsilon}(Q_2)-\lambda_{\epsilon}(r, x^\star_0, Q_2)}<\infty$ for all $r\in \setr$ and $x\in \tilde X_0$.

Define 
$$R(\kappa)=\max_{\norm{Q-P}_\infty\leq \kappa} \norm{\mu_\epsilon(Q)}^2_2+\norm{\eta'_\epsilon(Q)}^2_2\,.$$ 
The radius $R(\kappa)$ is finite due to the classical boundedness theorem which insures that continuous functions on compact sets are bounded.
We have that both $\norm{\mu_\epsilon(Q_2)}^2_2\leq R(\kappa)$ and $\norm{\eta'_\epsilon(Q_2)}^2_2\leq R(\kappa)$ are uniformly bounded for all $Q_2$ such that $\norm{Q_2-P}\leq \kappa$. Hence, we have $\max_{\norm{Q_2-P}_\infty\leq \kappa} \eta_\epsilon (x,Q_2)- \lambda_\epsilon( r, x, Q_2)-\beta_{\epsilon}(Q_2)-\alpha_{\epsilon}(Q_2) \cdot r \mb 1(x=x')<\infty$ as well as $\max_{\norm{Q_2-P}_\infty\leq \kappa}\abs{\alpha_{\epsilon}(Q_2)-\lambda_{\epsilon}(r, x^\star_0, Q_2)}<\infty$ for all $r\in \setr$ and $x\in \tilde X_0$.
Furthermore, by construction, the logarithmic rate $\eta'_\epsilon(Q_2)$  is uniformly bounded away from zero for all $Q_2$ such that $\norm{Q_2-P}\leq \kappa$. Likewise, we have that for any $x\in \tilde X$,  $x'\in \xd$, and $r\in \setr$, the terms
$\eta'_\epsilon(x, Q_2) - \lambda_{\epsilon}(r, x. Q_2)-\beta_{\epsilon}(Q_2)-\alpha_{\epsilon}(Q_2) \cdot r \mb 1(x=x')>0$ are also uniformly bounded away from zero for all $Q_2$ such that $\norm{Q_2-P}\leq \kappa$.\hfill\Halmos 
  
\section{Complexity of DUSA}
\label{sec:complexity}

The main computational bottleneck of DUSA is solving the deep update Algorithm \ref{alg:deep} in which we need to find an $\epsilon$-suboptimal exploration rate $\eta'_\epsilon(P_t)$ to the minimization problem \eqref{eq:silo-dual} characterizing the lower regret bound $C(P_t)$.
As we have mentioned in Section \ref{sec:dusa}, we need to select our exploration rate $\eta'_\epsilon(P_t)$ with a little care to ensure that our selection is in fact continuous.
In Proposition \ref{prop:continuous-deep-selection}, we show that our selection proposed by the deep update Algorithm \ref{alg:deep} is indeed continuous. To make our selection, we solve a minimum Euclidean norm problem over a convex restriction of the set of $\epsilon$-suboptimal exploration rates. Since the complexity of a minimal norm problem is of the order as the dual problem \eqref{eq:silo-dual}, in this section, we only discuss the complexity of the original dual problem itself.

As discussed previously, the minimization problem \eqref{eq:silo-dual} is convex and consequently can be represented in standard conic form as a linear optimization problem over the affine preimage of a convex cone; see for instance \citet{barvinok2002course}. We present such a standard conic representation of the minimization problem \eqref{eq:silo-dual} in Lemma \ref{lemm:conic-minimization} with the help of the exponential cone which we define first.

\begin{definition}[Exponential Cone]
  \label{def:exp_cone}
  The exponential cone $ K_{\mathrm exp}$ is the three dimensional closed convex cone, defined below:
  \[
     K_{\mathrm exp} = \set{(x, y, z)\in\real^3}{z>0,~ \exp(x/z)\leq y/z} \bigcup \set{(x, y, z)\in\real^3}{x\leq 0,~y\geq 0,~ z=0}.
  \]
\end{definition}

The exponential cone is a well known unsymmetric cone in that it is not self dual, i.e., $ K_{\mathrm exp}^\star\neq K_{\mathrm exp}$.  
The dual problem \eqref{eq:silo-dual} characterizing our lower regret bound function can be stated in a standard conic form with the help of the previous cone as indicated next.

\begin{proposition}[Conic Minimization]
  \label{lemm:conic-minimization}
  Let $\optp\in \mathcal P$. Then,  problem \eqref{eq:silo-dual} characterizing the lower regret bound $C(P)$ is equivalent to 
  \begin{align}
    \begin{split}\label{eq:exp:cone}
      \begin{array}{r@{~~}l}
        {\displaystyle\inf_{\eta, \mu,  \ell}} & \sum_{x\in \tilde X(\optp)} \eta(x) \Delta(x, \optp)   \\[0.5em]
        \st  &  \eta(x') \in \real_+, \mu(x')=(\alpha(x'), \beta(x'), \lambda^{\star}(x'))\in \real_+\times\real\times  K^\star, \ell(r, x)(x') \quad \forall x, x' \in X,\\[0.5em]
                                & (-\ell(r, x)(x'), \eta(x), \eta(x)-\lambda(r, x)(x')-\beta(x')-\alpha(x')\cdot r\mb 1(x=x'))\in {K}_{\mathrm exp} \\[0.3em]\quad
        & \hspace{26em} \forall r\in \setr, ~x, x'\in X,\\[0.5em]
             & -1 - \sum_{x\in \tilde X(P)} \ell(r, x)(x') P(r, x) - \sum_{x\in X^\star(P)} \lambda(r, x)(x') P(r, x)  \\[0.3em]
             &               \hspace{7em} +\alpha(x')\cdot \rew^\star(P) +\beta(x') \abs{\tilde X(P)}\in \real_+ \hspace{4.65em} \forall x'\in\xd(P).
      \end{array}
    \end{split}
  \end{align}

\end{proposition} 
\proof{Proof of Proposition \ref{lemm:conic-minimization}}
Observe that the objective function of our optimization problem \eqref{eq:exp:cone} is the same as that of the dual problem \eqref{eq:silo-dual}. The first set of constraints of problem \eqref{eq:exp:cone} ensures that the target rate $\eta$ is nonnegative and the dual variables are feasible. Its last two sets of constraints replace the last constraint of the dual problem \eqref{eq:silo-dual}, i.e., $ 1\leq  \dual ( \eta, x', \optp~;~ \mu(x'))$, using exponential cones. 
To make it clear, we  start with examining the second set of constraints that involves the exponential cone ${K}_{\mathrm exp}$.  Note that by \citet{chandrasekaran2017relative}, we have $(u, v, w)\in {K}_{\mathrm exp}   \iff    \set{(u, v, w)}{v\log(\tfrac{w}{v})\geq u}$. Considering this, the second set of constraints in problem \eqref{eq:exp:cone} can be written as 
\[ \eta(x)\log\left(\frac{\eta(x)-\lambda(r, x)(x')-\beta(x')-\alpha(x')\cdot r \mb 1(x=x')}{\eta(x)}\right) \ge -\ell(r, x)(x')\,. \]
Observe that the left hand side is $\eta(x)\log(\omega)$, where variable $\omega= \big( \eta(x)-\lambda(r, x)(x')-\beta(x')-\alpha(x') r \mb 1(x=x')\big)/\eta (x)$. Thus,  $-\ell(r, x)(x')$ is a lower bound on $\eta(x)\log(\omega)$, where this term, i.e., $\eta(x)\log(\omega)$, appears in the definition of the dual function $\dual$, given in Equation \eqref{eq:dual}. Having this interpretation of variable $\ell$'s in mind, it is easy to see that the last set of constraints of problem \eqref{eq:exp:cone} ensures that $ 1\leq  \dual ( \eta, x', \optp~;~ \mu(x'))$.\hfill\Halmos
\endproof

Proposition \ref{lemm:conic-minimization} presents a conic representation of the convex problem \eqref{eq:silo-dual}. However, not all convex optimization problems are easy to solve. Convex problems do admit polynomial time algorithms when their conic representation involves exclusively cones which admit an ``efficient'' representation. Efficiently representable cones, such as the positive orthant, second-order cone, or positive semi-definite cone, are precisely those for which an efficient barrier function is known and consequently the associated conic optimization problem can be tractably solved using a modern interior point method. That is, the conic minimization problem can be solved in time polynomial in the number of arms $\abs{X}$, the number of rewards $\abs{\setr}$ as well as in the number of required accurate digits of the optimal solution. As each of the cones $\real$, $\real_+$ and ${K}_{\mathrm exp}$ admits an efficient barrier, the dual cone $\mathcal K^\star$ will ultimately determine whether the dual deep update can be carried out tractably or not.

Recall that the dual cone $\mathcal K^\star$ is defined as the dual of the conic hull $\mathcal K$ of the set of all reward distributions $\mathcal P$ and  captures all structural information in the bandit problem. It is well known, c.f., \cite{nesterov1998introductory}, that $\mathcal K^\star=\cone(\mathcal P)^\star$ is efficiently representable if and only if its dual $\mathcal K=\cone(\mathcal P)$ is too. Hence, whether or not the convex problem \eqref{eq:silo-dual} can be solved tractably depends only on whether we can represent the structural information regarding the bandit reward distributions efficiently.

\section{Supporting Results on Duality and Continuity}
\label{sec:constraint-qualification}

In many of the proofs in this paper,  we need some basic results concerning the existence of the strong duality and  continuity of the minimum of parametric convex optimization problems.
The fact that strong duality and continuity  are intimately related has been long established. See for instance \citet[Section 4.4]{bank1982parametric} for a more in depth discussion of this connection. In this section, for readers convenience, 
we present the most relevant results on these subjects to our paper.

\subsection{Duality}
\label{app:duality}

Consider the following standard minimization problem
\begin{equation}
  \label{eq:standard-parametric-problem}
  \begin{array}{rl@{\hspace{3em}}l}
    v^\star(\theta)\defn \inf_y & f(y, \theta) & \\
    \st  & g(y, \theta)\leq 0, & [{\mathrm{dual~variable}}~\mu_1 \in \real_+^{m_1}]\\
                         & h(y, \theta)= 0, & [{\mathrm{dual~variable}}~\mu_2 \in \real^{m_2}]\\
                         & y \in K, &  [{\mathrm{dual~variable}}~\mu_3 \in K^\star]\\
         & y \in \Gamma(\theta) \subseteq Y& [{\mathrm{not~dualized}}]
  \end{array}
\end{equation}
parameterized in $\theta\in\Theta$. In what follows we assume that $Y$ and $\Theta$ are finite dimensional linear spaces. The inequality and equality constraint functions $g$ and $h$, respectively, characterize the first two functional constraints. Note that there are $m_1$  inequality constraints, where the $i^{th}$ inequality constraint is associated with the inequality constraint function $g_i$. Similarly, there are $m_2$  equality constraints, where the $i^{th}$ equality constraint is associated with equality constraint function $h_i$. The third constraint characterizes a conic set constraint. As before, we denote with $K^\star$ the dual cone of the closed convex cone $K$. The mapping $\Gamma(\theta)$, also referred to as set-valued mapping, characterizing the ultimate parametric set constraint is assumed to be non-empty for all $\theta\in \Theta$. The objective function $f(y, \theta)$ itself implicitly constrains the optimization problem to its domain $\dom(f(y, \theta))\defn\set{y}{f(y, \theta)<+\infty}$.
We will denote with $y^\star(\theta)$ the set of all minimizers in problem \eqref{eq:standard-parametric-problem}. We remark that $y^\star(\theta)$ may be an empty set as the infimum is not necessarily attained.
We will denote $M(\theta)$ the feasible set mapping in problem \eqref{eq:standard-parametric-problem}. Consequently, we have that $y^\star(\theta)\subseteq M(\theta)$.
The following standard constraint qualification condition plays a pivotal role in establishing both strong duality for our minimization problem \eqref{eq:standard-parametric-problem} as well as the continuity of its minimum $v^\star(\theta)$ in the parameter $\theta$. 

\begin{definition}[Slater's Constraint Qualification Condition]
\label{def:slater}
  We say that the feasible region of the primal optimization problem \eqref{eq:standard-parametric-problem} satisfies Slater's constraint qualification condition at some given $\theta_0\in\Theta$ if there exists a Slater point $y_0\in \Gamma(\theta_0)$ such that $g(y_0, \theta_0)<0$, $h(y_0, \theta_0)=0$ and $y_0\in \interior(K)$. We will denote the set of all such Slater points as $$M_0(\theta)\defn \Gamma(\theta)\bigcap \set{y\in\interior(K)}{g(y, \theta)< 0, ~h(y, \theta)=0}\,.$$
\end{definition}

Consider the Lagrangian function
\(
  L(y, \theta~;~\mu=(\mu_1, \mu_2, \mu_3)) = f(y, \theta)+\mu_1\tpose g(y, \theta) + \mu_2\tpose h(y, \theta) - \mu_3\tpose y
\)
and the associated dual problem
\begin{equation}
  \label{eq:standard-parametric-problem-dual}
  \begin{array}{rl}
    u^\star(\theta) = \sup_\mu & q(\theta~;~\mu=(\mu_1, \mu_2, \mu_3))\\
    \st & \mu_1 \in \real_+^{m_1}, \\
                      & \mu_2 \in \real^{m_2}, \\
                      & \mu_3 \in K^\star
  \end{array}
\end{equation}
for the dual function $q(\theta~;~\mu)\defn \inf \tset{L(y, \theta~;~\mu)}{y \in \Gamma(\theta)}$.  As indicated in the minimization problem \eqref{eq:standard-parametric-problem}, the dual variables $\mu=(\mu_1, \mu_2, \mu_3)$ are associated with the primal functional constraints $g(y, \theta)\leq 0$, $h(y, \theta)=0$ and the primal set constraint $y\in K$. The other primal set constraint  $y \in \Gamma(\theta)$ is not dualized and consequently has no associated dual variables. Note that we trivially have the weak duality inequality $u^\star(\theta)\leq v^\star(\theta)$. In order to have strong duality, i.e., $u^\star(\theta)= v^\star(\theta)$, one typically needs to require convexity and some type of constraint qualification condition. We recall the following classical result which can be derived immediately from \citet[Section 5.9]{boyd2004convex}.

\begin{proposition}[Strong Duality]
  \label{prop:strong-duality-general}
  Fix a parameter value $\theta_0\in \Theta$. Let the objective $f(y, \theta_0)$ be a convex function in $y$.  Likewise, the inequality constraint functions $g_i(y, \theta_0)$, $i\in \{1, \dots, m_1\}$ are convex in $y$ while the equality constraint functions $h_i(y, \theta_0)$ , $i\in \{1, \dots, m_2\}$ are affine in $y$.
  Assume that the primal minimum in problem \eqref{eq:standard-parametric-problem} is bounded from above and that $$M_0(\theta_0)\bigcap \relint(\dom(f(y, \theta_0))\bigcap \Gamma(\theta_0))\neq\emptyset.$$ That is, there is a Slater point which is in the intersection of the relative interior of the domain of $f(y, \theta_0)$ and the parametric set constraint $\Gamma(\theta_0)$. Then, we have the strong duality equality $u^\star(\theta_0)= v^\star(\theta_0)$ and the set of dual optimal solutions is non-empty and compact.
\end{proposition}
\proof{}
  It is evident that the primal can  be characterized as the convex optimization problem
  \[
    \begin{array}{rl}
      v^\star(\theta_0)= \inf_y & f(y, \theta_0) + \chi_{+\infty}(\Gamma(\theta_0))(y) \\
      \st  & g(y, \theta_0)\leq 0, \\
      & h(y, \theta_0)= 0, \\
                       & y\in K\,,
    \end{array}
  \]
  where the convex characteristic function $\chi_{+\infty}(\Gamma(\theta_0))(y)$ is zero if $y\in \Gamma(\theta_0)$ and positive infinity otherwise. From \citet[Section 5.9]{boyd2004convex}, we know that strong duality, i.e., $u^\star(\theta_0)= v^\star(\theta_0)$, holds if we can find a  point $$y_0\in \relint(\dom(f(y, \theta_0)) \bigcap \Gamma(\theta_0) \bigcap \dom(g(y, \theta_0))\bigcap K)$$ such that $y_0 \in \interior(K)$ and $g(y_0, \theta_0)<0$ and $h(y_0, \theta_0)=0$. From \citet[Proposition 1.3.8]{bertsekas2009convex} and the fact that $\relint(C)\subseteq \interior(C)$ for any set $C$ it follows that $$M_0(\theta_0)\bigcap \relint(\dom(f(y, \theta_0)\bigcap \Gamma(\theta_0)))\neq\emptyset$$  is sufficient to guarantee that such a  point $y_0$ exists.\hfill\Halmos
\endproof

\subsection{Continuity}
\label{app:continuity}

The same constraint qualification condition stated in Definition \ref{def:slater} also plays an important role for the minimum $v^\star(\theta)$ of problem \eqref{eq:standard-parametric-problem} to be a continuous function.
Establishing continuity of a minimum typically requires  its objective function to  lower semi-continuous and its feasible set mapping to be a lower semi-continuous mapping in addition to some type of constraint qualification condition. As the definition of what constitutes semi continuous mappings is not universally consistent\footnote{For instance, \citet{berge1997topological} requires upper semi-continuous mappings to be compact valued while we following \citet{aubin2009set} do not.} throughout the literature we state for completeness the definition taken in this paper explicitly.

\begin{definition}[Semi-continuous Mappings]
A mapping $G:A\rightarrow B$ is said to be upper semi-continuous at point $a\in A$ if for any open neighborhood $V$ of $G(a)$,  there exists a neighbourhood $U$ such that  all $x \in U$, $G(x)$ is a subset of $V$.
A mapping $G:A\rightarrow B$ is said to be lower semi-continuous at the point $a\in A$ if for any open set $V$ intersecting $G(a)$, there exists a neighbourhood $U$ such that for any $x\in U$, the intersection of $G(x)$ and $V$ is non-empty. A mapping is lower semi-continuous if it is lower semi-continuous at any point $a\in A$ and upper semi-continuous if it is upper semi-continuous at any point $a\in A$.
\end{definition}

In this paper, we only have a need to show continuity of the minimum $v^\star(\theta)$ in problem \eqref{eq:standard-parametric-problem} when only the first and final constraint are nontrivial. That is, it suffices here to consider problems with no explicit cone and equality constraints.
We have indeed the following direct corollary of a more general result by \citet[Theorem 3.1.6]{bank1982parametric} in arbitrary normed linear spaces.

\begin{proposition}[Lower Semi-continuity of Feasible Sets]
   \label{prop:continuity-mapping-general}
  Let the mapping $\Gamma(\theta)$ be a lower semi-continuous mapping at $\theta_0\in \Theta$, the set $\Gamma(\theta_0)$ be convex and the feasible set in problem \eqref{eq:standard-parametric-problem} satisfy the Slater's constraint qualification, i.e., $M_0(\theta_0)\neq \emptyset$, with no cone or equality constraints ($K=\real^n$ and $m_2=0$). Further,
  \begin{enumerate}
  \item[(i)] let the inequality constraint functions $g_i(y, \theta)$, $i\in\{1, \dots, m_1\}$, be upper semi-continuous on $M_0(\theta_0)\times \{\theta_0\}$ and
  \item[(ii)] let the inequality constraint functions $y\mapsto g_i(y, \theta_0)$, $i\in\{1, \dots, m_1\}$, be convex on $\Gamma(\theta_0)$.
  \end{enumerate}
  Then, the mapping $M(\theta)$ is lower semi-continuous at $\theta_0$.
\end{proposition}
The following standard result by the same authors \citet[Theorem 4.3.3]{bank1982parametric} can then be invoked to prove continuity of the primal minimum.

\begin{proposition}[Continuity of Optimal Value]
  \label{prop:continuity-function-general}
  Let the feasible set mapping $M(\theta)$ be lower semi-continuous at $\theta_0\in \Theta$ in the primal minimization problem \eqref{eq:standard-parametric-problem} with no cone or equality constraints ($K=\real^n$ and $m_2=0$). Furthermore, assume that its feasible region $M(\theta)$ is convex for all $\theta\in \Theta$ and the mapping $M(\theta)$ is closed at $\theta_0$. Finally, assume that $y^\star(\theta_0)$ is a non-empty bounded set. Let the objective function $(y, \theta)\mapsto f(y, \theta)$ satisfy the following technical conditions:
  \begin{enumerate}
  \item the function $f(y, \theta)$ is lower semi-continuous on $Y\times \{\theta_0\}$.
  \item the function $f(y, \theta)$ is upper semi-continuous at some $(y^\star_0, \theta_0)$ with $y^\star_0\in y^\star(\theta_0)$.
  \item the function $f(y, \theta)$ is convex in $y\in Y$ for each parameter value $\theta\in \Theta$.
  \end{enumerate}
  Then, the minimum $v^\star(\theta)$ is continuous at $\theta_0$.
\end{proposition}

\end{APPENDICES}

\end{document}